%% file: main.tex
\documentclass[dvipsnames]{article}

\usepackage{iclr2024_conference,times}

\usepackage[utf8]{inputenc} 
\usepackage[T1]{fontenc}    

\usepackage{xcolor} 
\definecolor{mydarkblue}{rgb}{0,0.08,0.45}
\definecolor{myredorange}{rgb}{0.8,0.33,0}

\usepackage[pagebackref=true,breaklinks=true,colorlinks,citecolor=mydarkblue,linkcolor=myredorange,bookmarks=false]{hyperref}

\usepackage{url}            
\usepackage{booktabs}       
\usepackage{amsfonts}       
\usepackage{nicefrac}       
\usepackage{microtype}      

\usepackage{graphicx}
\usepackage{lscape}
\usepackage{rotating}
\usepackage{tablefootnote}

\usepackage[ruled, vlined, linesnumbered]{algorithm2e}
\usepackage{algorithmic}
\usepackage[utf8]{inputenc}
\usepackage{float}
\usepackage{url}            
\usepackage{xspace}
\usepackage{amsthm, amsmath, mathtools, amsfonts, amssymb, dsfont, enumitem, mathabx}
\usepackage{bm, bbm}
\usepackage[mathscr]{euscript}
\usepackage{adjustbox}
\usepackage{footmisc}
\usepackage[capitalize,noabbrev]{cleveref}
\usepackage{multirow}
\usepackage{transparent}
\usepackage{caption}
\usepackage{subcaption}
\usepackage{tcolorbox}
\usepackage{wrapfig}
\usepackage{enumitem}
\usepackage[font=small]{caption}
\usepackage{bbm}
\usepackage{pifont}
\usepackage{cancel}
\newcommand{\xmark}{\ding{55}}%

\usepackage{titlesec}
\titlespacing{\section}{0pt}{2pt}{0pt}
\titlespacing{\subsection}{0pt}{2pt}{0pt}
\titlespacing{\subsubsection}{0pt}{2pt}{0pt}
\titlespacing{\paragraph}{0pt}{0pt}{5pt}
\everypar{\looseness=-1}

\iclrfinalcopy

\usepackage{minitoc}
\usepackage{thmtools, thm-restate}

\makeatletter
\def\thm@space@setup{%
  \thm@preskip=2pt
  \thm@postskip=2pt 
}
\makeatother

\theoremstyle{plain}
\theoremstyle{definition}

\theoremstyle{remark}

\input{math_commands}

\definecolor{lightorange}{HTML}{ff7f2a}
\definecolor{lighterorange}{HTML}{ffe6d5}
\newtcolorbox{summarybox}{colback=lighterorange,colframe=lightorange}

\newcommand{\E}[2]{\mathbb{E}_{#1}{\left[#2\right]}}

\newcommand{\kl}[2]{D_{\mathrm{KL}}(#1\ ||\ #2)}
\newcommand{\f}[2]{D_{f}(#1\ ||\ #2)}
\DeclarePairedDelimiter\abs{\lvert}{\rvert}%

\newcommand{\reb}[1]{\textcolor{black}{#1}}

\renewcommand{\abs}[1]{\left|#1\right|}
\newcommand{\set}[1]{\left\{#1\right\}}

\renewcommand{\S}{\mathcal{S}}
\newcommand{\A}{\mathcal{A}}

\renewcommand{\hat}{\widehat}
\renewcommand{\S}{\mathcal{S}}

\newcommand{\bellman}{\mathcal{T}_r^{\pi}}
\newcommand{\primalQ}{\texttt{primal-Q}\xspace}
\newcommand{\primalV}{\texttt{primal-V}\xspace}
\newcommand{\dualQ}{\texttt{dual-Q}\xspace}
\newcommand{\dualV}{\texttt{dual-V}\xspace}
\newcommand{\tdV}{\delta_V}

\newcommand{\Timit}{\mathcal{T}_{r^{\text{imit}}}}
\newcommand{\fdvl}{\texttt{$f$-DVL}\xspace}
\newcommand{\dmix}{d^S_\text{mix}(s,a)}
\newcommand{\demix}{d^{E,S}_\text{mix}(s,a)}

\newcommand{\highlight}[1]{\textcolor{RoyalBlue}{\textbf{#1}}}

\everypar{\looseness=-1}

\begin{document}
\title{Dual RL: Unification and New Methods for Reinforcement and Imitation Learning}

\author{Harshit Sikchi$^1$,  Qinqing Zheng $^2$,  Amy Zhang$^{1,2}$,  Scott Niekum$^3$ \\
$^1$ The University of Texas at Austin, $^2$ FAIR, Meta AI\\
$^3$ University of Massachusetts Amherst\\
\texttt{\{hsikchi\}@utexas.edu} \\
}


\doparttoc 
\faketableofcontents 

\maketitle

\begin{abstract}

The goal of reinforcement learning (RL) is to find a policy that maximizes the expected cumulative return. It has been shown that this objective can be represented as an optimization problem of state-action visitation distribution under linear constraints. The dual problem of this formulation, which we refer to as \textit{dual RL}, is unconstrained and easier to optimize. In this work, we first cast several state-of-the-art offline RL and offline imitation learning (IL) algorithms as instances of dual RL approaches with shared structures. Such unification allows us to identify the root cause of the shortcomings of prior methods. For offline IL, our analysis shows that prior methods are based on a restrictive coverage assumption that greatly limits their performance in practice. To fix this limitation, we propose a new discriminator-free method ReCOIL that learns to imitate from arbitrary off-policy data to obtain near-expert performance. For offline RL, our analysis frames a recent offline RL method XQL in the dual framework, and we further propose a new method $f$-DVL that provides alternative choices to the Gumbel regression loss that fixes the known training instability issue of XQL. The performance improvements by both of our proposed methods, ReCOIL and $f$-DVL, in IL and RL are validated on an extensive suite of simulated robot locomotion and manipulation tasks.\\

\centering{\textbf{Project page (Code and Videos):} \href{https://hari-sikchi.github.io/dual-rl/}{\color{myredorange}hari-sikchi.github.io/dual-rl/}}
\end{abstract}


\input{1introduction}
\input{2related}

\input{3preliminaries}
\input{4dual}
\input{5.1unification}
\input{5.2recoil}
\input{5.3DVL}
\input{6experiments}
\input{7conclusion}

\input{8acknowledgements}
\newpage
\bibliography{example_paper}
\bibliographystyle{abbrvnat}

\newpage
\appendix
\input{9appendix}

\end{document}

%% file: math_commands.tex
\usepackage{amsmath,amsfonts,bm}

\def\1{\bm{1}}

\DeclareMathAlphabet{\mathsfit}{\encodingdefault}{\sfdefault}{m}{sl}
\SetMathAlphabet{\mathsfit}{bold}{\encodingdefault}{\sfdefault}{bx}{n}

\newcommand{\R}{\mathbb{R}}

\DeclareMathOperator*{\argmax}{arg\,max}
\DeclareMathOperator*{\argmin}{arg\,min}

%% file: 1introduction.tex
\section{Introduction}
A number of deep Reinforcement Learning (RL) algorithms optimize a regularized policy learning objective using approximate dynamic programming (ADP)~\citep{bertsekas1995neuro}. Popular off-policy temporal difference algorithms spanning both imitation learning~\citep{kostrikov2018discriminator,ni2021f} and RL~\citep{haarnoja2018soft,janner2019trust, sikchi2022learning,hafner2023mastering} exemplify this class. 
As we discuss in Section~\ref{sec:preliminaries}, one way to develop a principled off-policy algorithm is to ensure unbiased estimation of the on-policy policy gradient using off-policy data~\citep{nachum2020reinforcement}. Unfortunately, many classical off-policy algorithms do not guarantee this property, resulting in issues like training instability and over-estimation of the value function~\citep{fu2019diagnosing,fujimoto2018addressing,baird1995residual}. To obtain high learning performance, these algorithms require that most data be nearly on-policy, otherwise require special algorithmic treatments (e.g., importance sampling~\citep{precup2001off}, layer normalization~\citep{ball2023efficient}, prioritized sampling~\citep{vecerik2017leveraging}) to avoid the aforementioned issues. Recently, there have been developments leading to new off-policy algorithms with improved performance for RL~\citep{kumar2020conservative,garg2021iq,kostrikov2021offline} and IL~\citep{zhu2020off,ma2022smodice,garg2021iq,florence2022implicit}. These methods are derived via a variety of mathematical tools and attribute their success to different aspects. It remains an open question if we can inspect these algorithms under a unified framework to understand their limitations, and subsequently propose better methods.

In this work, we consider a specific formulation for RL that writes the performance of a policy as a convex program with linear constraints~\citep{manne1960linear}. 
This convex program can be converted into unconstrained forms using Lagrangian duality, which is more amenable for numerical optimization. We refer to the class of approaches that admit the dual formulations as \emph{Dual RL}. Dual RL approaches naturally provide unbiased estimation of the on-policy policy gradient using off-policy data, in a principled way. They avoid explicit importance sampling that leads to high variance and ensures training stability and convergence~\citep{tsitsiklis1996analysis}. Related approaches in this space have often been referred to as DICE (DIstribution Correction Estimation) methods in previous literature~\citep{nachum2019algaedice,kostrikov2019imitation,lee2021optidice,ma2022smodice, zhang2020gendice}. We note that the linear programming formulation for the RL objective has been used and studied in \cite{manne1960linear, denardo1970linear, de1967linear,borkar1988convex,malek2014linear} and the general duality framework for regularized RL was first introduced in~\citet{nachum2020reinforcement}.

Our \textit{first} contribution is to extend the work of \citet{nachum2020reinforcement} and show that
many recent algorithms in deep RL and IL can all be viewed 
as different instantiations of dual problems for regularized policy optimization, see Table~\ref{tab:works} for the complete list. These algorithms have been motivated from a variety of perspectives and differing derivations. For example, XQL~\citep{garg2023extreme} focuses on introducing Gumbel regression into RL, 
CQL~\citep{kumar2020conservative} aims at learning a pessimistic $Q$ function, IQLearn~\citep{garg2021iq} and OPOLO~\citep{zhu2020off} use the change of variables for IL, and IBC~\citep{florence2022implicit} uses a contrastive loss for imitation learning, but as we show all can actually be derived from the dual formulation.

\textit{Second}, the dual unification in IL reveals an important shortcoming of prior methods that learn to imitate the expert by leveraging arbitrary off-policy data.
Prior work~\citep{ma2022smodice,zhu2020off,kim2022demodice} imposes a coverage assumption (the suboptimal data covers the visitations of the expert data) and learn a density ratio between suboptimal and expert data via a discriminator to use it for downstream learning. 
In an offline setting, with limited data and coverage, learning a density ratio between suboptimal and expert can be challenging, and the inaccuracies of the discriminator can compound in downstream RL, negatively affecting resulting policy performance. We show that by a simple modification to the dual formulation, we can get away from this limitation. In Section~\ref{sec:new_il_method}, we present $\texttt{ReCOIL}$, a simple, theoretically principled, and discriminator-free imitation learning method from arbitrary off-policy data. We empirically demonstrate the failure of previous IL methods based on the coverage assumption in a number of MuJoCo environments and show substantial performance improvements of $\texttt{ReCOIL}$ in Section~\ref{sec:result}.   

\textit{Third}, the presented unification also provides us with a useful tool to examine the limitation for a recent offline RL method, XQL~\citep{garg2023extreme}. XQL's success was originally attributed to better modeling of Bellman errors using Gumbel regression. On the other hand, XQL also suffers from training instability, also caused by Gumbel regression. By situating the implicit policy improvement algorithms like XQL in the dual RL framework, in Section~\ref{sec:new_rl_method} we are able to propose a family of implicit algorithms $f$-Dual V Learning ($f$-\texttt{DVL}), which successfully addresses the training instability issue. The empirical experiments on the D4RL benchmarks establish the superior performance of $f$-\texttt{DVL}, see Section~\ref{sec:result}. \looseness=-1

\begin{table*}[t]
\vspace{-10pt}
    \centering
    \scriptsize
    \resizebox{0.85\textwidth}{!}{
    \begin{tabular}{c|cccccc} 
    \toprule
    &\multicolumn{1}{c}{\textbf{Method}} & \textbf{dual-Q/V}  & \textbf{Gradient} & {\textbf{Objective}} & \textbf{Off-Policy Data} \\
    \midrule
    \vspace{0.5mm}    
    \multirow{4}{*}{RL}& AlgaeDICE, GenDICE, \textit{CQL}& $Q$ &  semi  & regularized RL  &  Arbitrary \\
    & OptiDICE  &$V$ & full  & regularized RL  & Arbitrary\\
     & \textit{XQL}, REPS, $f$-\textbf{DVL} & $V$&  semi  & regularized RL &  Arbitrary \\
      & VIP, GoFAR & $V$ &  full  & regularized RL &  Arbitrary \\
    &Logistic Q-learning & $QV$\footnotemark[1] & full   & regularized RL  &  \xmark \\
     \midrule
     \multirow{9}{*}{IL}& \textit{IQLearn, IBC} & $Q$ & semi  & $D_f(\rho^\pi\|\rho^E)$  &  Expert-only   \\
    & \textit{OPOLO, OPIRL} & $Q$ & semi & $D_{kl}(\rho^\pi\|\rho^E)$   & Arbitrary  \\
    & SMODICE & $V$ & full & $D_{kl}(\rho^\pi\|\rho^E)$  & Arbitrary \\
    & DemoDICE, LobsDICE & $V$ & full & $D_{kl}(\rho^\pi\|\rho^E)+\alpha D_{kl}(\rho^\pi\|\rho^R) $  & Arbitrary  \\
     & P$^2$IL & $QV$\footnotemark[1]  & full & $D_{C}(\rho^\pi\|\rho^E)$\footnotemark[1]  &\xmark   \\
    & \textbf{ReCOIL-Q} & $Q$  & full & $D_{f}(\rho^\pi_{mix}\|\rho^{E,R}_{mix})$&   Arbitrary  \\[0.1cm] 
    & \textbf{ReCOIL-V}  & $V$  & full & $D_{f}(\rho^\pi_{mix}\|\rho^{E,R}_{mix})$&   Arbitrary \\
    \bottomrule
    \end{tabular}
    }
    \vspace{-5pt}
    \caption{A number of recent works can be studied together under the unified umbrella of \textbf{dual-RL}. These methods are instantiations of dual-RL with a choice of update strategy, objective, constraints, and their ability to handle off-policy data. \textbf{Bold} names correspond to the methods proposed in the paper and \textit{Italic} names correspond to methods that aren't yet known to be dual-approaches. }
    \label{tab:works}
\end{table*}
\footnotetext[1]{These methods use a different regularizer. More details in Appendix~\ref{ap:LQL_P2IL}.}

%% file: 2related.tex
\section{Related Work}
\label{sec:related}

\paragraph{Off-Policy Methods for IL} Imitation learning has benefited greatly from using off-policy data to improve learning performance~\citep{kostrikov2018discriminator,agarwal2020imitative,zhu2020off,ni2021f,sikchi2022ranking}. Often, replacing the on-policy expectation common in most Inverse RL formulations~\citep{ziebart2008maximum,swamy2021moments} by expectation under off-policy samples, which is unprincipled, has led to gains in sample efficiency~\citep{kostrikov2018discriminator}. Previous works have proposed a solution in the dual RL space for principled off-policy imitation but is based on a restrictive coverage assumption~\citep{ma2022smodice,zhu2020off, kim2022demodice} which requires estimating a density ratio using a discriminator and further limit themselves to matching a particular $f$-divergence. In this work, we eliminate this assumption and allow for generalization to all $f$-divergences, presenting a principled off-policy discriminator-free approach to imitation. 

\paragraph{Off-Policy Methods for RL} Off-policy RL methods promise a way to utilize data collected by arbitrary behavior policies to aid in learning an optimal policy and thus are advantageous over on-policy methods.   This promise falls short, as previous off-policy algorithms are plagued with a number of issues such as overestimation of the value function, training instability, and various biases~\citep{thrun1993issues,fu2019diagnosing,fujimoto2018addressing,kumar2019stabilizing}. A common cause for a number of these issues is \emph{distribution mismatch}. As we shall discuss later, the RL objective requires on-policy samples but is often estimated by off-policy samples in practice. Prior works have proposed fixing the distribution mismatch by using importance weights~\citep{precup2000eligibility}, which can lead to high variance policy gradients or ignoring the distribution mismatch completely~\citep{haarnoja2018soft,fujimoto2018addressing, sun2022optimistic,ji2023seizing,chen2021randomized}. Unfortunately, these approaches do not carry over well to the offline setting. For example, when deploying the policy online, the overestimation bias can be corrected by the environment feedback, which is infeasible for offline RL.A number of solutions exist for controlling overestimation in prior work---$f$-divergence regularization to the training distribution~\citep{wu2019behavior,fujimoto2019off}, support regularization~\citep{singh2022offline}, implicit maximization~\citep{kostrikov2021offline}, Gumbel regression~\citep{garg2023extreme} and learning a Q function that penalizes OOD actions~\citep{kumar2020conservative}. Dual-RL methods consider a regularized RL setting suitable for offline RL as well as fixing the distribution mismatch issue in a principled way.

%% file: 3preliminaries.tex
\section{Preliminaries}
\label{sec:preliminaries}
We consider an infinite horizon discounted Markov Decision Process denoted by the tuple $\mathcal{M} = (\mathcal{S}, \mathcal{A}, p, r, \gamma, d_0)$, where $\S$ is the state space, $\A$ is the action space, $p$ is the transition probability function, $r: \S \times \A \rightarrow \R $ is the reward function, $\gamma \in (0, 1)$ is the discount factor, and $d_0$ is the distribution of initial state $s_0$. Let $\Delta(\A)$ denote the probability simplex supported on $\A$.
The goal of RL is to find a policy $\pi: \S  \rightarrow \Delta(\A)$ that maximizes the expected return:  $\E{\pi}{\sum_{t=0}^\infty \gamma^t r(s_t, a_t)}$, where we use $\mathbb{E}_{\pi}$ to denote the expectation under the distribution induced by $a_t\sim\pi(\cdot|s_t), s_{t+1}\sim p(\cdot|s_t,a_t)$. We also define the discounted state-action visitation distribution $d^\pi(s,a)=\reb{(1-\gamma)}\pi(a|s)\sum_{t=0}^\infty \gamma^t P(s_t=s|\pi)$. The unique stationary policy that induces a visitation $d(s,a)$ is given by $\pi(a|s) = d(s,a)/\sum_a d(s,a)$. We will use $d^O$ and $d^E$ to denote the visitation distributions of the behavior policy of the offline dataset and the expert policy, respectively. \reb{$f$-divergences are denoted by $D_f$ and measure the distance between two distributions. For a more formal overview of the above concepts, refer to Appendix~\ref{ap:dual_rl_review}.}

\paragraph{Value Functions and Bellman Operators} Let $V^\pi$: $\S\rightarrow \mathbb{R}$ be the state value function of $\pi$. 
$V^\pi(s)$ is the expected return when starting from $s$ and following $\pi$: $V^\pi(s) = \E{\pi}{\sum_{t=0}^\infty \gamma^t r(s_t, a_t) | s_0 = s}.$
Similarly, let $Q^\pi$ : $\S \times \A \rightarrow  \mathbb{R}$ be the state-action value function of $\pi$, such that
$ Q^\pi(s, a) = \E{\pi}{\sum_{t=0}^\infty \gamma^t r(s_t, a_t) | s_0 = s, a_0 = a}$.
\looseness=-1
Let $V^*$ and $Q^*$ denote the value functions corresponding to an optimal policy $\pi^*$. Let $\bellman$ be the Bellman operator with policy $\pi$ and reward function $r$ such that
$\bellman Q(s, a) = r(s, a) + \gamma \E{s' \sim p(\cdot | s, a), a' \sim \pi(\cdot | s') }{ Q(s',a') }$.
We also define the Bellman operator for the state value function 
$\mathcal{T}_r V(s,a) = r(s,a) + \gamma \E{s' \sim p(\cdot|s,a)}{V(s')}$.




%% file: 4dual.tex
\subsection{Reinforcement Learning via Lagrangian Duality}
\label{sec:dual_rl}
RL optimizes the expected return of a policy. We consider the linear programming
formulation of the expected return~\citep{manne1960linear}, to which we can apply Lagrangian duality or Fenchel-Rockfeller duality to obtain corresponding constraint-free problems. We now review this framework, first introduced by~\citet{nachum2020reinforcement}\footnote{We use Lagrangian duality instead of Fenchel-Rockfeller duality for ease of exposition.}. Consider the regularized policy learning problem 
\begin{equation}
\label{eq:reg_rl_main}
\max_\pi \, J(\pi)=\mathbb{E}_{d^\pi(s,a)}[r(s,a)] -\alpha \f{d^\pi(s,a)}{d^O(s,a)},
\end{equation}
where 
$\f{d^\pi(s,a)}{d^O(s,a)}$ is a conservatism regularizer that encourages the visitation distribution of $\pi$ to stay close to some distribution $d^O$, and
$\alpha$ is a temperature parameter that balances the expected return and the conservatism.
An interesting fact is that $J(\pi)$ can be rewritten as a convex problem that searches for a visitation distribution that satisfies the \textit{Bellman-flow} constraints. We refer to this form as \primalQ:
\begin{equation}
\label{eq:primal_rl_q_main}
\begin{aligned}
& \colorbox{RoyalBlue!15}{\primalQ} \;  \max_\pi J(\pi)  = \max_\pi \big[\max_{d} \; \mathbb{E}_{d(s,a)}[r(s,a)]-{\alpha}\f{d(s,a)}{d^O(s,a)}\\
&   \hskip10pt \text{s.t} \resizebox{0.85 \hsize}{!}{$\; d(s,a)=(1-\gamma)d_0(s).\pi(a|s)+\gamma \textstyle \sum_{s',a'} d(s',a')p(s|s',a')\pi(a|s), \; \forall s \in \S, a \in \A \big].$}  
\end{aligned}
\end{equation}
We can convert this to an unconstrained problem with dual variables $Q(s, a)$ defined for all $s, a \in \S \times \A$ by applying Lagrangian duality and the convex conjugate, giving us the \dualQ formulation:
\begin{align}
    \label{eq:dual_q_original_equation}
    \colorbox{Goldenrod!30}{\dualQ} & \; 
\resizebox{0.8\hsize}{!}{$\max_{\pi} \min_{Q} (1-\gamma)\E{s\sim d_0,a \sim \pi(s)}{Q(s,a)} +{\alpha}\E{(s,a)\sim d^O}{f^*\left(\left[ \bellman Q(s,a)-Q(s,a)\right]/ \alpha \right)}$},
\end{align}
where $f^*$ is the convex conjugate of $f$. 
In fact, one can note that Problem~\eqref{eq:primal_rl_q_main} is overconstrained---the constraints already determine the unique solution $d^\pi$, rendering the inner maximization w.r.t $d$ unnecessary. Therefore, we can relax the constraints to obtain another problem with the same optimal solution $\pi^*$ and $d^*$, which we call \primalV below:
\begin{equation}
  \begin{aligned}
\label{eq:primal_rl_v_main}
\colorbox{RoyalBlue!15}{\primalV} &~~
    \max_{d \geq 0}  \; \mathbb{E}_{d(s,a)}[r(s,a)]-\alpha\f{d(s,a)}{d^O(s,a)}\\
   &\text{s.t} \; \textstyle \sum_{a\in\mathcal{A}} d(s,a)=(1-\gamma)d_0(s)+\gamma \sum_{(s',a') \in \S \times \A } d(s',a') p(s|s',a'), \; \forall s \in \S.
\end{aligned}  
\end{equation}
 Similarly, we consider the Lagrangian dual of \eqref{eq:primal_rl_v_main}, with dual variables $V(s)$ defined for all $s \in S$:
\begin{align}
\label{eq:dual-V}
\colorbox{Goldenrod!30}{\dualV} \;
\min_{V} {(1-\gamma)}\E{s \sim d_0}{V(s)} +{\alpha}\E{(s,a)\sim d^O}{f^*_p\left(\left[\mathcal{T}V(s,a)-V(s))\right]/
\alpha\right)},
\end{align}
where $f^*_p$ is a variant of $f^*$ defined as $f^*_p(x)=\max(0,{f'}^{-1}(x))(x)-{f(\max(0,{f'}^{-1}(x) ))}$. 
Such modification is to cope with the nonnegativity constraint $d(s, a) \geq 0$ in \primalV. Note that in both cases for \dualQ and \dualV, the optimal solution is the same as their primal formulations due to strong convexity. See Appendix~\ref{ap:dual_rl_review} for a detailed review, connections between Fenchel and Lagrangian duality, and discussion of computing $\pi^*$ from $V^*$ for the \dualV formulation.

\emph{Remarks.} The dual formulations have a few appealing properties. (a) They allow us to transform constrained distribution-matching problems into unconstrained forms w.r.t previously logged data. (b)  One can show that the gradient of \dualQ w.r.t $\pi$, when $Q$ is optimized for the inner problem, is the on-policy policy gradient computed by off-policy data~\citep{nachum2020reinforcement}. This property is key to relieving the instability or divergence issue in many off-policy learning algorithms~\citep{thrun1993issues,fu2019diagnosing,fujimoto2018addressing}. 

%% file: 5.1unification.tex
\section{A unified perspective on RL and IL through Duality }
\label{sec:unification}

In this section, we discuss how a number of recent RL and IL algorithms can be cast as dual-RL methods.  We restrict ourselves to demonstrating this equivalence on a subset of methods in \textit{offline} RL and IL settings from Table~\ref{tab:works} whose shortcomings we study via the unified viewpoint. In later sections, we present approaches for addressing these shortcomings in both RL and IL. For the interested reader, a complete discussion on Table~\ref{tab:works}, particularly how algorithms like implicit behavior cloning, CQL and OPOLO can be cast as dual-RL, can be found in Appendix~\ref{ap:complete_unification}. We further discuss the extension of dual-RL formulation to the online setting in Appendix~\ref{ap:offline_to_online}. All proofs are deferred to the appendix. 

\subsection{Dual Formulation for existing imitation learning algorithms}
\label{sec:connections_to_il}

We first consider the standard imitation learning setup where the agent is given a set of expert demonstrations, i.e. state-action trajectories, and does not have access to environment reward. We consider two possible offline IL settings --- 1) only expert demonstrations are available and 2) we additionally have access to suboptimal transitions from the environment. Intuitively, these suboptimal transitions should aid in better matching the expert behavior. 

\paragraph{a. Offline IL with Expert Data Only}
\label{main:offline_il_expert}
Imitation learning, or occupancy matching~\citep{ghasemipour2020divergence} is a direct consequence of the regularized RL problem (Eq.~\ref{eq:reg_rl_main}) when the reward is set to be $0$ uniformly across the state-action space and the regularization distribution and $d^O$ are set to be the expert visitation distribution $d^E$. The corresponding \dualQ under these conditions simplifies to:
\begin{equation}
\label{eq:dual-Q-imit}
\resizebox{0.94\textwidth}{!}{
    \colorbox{Goldenrod!30}{\dualQ}\,
    $\max_{\pi}\min_{Q} (1-\gamma)\E{d_0(s),\pi(a|s)}{Q(s,a)} +\alpha\E{s,a\sim d^E}{f^*\left(\left[\mathcal{T}^\pi_0 Q(s,a)-Q(s,a)\right]/
    \alpha\right)}$.
}
\end{equation}
Interestingly, this reduction directly leads us to IQLearn~\citep{garg2021iq}, which was derived using a change of variables in the form of an inverse backup operator. 
\begin{restatable}[]{proposition}{iqlearn}
\label{lemma:iqlearn}
 IQLearn~\citep{garg2021iq} is an instance of $\texttt{dual-Q}$ using the semi-gradient \footnote{For an overview of semi-gradient vs full-gradient methods please refer to Appendix~\ref{ap:semi_gradient_info}.} update rule with a (soft) Bellman operator, where $r(s,a)=0 \, \forall s \in \mathcal{S}, a \in \mathcal{A}$, $d^O=d^E$.
 \vspace{-2mm}
\end{restatable}

\paragraph{b. Offline IL with Additional Suboptimal Data}
Unfortunately, the dual-RL formulations above offer no way to naturally incorporate additional suboptimal data $d^S$. To remedy this, prior methods have relied on careful selection of the $f$-divergence and a \emph{coverage assumption} to craft an off-policy objective~\citep{zhu2020off,hoshino2022opirl,ma2022smodice,kim2022demodice,kim2022lobsdice}. More precisely, under the \emph{coverage assumption} that the suboptimal data visitation covers the expert visitation ($d^S>0$ wherever $d^E>0$)~\citep{ma2022smodice}, and with the KL divergence, we obtain the following simplification for the imitation objective:
\scalebox{0.98}{
$\begin{aligned}
  \kl{d(s,a)}{d^E(s,a)} &= \E{s,a\sim d(s,a)}{\log \frac{d(s,a)}{d^E(s,a)}}= \E{s,a\sim d(s,a)}{\log \frac{d(s,a)}{d^S(s,a)}+\log \frac{d^S(s,a)}{d^E(s,a)}}\\
  &= \E{s,a\sim d(s,a)}{\log \frac{d^S(s,a)}{d^E(s,a)}}+\kl{d(s,a)}{d^S(s,a)}.
\end{aligned}$}
 The final objective now resembles \primalQ when $r(s,a)= - \log\tfrac{d^S(s,a)}{d^E(s,a)}$ and correspondingly we obtain the following \dualQ problem using Eq.~\ref{eq:dual_q_original_equation}:
\begin{equation}
\label{eq:dual_q_il_coverage}
\resizebox{0.92\textwidth}{!}{
    \colorbox{Goldenrod!30}{\dualQ}\,
    $\max_{\pi(a|s)}\min_{Q(s,a)} (1-\gamma)\E{\rho_0(s),\pi(a|s)}{Q(s,a)}+\E{s,a\sim d^S}{f^*(\mathcal{T}^\pi_{r^{imit}}Q(s,a)-Q(s,a))}$,
}
\end{equation}
where $\Timit$ denote Bellman operator under the \emph{pseudo-reward} function $r^\text{imit}(s,a) = - \log\tfrac{d^S(s,a)}{d^E(s,a)}$. {This \reb{objective} also allows us to cast IL method OPOLO~\citep{zhu2020off} in the \dualQ framework (see Appendix~\ref{ap:off_policy_imitation_coverage})}. The pseudo-reward is a logarithmic density ratio learned using a discriminator and is later used for policy learning by optimizing Equation~\ref{eq:dual_q_il_coverage}. Density ratio learning is difficult \reb{in a limited data regime} as well as when the expert and suboptimal data share low coverage, and errors in learned discriminators can cascade for RL training and deteriorate the performance of output policy. We show how a simple modification to the imitation objective can allow us to relax the coverage assumption and propose a discriminator-free IL method that learns performant policies from arbitrary suboptimal data,  see Section~\ref{sec:new_il_method}.

\subsection{Dual Formulation for existing reinforcement learning algorithms}
\label{sec:connections_to_rl}
Now that we have seen how IL can be understood as a special case of the full regularized RL objective, we consider the full objective in Eq.~\eqref{eq:reg_rl_main}.  
Regularized policy learning, in its various forms~\citep{nachum2019algaedice, wu2019behavior}, is a natural objective for offline RL algorithms, preventing the policy from incorrectly deviating out-of-distribution by regularizing against the offline data visitation. 
\reb{However, implicit policy improvement (XQL~\citep{garg2023extreme}\reb{, IQL~\citep{kostrikov2021offline}}), one of the most successful classes of offline RL methods which uses in-distribution samples to pessimistically estimate the greedy improvement to the $Q$-function, has evaded connections to regularized policy optimization. }Proposition~\ref{thm:XQL} shows, perhaps surprisingly, that XQL can be cast as a dual of regularized policy learning, concretely as a \dualV problem.


\begin{restatable}[]{proposition}{xql}
\label{thm:XQL}
XQL is an instance of \dualV under the semi-gradient update rule,
where the $f$-divergence is the reverse Kullback-Liebler divergence,
and $d^O$ is the offline visitation distribution. 
\vspace{-2mm}
\end{restatable}

The success of XQL was attributed to the property that Gumbel distribution better models the Bellman errors~\citep{garg2023extreme}. Despite its decent performance, XQL is prone to training instability (see e.g., Figure~\ref{fig:offline-rl-analysis}), since the Gumbel loss is an exponential function that can produce large gradients during training. Situating XQL in the dual-RL framework allows us to propose a solution to the training instability problem, a new insight we discuss in Section~\ref{sec:new_rl_method}.

\paragraph{A consequence of unification in RL: } Offline RL can be broadly categorized in three approaches: 1) regularized policy learning, 2) pessimistic value learning 
e.g. CQL~\citep{kumar2020conservative}, \reb{ATAC~\citep{cheng2022adversarially}} and 3) implicit policy improvement algorithms (e.g. XQL). The latter two frameworks have seemingly been exceptions to the regularized policy learning formulation (e.g. Eq.~\eqref{eq:reg_rl_main}). In Proposition~\ref{thm:CQL}, we show that with an appropriate choice of $f$-divergence, CQL \reb{and ATAC} can be cast as a \dualQ problem. Overall, our results (Proposition~\ref{thm:CQL} and Proposition~\ref{thm:XQL}) are the first, to our knowledge, to bring together the latter two approaches, pessimistic value learning and implicit policy improvement as dual approaches to regularized policy learning. 


%% file: 5.2recoil.tex
\section{\textbf{\texttt{ReCOIL}}: Imitation Learning from Arbitrary Experience}
\label{sec:new_il_method}
As demonstrated in Section~\ref{sec:connections_to_il},
previous off-policy IL methods often rely on the coverage assumption and train a discriminator between the demonstration and the offline data to obtain a pseudo-reward $r^\text{imit}$. 
We propose \textbf{RE}laxed \textbf{C}overage for \textbf{O}ff-policy \textbf{I}mitation \textbf{L}earning~(\texttt{ReCOIL}),
an off-policy IL algorithm that relaxes the coverage assumption and eliminates the need for the discriminator. To achieve this, we consider an alternative way to leverage suboptimal data for imitation: matching two mixture distributions $d_\text{mix}^S :=\beta d(s,a) + (1-\beta) d^S(s,a)$ and $d_\text{mix}^{E,S} := \beta d^E(s,a)+(1-\beta)d^S(s,a)$, where $\beta \in (0,1) $ is a fixed hyperparameter. We consider the following problem in \primalQ form:
\begin{align}
\label{eq:primal_imitation_q_f_mixture}
\colorbox{RoyalBlue!15}{\primalQ} &~~~~~~~~~~~~~~~~~~~~~~~~~~~~\max_{d(s,a)} -\f{d^S_\text{mix}(s,a)}{d^{E,S}_\text{mix}(s,a)} \nonumber\\
   \text{s.t}~ \forall s \in \S, a\in \A,~~&\textstyle  d(s,a)=(1-\gamma)d_0(s)\pi(a|s)+\gamma \sum_{(s',a') \in \S \times \A } d(s',a') p(s|s',a') \pi(a|s).
\end{align}
This is a valid imitation learning formulation~\citep{ghasemipour2020divergence} since the global maximum of the objective is attained at $d=d^E$, irrespective of the suboptimal data distribution $d^S$. The primal formulation (Eq.~\ref{eq:primal_imitation_q_f_mixture}) deters offline learning, as it requires sampling from $d$ to estimate the $f$-divergence. We thus consider its dual formulation that allows us to derive an off-policy objective that only requires samples from the offline data. We term this formulation and associated approach \texttt{ReCOIL}. 

\begin{restatable}[]{theorem}{closerq}
\label{thm:recoilq_main}
(\textit{ReCOIL} objective) The \dualQ problem to the mixture distribution matching objective in Eq.~\ref{eq:primal_imitation_q_f_mixture} is given by:
\begin{equation}
   \resizebox{0.94\textwidth}{!}{ $\max_{\pi}\min_{Q}  \beta (1-\gamma)\E{d_0,\pi}{Q(s,a)} +\E{s,a\sim d_\text{mix}^{E,S}}{\reb{f^*}(\mathcal{T}^{\pi}_0 Q(s,a )-Q(s,a))}  -(1-\beta) \E{s,a\sim d^S}{\mathcal{T}^{\pi}_0 Q(s,a )-Q(s,a)}$}
\end{equation}
and recovers the same optimal policy $\pi^*$ as Eq.~\ref{eq:primal_imitation_q_f_mixture} since strong duality holds from Slater's conditions.
\end{restatable}
In other words, imitation learning can be solved by optimizing the unconstrained problem \texttt{ReCOIL} with arbitrary off-policy data, without the coverage assumption. Besides, as opposed to many previous algorithms, \texttt{ReCOIL} uses the Bellman operator $\mathcal{T}_0$ which does not need the pseudo-reward $r^\text{imit}$, making it discriminator-free. Although the pseudo-reward is not needed for training, \texttt{ReCOIL} allows for recovering the reward function using the learned $Q^*$, which corresponds to the intent of the expert. That is, $r(s,a)= Q^*(s,a)-\mathcal{T}^{\pi}_0(Q^*(s,a))$. 
Moreover, our method is generic to incorporate any $f$-divergence. {We also present the \dualV form for \texttt{ReCOIL} in Appendix~\ref{ap:closer} but defer its investigation for future work.}

\paragraph{A Bellman Consistent Energy-Based Model (EBM) View for ReCOIL}
Instantiating \texttt{ReCOIL} with $\chi^2$ Divergence, we \reb{present a simplified objective} (complete derivation in Appendix~\ref{ap:recoil_with_chi_square}) to:
\begin{equation}
\label{eq:recoil_intuition}
\resizebox{0.92\textwidth}{!}{
$\max_\pi \min_Q \beta (\E{d^S,\pi(a|s)}{Q(s,a)}-\E{d^E(s,a)}{Q(s,a)}) + \reb{0.25} \color{orange}
      \underbrace{\color{black}\E{s,a\sim \demix}{(\gamma Q(s',\pi(s'))-Q(s,a))^2}}_{\text{Bellman consistency}}$.
}
\end{equation}
\begin{wrapfigure}{r}{0.40\textwidth}
\centering
\begin{minipage}[t]{.40\textwidth}
\begin{algorithm}[H]
\algsetup{linenosize=\tiny}
\caption{\texttt{ReCOIL} (offline,$~~\chi^2$)}
\label{algo:recoil_algorithm_final}
\begin{algorithmic}[1]
    \STATE Initialize $Q_\phi$, $V_\theta$, and $\pi_\psi$, mixing ratio $\beta$, conservatism $\tau$, temperature $\alpha$
    \STATE  $\mathcal{D^S} =( s, a, s')$ be suboptimal dataset
    \STATE  $\mathcal{D^E} =( s, a, s')$ be expert dataset.
    \FOR{$t=1..T$ iterations}
        \STATE  Train $Q_\phi$ using $\min_{\phi}\mathcal{L}(\phi)$:
        \STATE Train $V_\theta$ using $\min_{\theta} \mathcal{J}(\theta)$
        \STATE Update $\pi_\psi$ via $\max_\psi \mathcal{M}(\psi)$:
        \vspace{0.025in}
    \ENDFOR
\end{algorithmic}
\end{algorithm}
\end{minipage}
\end{wrapfigure}
One can see that \texttt{ReCOIL} learns a score function $Q$ whose expected value is low over the suboptimal distribution but high over the expert distribution, while ensuring that $Q$ is Bellman consistent over the mixture.
The Bellman consistency is crucial to propagate the information of how to recover when the policy makes a mistake. 
The $Q$ value can be interpreted as a score as it is not representative of any expected return, and thus 
\texttt{ReCOIL} is an energy-based model with Bellman consistency. Figure~\ref{fig:recoil_main} in the appendix illustrates this intuition.

\paragraph{Practical Algorithm}  In Algorithm~\ref{algo:recoil_algorithm_final} we consider three parameterized functions $Q_\phi(s,a)$, $V_\theta(s)$ and $\pi_\psi$. Furthermore, we rely on Pearson $\chi^2$ divergence (Eq.~\ref{eq:recoil_intuition}) as it has been shown to lead to stable learning for the imitation setting~\citep{garg2021iq}. Our practical algorithm uses a \textit{semi-gradient} update that results in minimizing the following loss for $Q_\phi$:
\begin{equation}
    \label{eq:recoil_qphi_update}
    \resizebox{0.92\textwidth}{!}{$ \mathcal{L}(\phi) = \beta (\E{d^S,\pi(a|s)}{Q_\phi(s,a)}-\E{d^E(s,a)}{Q_\phi(s,a)}) + \reb{0.25}~\E{s,a\sim \demix}{(\gamma V_\theta(s')-Q_\phi(s,a))^2}$}.
\end{equation}
Naively maximizing $Q$ over $\pi$ in Eq.~\ref{eq:recoil_intuition} can result in the selection of out-of-distribution actions in the offline setting. To prevent extrapolation error in the offline setting, we rely on an implicit maximizer~\citep{garg2023extreme} that estimates the maximum over the $Q$-function conservatively with in-distribution samples. 
\begin{equation}
 \mathcal{J}(\theta) =\E{s,a\sim \demix}{\text{exp}((Q_\phi(s,a)-V_\theta(s))/\tau) + (Q_\phi(s,a)-V_\theta(s))/\tau)}.   
\end{equation}
Finally, the policy is extracted via advantage-weighted regression~\citep{peters2007reinforcement}: 
\begin{equation}
    \mathcal{M}(\psi) =\max_\psi \E{s,a\sim \demix}{\text{exp}(\alpha (Q_\phi(s,a)-V_\theta(s))) \log(\pi_\psi(a|s))}.
\end{equation}

%% file: 5.3DVL.tex
\section{\textbf{\texttt{$f$-DVL}} : Better Implicit Maximizers for Offline RL}
\label{sec:new_rl_method}
Proposition~\ref{thm:XQL} shows that XQL is a particular \dualV problem where the Gumbel loss is the conjugate $f^*_p$ corresponding to reverse KL divergence. This insight allows us to extend XQL by choosing different $f$-divergences, where the conjugate functions are more amenable to optimization. We further show that the proposed methods enjoy both improved performance and better training stability in Section~\ref{sec:result}. \looseness=-1

Implicit policy improvement algorithms iterate two steps alternately: 1) regress $Q(s,a)$ to $r(s,a)+\gamma V(s')$ for transition $(s, a, s')$ and 2) estimate $V(s) = \max_{a \in A} Q(s,a)$. The learned $Q$, $V$ functions can be used to extract a policy. As for the \dualV formulation, see Appendix~\ref{ap:recovering_policy}. Step 1) is akin to the \emph{policy evaluation} step of generalized policy iteration (GPI), and step 2) acts like the \emph{policy improvement} step without explicitly learning a policy $\pi(s) = \argmax_a Q(s,a)$. The crux is to conservatively estimate the maximum of $Q$ in step 2.
\looseness=-1

Consider a rewriting of \dualV with the temperature parameter $\lambda$ \reb{and a chosen surrogate function $\bar{f^*_p}$ that extends the domain of $f^*_p$ to $\mathbb{R}$. We discuss the need for a surrogate function below.}
\begin{equation}
\label{eq:implicit_maximization}
    \min_{V} (1-\lambda)\E{s \sim d^O}{V(s)} +\lambda\E{(s,a)\sim d^O}{
    \bar{f^*_p}\left(\bar{Q}(s,a)-V(s)\right)},
\end{equation}
where $\bar{Q}(s,a)$ denotes $\texttt{stop-gradient}(r(s,a)+\gamma \sum_{s'} p(s'|s,a)V(s'))$. 
Let $x$ be a random variable of distribution $D$. Eq~\eqref{eq:implicit_maximization} can be considered as a special instance of the following problem:\looseness=-1
\begin{equation}
\label{eq:implicit_maximization_general}
\min_{v} (1-\lambda) v + \lambda\E{x\sim D}{\bar{f^*_p}\left( x-v  \right)},
\end{equation}
where $x$ is analogous to $\bar{Q}$ and $v$ is analogous to $V$. As opposed to handcrafted choices~\citep{kostrikov2021offline, garg2023extreme},
we show through Proposition~\ref{thm:implicit_maximizer} below that objective~\eqref{eq:implicit_maximization_general}
naturally gives rise to a family of \emph{implicit maximizers} that estimates $\sup_{x \sim D} x$ as $\lambda \rightarrow 1$.

\begin{restatable}[]{proposition}{implicit}
\label{thm:implicit_maximizer}
Let $x$ be a real-valued random variable such that $\Pr(x > x^*) = 0$.
Let $v_\lambda$ be the solution of Problem~\eqref{eq:implicit_maximization_general}. It holds that $v_{\lambda_1}\le v_{\lambda_2}$, $\forall \, 0 < \lambda_1  < \lambda_2 < 1$. Further,
$\lim_{\lambda\to 1} v_\lambda = x^*$.
\end{restatable}
 Figure~\ref{fig:implicit_maximizers} provides an illustration for Proposition~\ref{thm:implicit_maximizer}. We propose a family of maximizers associated with different $f$-divergences and apply them to \dualV. We call the resulting methods $f$-\texttt{DVL} (Dual-V Learning).
 
\begin{wrapfigure}{R}{0.45\textwidth}
\centering
\begin{minipage}[t]{.45\textwidth}
\begin{algorithm}[H]
\algsetup{linenosize=\tiny}
\caption{\fdvl (Under Stochastic Dynamics)}
\label{algo:dvl_algorithm}
\begin{algorithmic}[1]
    \STATE Initialize $Q_\phi$, $V_\theta$, $\pi_\psi$, temperature $\alpha$, weight $\lambda$
    \STATE Let $\mathcal{D} =( s, a, r, s')$ be offline dataset
    \FOR{$t=1..T$ iterations}
        \STATE  Train $Q_\phi$ by minimizing:\\
        \resizebox{0.8\textwidth}{!}{
             $\E{s,a,s'\sim \mathcal{D}}{(Q_\phi(s,a) - (r(s,a) + \gamma V_\theta(s')))^2}.$
        }
        \STATE Train $V_\theta$ by minimizing Eq~\ref{eq:implicit_maximization}\\
        with surrogate $\bar{f^*_p}$
         \STATE Update $\pi_\psi$ by maximizing:
        \resizebox{0.8\textwidth}{!}{
              $\mathbb{E}_{s,a\sim\mathcal{D}}[e^{\alpha(Q_\phi(s, a) - V_\theta(s))} \log \pi_\psi(s|a)].$
        }
        \vspace{0.025in}
    \ENDFOR
\end{algorithmic}
\end{algorithm}
\end{minipage}
\vskip-10pt
\end{wrapfigure}
 \paragraph{Practical Considerations and Algorithm} A practical issue for optimizing Eq~\ref{eq:implicit_maximization} is that $f^*_p$ is not well defined over the entire domain of $\mathbb{R}$. To remedy this, we consider an extension of $f^*_p$ on  $\mathbb{R}$ that leads to the surrogate function denoted by $\bar{f^*_p}$: for (1) Total Variation: $f(x) = \tfrac{1}{2}|x - 1|,   \bar{f^*_p}(y) = \max(y, 0)$, (2) Pearson $\chi^2$ divergence: $f(x) = (x-1)^2, \; \bar{f^*_p}(y) = \max(\tfrac{1}{4}y^2 + y, 0)$. We defer the derivation of these surrogates to Appendix~\ref{ap:f_star_p_practical}. Recall that XQL uses the implicit maximizer associated with reverse KL divergence, where $f^*_p$ is exponential. Compared with XQL, our $\bar{f^*_p}$ functions are low-order polynomials and are thus stable for optimization. Algorithm~\ref{algo:dvl_algorithm} details the steps for \texttt{$f$-DVL}.
\looseness=-1

%% file: 6experiments.tex
\section{Experiments}
\label{sec:result}
Our experiments aim to answer the following four questions. \textbf{IL:} 1) How does \texttt{ReCOIL} perform and compare with previous offline IL methods?
2) Can \texttt{ReCOIL} accurately estimate the policy visitation distribution $d^\pi$ and the reward function/intent of the expert?
\textbf{RL:} 3) How does \texttt{$f$-DVL} perform and compare with previous offline RL methods?
4) Is the training of \texttt{$f$-DVL} more stable than XQL?

In order to circumvent the intricacies associated with exploration and direct our attention towards the intrinsic nature of dual RL formulation, we focus on the offline setting in this section, although the approaches can also be applied to online settings. 
We consider the locomotion and manipulation tasks from the D4RL benchmark~\citep{fu2020d4rl}, and report the results in Section~\ref{sec:expr_offline_il} and \ref{sec:expr_offline_rl}, respectively. 
For each algorithm, we train $7$ instances with different seeds and report their average return and standard derivation. Complete experiment details can be found in Appendix~\ref{ap:experiment_details}.

\subsection{Offline IL}
\label{sec:expr_offline_il}
\captionsetup[subfigure]{skip=0pt}
\begin{figure*}
\vspace{-20pt}
\centering
\begin{subfigure}{0.57\linewidth}
\includegraphics[width=\columnwidth]{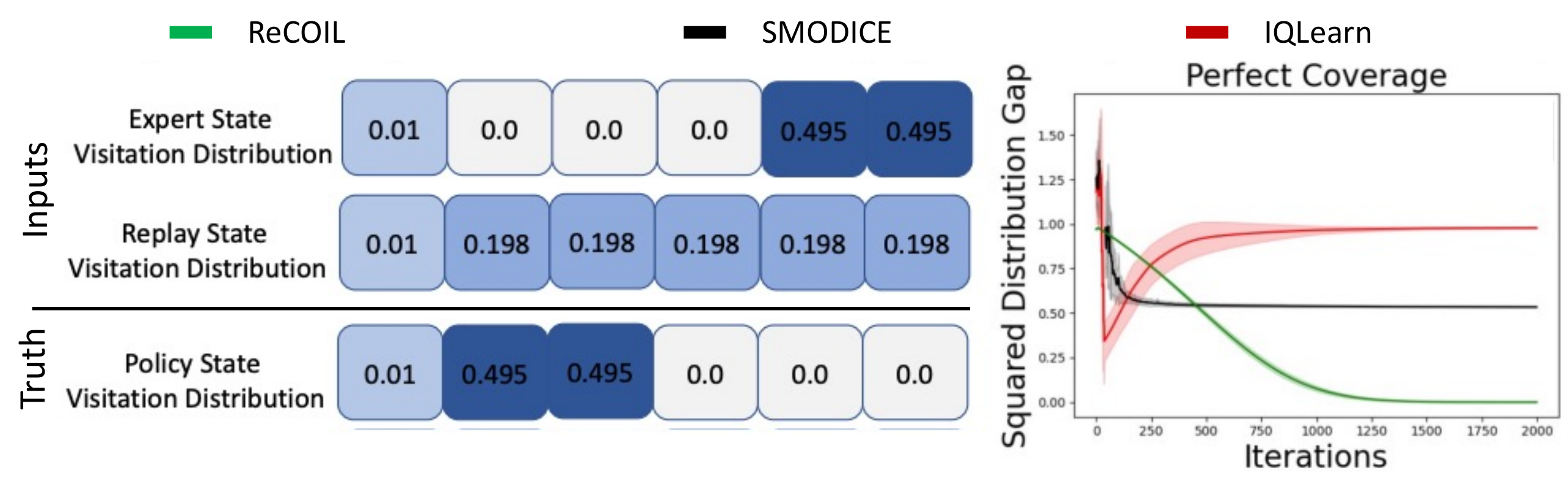}
\caption{}
 \label{fig:distribution_ratio_estimation1}
\end{subfigure}
\begin{subfigure}{0.42\linewidth}
\includegraphics[width=\columnwidth]{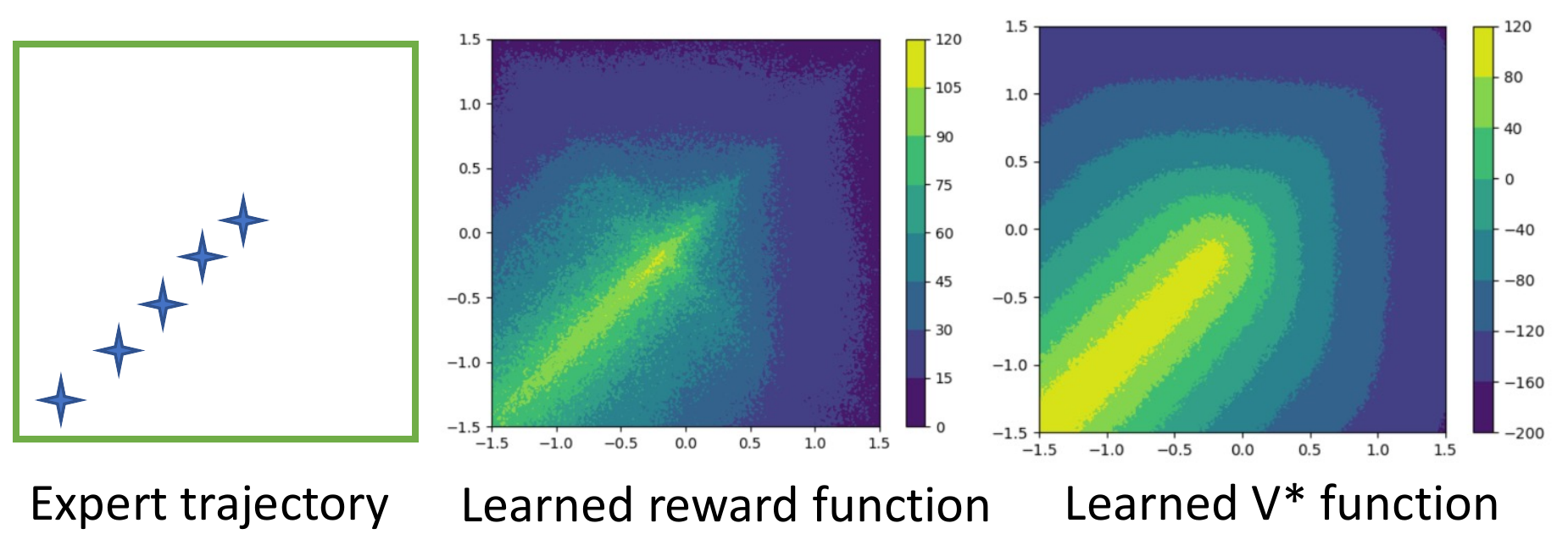}
\caption{}
\label{fig:recovering_rewards}
\end{subfigure}
\vspace{-1.8\baselineskip}
\caption{(a)
\reb{[Left] shows an MDP that starts at the leftmost state and transitions to one of the five absorbing states on the right. Under the given expert and replay/offline visitation we study if a prespecified policy's visitation can be inferred whose ground truth visitation is known [Right] shows MSE error plots with policy's ground truth visitation where }\texttt{ReCOIL} perfectly infers $d^\pi$ whereas a method that only relies on expert data or the replay data with the coverage assumption fails. Results averaged over 100 seeds. \reb{More details in Appendix~\ref{ap:density_ratio_estimation}} (b) Recovered $R$ and $V^*$ on a simple grid-world environment by \texttt{ReCOIL}.}
\end{figure*}
\paragraph{Benchmark Comparisons}
For every task, our agent is given 1 expert demonstration 
and a set of suboptimal transitions, both
extracted from the D4RL datasets. 
We follow the construction of suboptimal dataset in SMODICE~\citep{ma2022smodice}. For locomotion tasks, the suboptimal dataset consists of 1 million transitions of the random or medium D4RL datasets and $200$ expert demonstrations,
which we label as 'random+expert' and 'medium+expert', respectively. 
We also consider suboptimal datasets mixed with only $30$ expert demonstrations, which are called \texttt{random+few-expert} and \texttt{medium+few-expert} to simulate \reb{a more difficult setting}. Similarly, we construct datasets for the manipulation tasks. More details in Appendix~\ref{ap:offline_il_experiment_details}. 

\begin{table}[h]
    \centering
    \vspace{5pt}
    \resizebox{1.00\textwidth}{!}{
    \begin{tabular}{c|c|c|c|c|c|c|c|c||c}
    \toprule
    \multirow{1}{*}{Suboptimal Dataset} & Env & RCE & ORIL & SMODICE & BC (only expert data) & BC (full dataset)& IQ-Learn (offline) & \texttt{ReCOIL}  & Expert  \\
    \midrule
    \multirow{2}{*}{random+}& hopper&51.41$\pm$38.63&73.93$\pm$11.06&\highlight{101.61$\pm$7.69}&4.52$\pm$1.42&5.64$\pm$4.83&1.85 $\pm$2.19&\highlight{108.18$\pm$3.28}&111.33\\
    & halfcheetah&64.19$\pm$11.06&60.49$\pm$3.53& \highlight{80.16$\pm$7.30}&2.2$\pm$0.01&2.25$\pm$0.00&4.83$\pm$7.99&\highlight{80.20$\pm$6.61}&88.83\\
    \multirow{1}{*}{expert}& walker2d&20.90$\pm$26.80&2.86$\pm$3.39&\highlight{105.86$\pm$3.47}&0.86$\pm$0.61&0.91$\pm$0.5&0.57$\pm$0.09&\highlight{102.16$\pm$7.19}& 106.92\\
    & ant&105.38$\pm$14.15&73.67$\pm$12.69&\highlight{126.78$\pm$5.12}&5.17$\pm$5.43&30.66$\pm$1.35&42.23$\pm$20.05&\highlight{126.74$\pm$4.63}&130.75\\
    \midrule
        \multirow{2}{*}{random+}& hopper&25.31$\pm$18.97&42.04$\pm$13.76&60.11$\pm$18.28&4.84$\pm$3.83&3.0$\pm$0.54&1.37 $\pm$1.23&\highlight{97.85$\pm$17.89}&111.33\\
    & halfcheetah& 2.99$\pm$1.07&2.84$\pm$5.52& 2.28$\pm$0.62&-0.93$\pm$0.35&2.24$\pm$0.01&1.14$\pm$1.94& \highlight{76.92$\pm$7.53}&88.83\\
    \multirow{1}{*}{few-expert}& walker2d&40.49$\pm$26.52& 3.22$\pm$3.29&\highlight{107.18$\pm$1.87}&0.98$\pm$0.83&0.74$\pm$0.20&0.39$\pm$0.27&83.23$\pm$19.00& 106.92\\
    & ant&\highlight{67.62$\pm$15.81}&25.41 $\pm$ 8.58&-6.10$\pm$7.85&0.91$\pm$3.93&35.38$\pm$2.66&32.99$\pm$3.12&\highlight{67.14$\pm$ 8.30}&130.75\\
    \midrule
    \multirow{2}{*}{medium+}& hopper&58.71$\pm$34.06&61.68$\pm$7.61&49.74$\pm$3.62&16.09$\pm$12.80&59.25$\pm$3.71&12.90$\pm$24.00&\highlight{88.51$\pm$16.73}&111.33\\
    & halfcheetah&65.14$\pm$13.82& 54.66$\pm$0.88&59.50$\pm$0.82&-1.79$\pm$0.22&42.45$\pm$ 0.42&25.67$\pm$20.82&\highlight{81.15$\pm$2.84}&88.83\\
    \multirow{1}{*}{expert}& walker2d&\highlight{96.24$\pm$14.04}&8.19$\pm$7.70&2.62$\pm$0.93&2.43$\pm$1.82&72.76$\pm$3.82&59.37$\pm$30.14&\highlight{108.54$\pm$1.81}& 106.92\\
    & ant&\highlight{86.14$\pm$38.59}&102.74$\pm$6.63&104.95$\pm$6.43&0.86$\pm$7.42&95.47$\pm$10.37&37.17$\pm$41.15&\highlight{120.36$\pm$7.67}&130.75 \\
    \midrule
    \multirow{2}{*}{medium}& hopper&\highlight{66.15$\pm$35.16}&17.40$\pm$15.15&47.61$\pm$7.08& 7.37$\pm$1.13&46.87$\pm$5.31&11.05$\pm$20.59&\highlight{50.01$\pm$10.36}&111.33\\
    & halfcheetah&\highlight{61.14$\pm$18.31}& 43.24$\pm$0.75&46.45$\pm$3.12&-1.15$\pm$0.06&42.21$\pm$0.06&26.27$\pm$20.24&\highlight{75.96$\pm$4.54}&88.83\\
    \multirow{1}{*}{few-expert}& walker2d&\highlight{85.28$\pm$34.90}&6.81$\pm$6.76&6.00$\pm$6.69&2.02$\pm$0.72&70.42$\pm$2.86&73.30$\pm$2.85&\highlight{91.25$\pm$17.63}& 106.92\\
    & ant&\highlight{67.95$\pm$36.78}&81.53$\pm$8.618&81.53$\pm$8.618&-10.45$\pm$1.63&81.63$\pm$6.67&35.12$\pm$50.56&\highlight{110.38$\pm$10.96}&130.75\\
    \midrule
    \multirow{4}{*}{cloned+expert}& pen&19.60$\pm$11.40&-3.10$\pm$0.40&-3.36$\pm$0.71&13.95$\pm$11.04&34.94$\pm$11.10&2.18$\pm$8.75&\highlight{95.04$\pm$4.48}&106.42\\
    & door&0.08$\pm$ 0.15&-0.33$\pm$0.01& 0.25$\pm$ 0.54&-0.22$\pm$0.05&0.011$\pm$0.00&0.07$\pm$0.02&\highlight{102.75$\pm$4.05}&103.94\\
    & hammer&1.95$\pm$3.89&0.25$\pm$ 0.01&0.15$\pm$ 0.078&2.41$\pm$4.48&5.45$\pm$ 7.84&0.27$\pm$0.02&\highlight{95.77$\pm$17.90}&125.71\\
    & relocate&-0.25$\pm$0.04&-0.29$\pm$0.01&1.75$\pm$3.85&-0.17$\pm$0.04&-0.24$\pm$ 0.01&-0.1$\pm$0.12&\highlight{67.43$\pm$14.60}&118.39\\
    \midrule
        \multirow{4}{*}{human+expert}& pen&17.81$\pm$5.91&-3.38$\pm$2.29&-2.20$\pm$2.40&13.83$\pm$10.76&90.76$\pm$25.09&14.29$\pm$28.82&\highlight{103.72$\pm$2.90}&106.42\\
    & door& -0.05$\pm$0.05&-0.33$\pm$0.01& -0.20$\pm$ 0.11&-0.03$\pm$0.05&103.71$\pm$1.22&5.6$\pm$7.29& \highlight{104.70$\pm$0.55}&103.94\\
    & hammer&5.00$\pm$5.64& 1.89$\pm$0.70&-0.07$\pm$0.39&0.18$\pm$0.14&122.61$\pm$4.85&5.32$\pm$1.38&\highlight{125.19$\pm$3.29}&125.71\\
    & relocate&0.02$\pm$0.10& -0.29$\pm$0.01 &-0.16$\pm$0.04&-0.13$\pm$0.11&81.19$\pm$7.73&-0.04$\pm$0.22&\highlight{91.98$\pm$ 2.89}&118.39\\
    \midrule
    \multirow{1}{*}{partial+expert}& kitchen&6.875$\pm$9.24&0.00$\pm$0.00&39.16$\pm$ 1.17&2.5$\pm$5.0&45.5$\pm$1.87&0.0$\pm$0.0&\highlight{60.0$\pm$5.70}&75.0\\
    \midrule
    \multirow{1}{*}{mixed+expert}& 
    kitchen&1.66$\pm$2.35& 0.00$\pm$0.00&42.5$\pm$2.04&2.2$\pm$3.8&42.1$\pm$1.12&0.0$\pm$0.0&\highlight{52.0$\pm$1.0}&75.0\\
    \bottomrule
    \end{tabular}
}
\vspace{-3mm}
\caption{The normalized return obtained by different offline IL methods trained on the D4RL suboptimal datasets with 
\reb{1 expert trajectory}. \reb{Methods with avg. perf within the std-dev of the top performing method is highlighted.}} 
    \vspace{5pt}
    \label{table:offline_il_results} 
\end{table}

We compare \texttt{ReCOIL} against recent offline IL methods 
RCE~\citep{eysenbach2021replacing}, IQLearn~\citep{garg2021iq}, SMODICE~\citep{ma2022smodice}, ORIL~\citep{zolna2020offline} and behavior cloning. We do not compare to DEMODICE~\citep{kim2022demodice} and ValueDICE~\citep{kostrikov2019imitation} as SMODICE was shown to outperform DEMODICE in \citet{ma2022smodice} and IQLearn was shown to outperform ValueDICE in~\citep{garg2021iq} on the same environments. 
Both SMODICE and ORIL require learning a discriminator, and SMODICE relies on the coverage assumption. 
RCE also uses a recursive discriminator to test the proximity of the policy visitations to successful examples.
In contrast, \texttt{ReCOIL} is discriminator-free and does not need this coverage assumption. 
Table~\ref{table:offline_il_results} shows that \texttt{ReCOIL} strongly outperforms the baselines in most environments. SMODICE exhibits poor performance in cases when the combined offline dataset has \reb{few expert samples} (\texttt{random+few-expert}) or where the discriminator can easily overfit (high-dimensional environments like dextrous manipulation). \looseness=-1


\paragraph {Estimation of the Policy Visitation Distribution and Reward Recovery}
Correctly estimating a given policy's visitation distribution $d^\pi$ is key to testing its closeness to the expert visitation.
 $d^\pi$ can be computed via Eq~\eqref{ap:optimal_distribution_ratio} (appendix). Figure~\ref{fig:distribution_ratio_estimation1} and Figure~\ref{fig:distribution_ratio_estimation2} show that \texttt{ReCOIL}  can estimate $d^\pi$ more accurately than \reb{SMODICE~\citep{ma2022smodice}} which relies on coverage assumption and IQLearn~\citep{garg2021iq} which only utilizes expert data. This empirically validates our hypothesis that a learned discriminator with low coverage can lead to poor performance of the downstream policy. Further, Figure~\ref{fig:recovering_rewards} shows the reward function recovered by \texttt{ReCOIL} for a simple grid-world task. For Hopper and Walker, we respectively observe a Pearson correlation of \textbf{0.98} and \textbf{0.92} between the recovered reward with the ground truth. See more details in \reb{Appendix~\ref{ap:density_ratio_estimation} and \ref{ap:recovering_rewards}}.

\subsection{Offline RL}
\label{sec:expr_offline_rl}
\begin{wrapfigure}{r}{0.35\textwidth}
\vspace{-10pt}
\begin{center}
\includegraphics[width=1.0\linewidth]{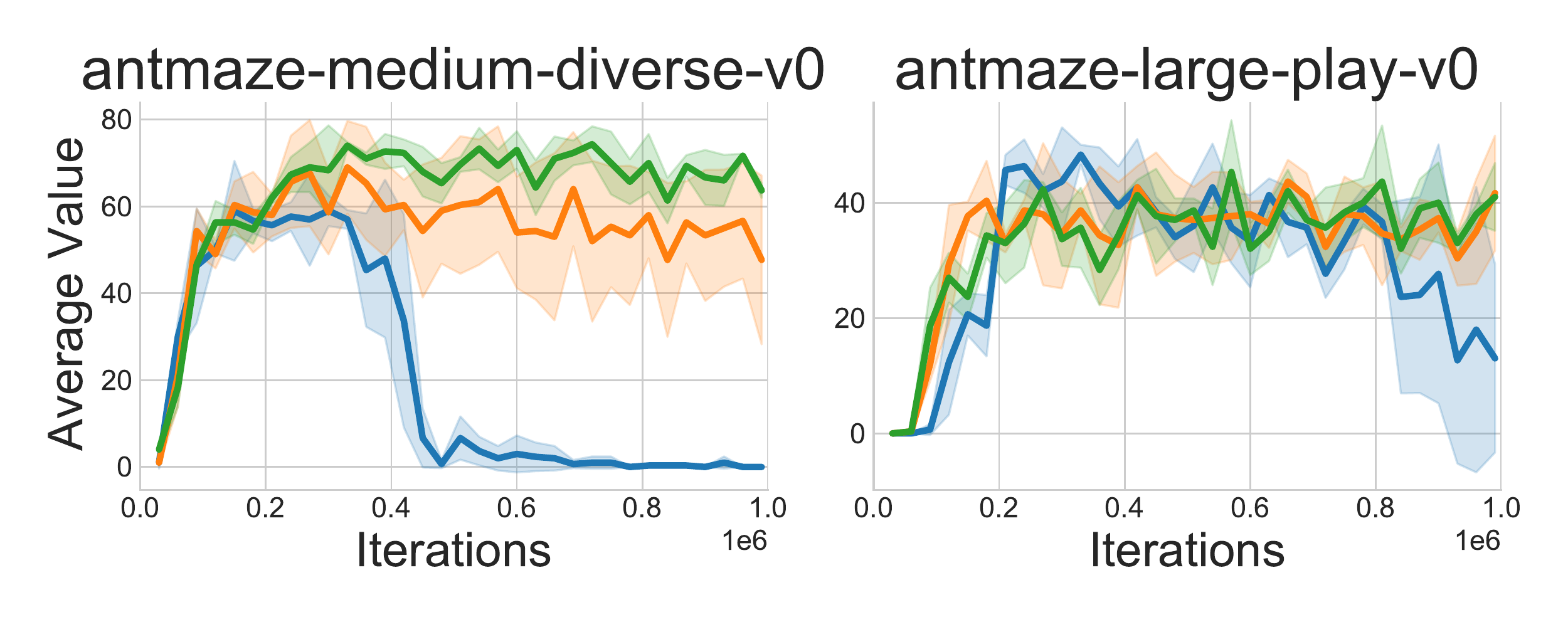}
      \\
\vspace{-2.0mm}
\includegraphics[width=1.0\linewidth]{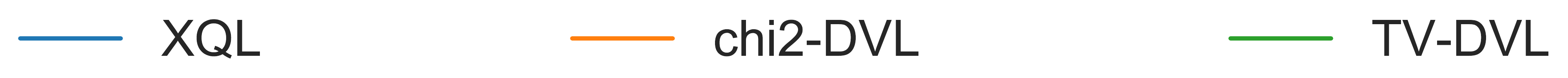}  
    \vspace{-7.5mm}
\end{center}
\caption{\small XQL training diverges due to the numerical instability of its loss function. $f$-DVL fixes this problem by using more well-behaved $f$-divergences.}
\label{fig:offline-rl-analysis}
\end{wrapfigure}
\paragraph{Benchmark Comparison}
Table~\ref{tab:d4rl} shows that \fdvl outperforms XQL and other prior offline RL methods~\citep{chen2021decision,kumar2019stabilizing,kumar2020conservative,kostrikov2021offline,fujimoto2021minimalist}
on a broad range of continuous control tasks. We note an inconsistency between our reproduced XQL results and
the results reported in the original paper: 
their results were reported by taking the best average return during training as opposed to the standard practice of taking the average of the last iterate performance across different seeds at 1 million gradient steps. Such inconsistency can be validated by comparing their training plots and reported results (Fig 11 and Table 1 in  \cite{garg2023extreme}). XQL(r) shows the results for XQL under the standard evaluation protocol.

\begin{table}[t]
\vspace{-20pt}
\centering
\scriptsize
\setlength\tabcolsep{3pt}
\renewcommand{\arraystretch}{1.0}
\resizebox{0.8\textwidth}{!}{
\begin{tabular}{l||rrrrrrr|cc}
\multicolumn{1}{c||}{Dataset} & BC & 10\%BC & DT  & TD3+BC & CQL & IQL &  {XQL}(r) & \bf{$f$-DVL} ($\chi^2$) & \bf{$f$-DVL} (TV)\\\hline
halfcheetah-medium-v2 & 42.6 & 42.5 & 42.6  & \highlight{48.3} &44.0 & \highlight{47.4} & 47.4 & \highlight{47.7} & 47.5\\ 
hopper-medium-v2 & 52.9 & 56.9 & 67.6 &  59.3 & 58.5 & 66.3  & \highlight{68.5} & 63.0& 64.1\\ 
walker2d-medium-v2 & 75.3 &75.0 &74.0 & \highlight{83.7} & 72.5 & 78.3  & 81.4 & 80.0 & 81.5\\ 
halfcheetah-medium-replay-v2 & 36.6 & 40.6 & 36.6 &  \highlight{44.6} & \highlight{45.5} & 44.2  &44.1  & 42.9 & \highlight{44.7}\\ 
hopper-medium-replay-v2 & 18.1 & 75.9 & 82.7 & 60.9 & 95.0 &  94.7  &95.1  & 90.7&\highlight{98.0}\\ 
walker2d-medium-replay-v2 & 26.0 & 62.5 & 66.6 &  \highlight{81.8} & 77.2 & 73.9  &58.0 &52.1&68.7  \\ 
halfcheetah-medium-expert-v2 & 55.2 & \highlight{92.9} & 86.8 &   90.7 & 91.6 & 86.7 &  90.8 &89.3 &91.2 \\ 
hopper-medium-expert-v2 &52.5 & \highlight{110.9} & {107.6}  & 98.0 & 105.4 & 91.5  & 94.0 &{105.8} & 93.3 \\ 
walker2d-medium-expert-v2 & 107.5 & {109.0} & 108.1  & \highlight{110.1} & 108.8 & \highlight{109.6}  & \highlight{110.1} &\highlight{110.1} & \highlight{109.6}\\ \hline 
antmaze-umaze-v0 & 54.6 & 62.8 & 59.2 & 78.6 & 74.0 & \highlight{87.5}  & 47.7&83.7 &\highlight{87.7} \\ 
antmaze-umaze-diverse-v0 & 45.6 & 50.2 & 53.0 &  71.4 & \highlight{84.0} & 62.2  & 51.7   &  50.4&48.4\\  
antmaze-medium-play-v0 & 0.0 & 5.4 & 0.0 &  10.6 & 61.2 & \highlight{71.2}  &  31.2 & 56.7&\highlight{71.0} \\  
antmaze-medium-diverse-v0 & 0.0 & 9.8 & 0.0 &  3.0 & 53.7 & \highlight{70.0}  &  0.0& 48.2& 60.2\\  
antmaze-large-play-v0 &0.0 &0.0 &0.0 &0.2 &15.8 & 39.6  &10.7 &36.0 & \highlight{41.7}\\  
antmaze-large-diverse-v0 & 0.0 &6.0 & 0.0 & 0.0 & 14.9 & \highlight{47.5} &31.28 &44.5 & 39.3\\ \hline 
kitchen-complete-v0 & 65.0 & -  & -  & - &43.8 & 62.5  & 56.7 &\highlight{67.5} &61.3 \\ 
kitchen-partial-v0 &38.0 & - & - & - & 49.8 & 46.3   & 48.6 &58.8 & \highlight{70.0} \\ 
kitchen-mixed-v0 & 51.5 & - & - &  - & 51.0 & 51.0 & 40.4  & \highlight{53.75}&52.5 \\ \hline \hline
\end{tabular}
}
\vspace{-4pt}
\caption{The normalized return of offline RL methods on D4RL tasks. XQL(r) denotes the results obtained under the standard evaluation protocol. Results aggregated over 7 seeds. Highlighted results are within one performance point of the best-performing algorithm.}
\vspace{-5pt}
\label{tab:d4rl}
\end{table}

\paragraph{Training Stability} As pointed out by the authors, the exponential loss function of XQL causes numerical instabilities during optimization. As 
discussed in Section~\ref{sec:new_rl_method}, this is a by-product of reverse KL divergence. Fig.~\ref{fig:offline-rl-analysis} confirms that this is fixed by $f$-\texttt{DVL} by using other $f$-divergences with more stable loss functions. Additionally, Fig.~\ref{fig:online_rl} demonstrates that $f$-\texttt{DVL} \reb{is competitive to}
XQL and SAC in the online setting as well. See Appendix~\ref{ap:experiment_details} for additional experimental details.


\subsection{Additional Experiments} 
\begin{wrapfigure}{r}{0.40\textwidth}
\vspace{-10pt}
\begin{center}
\includegraphics[width=1.0\linewidth]{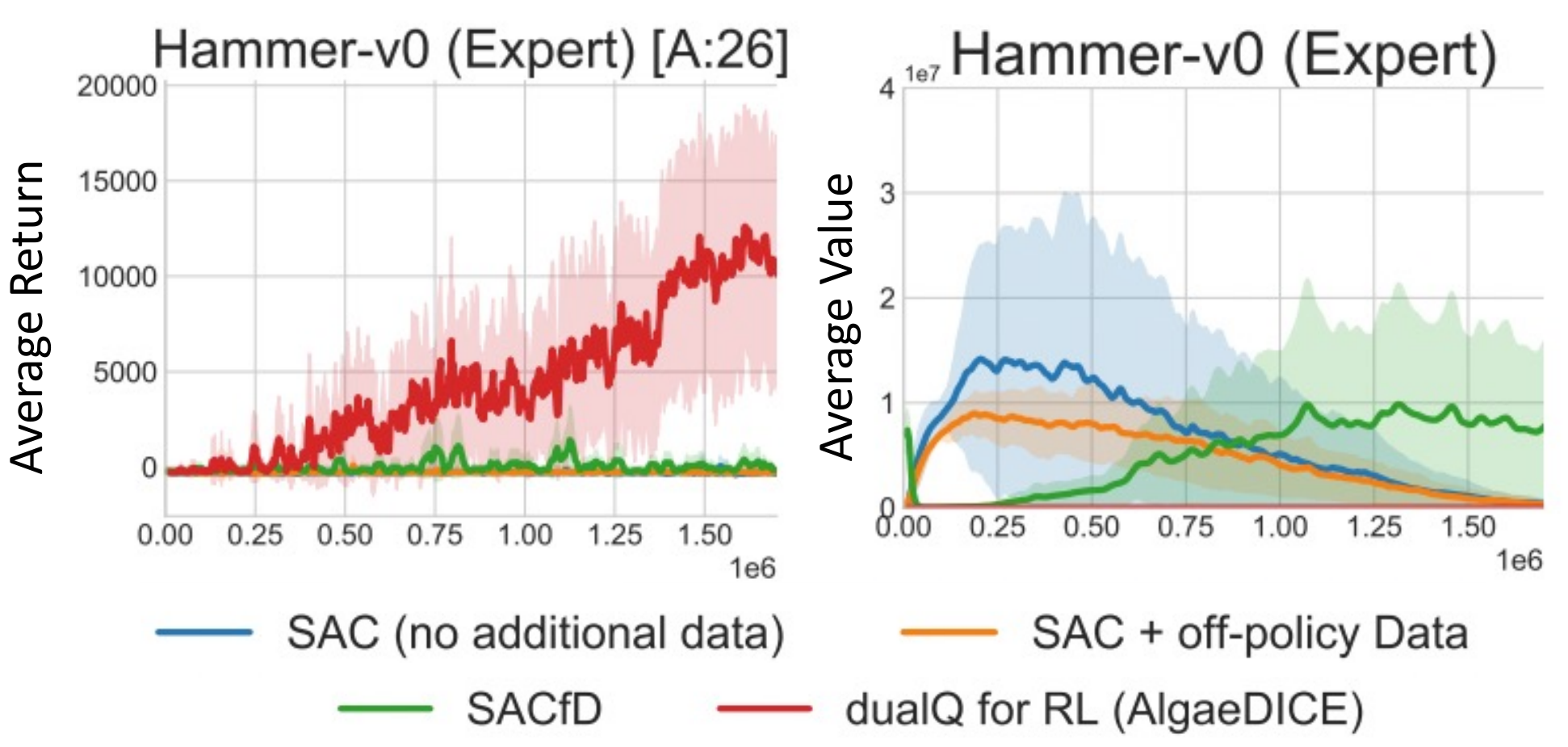}  
    \vspace{-7.5mm}
\end{center}
\caption{\small \reb{Augmenting SAC with expert data at the start of training destabilizes value function learning (r), but dual-RL approaches can make effective use of the additional data to learn performant policy (l).}}
\label{fig:sac_overestimation_example}
\vspace{-5.0mm}
\end{wrapfigure}
We conduct additional experiments in Appendix~\ref{ap:additional_experiments}. We further demonstrate a) when incorporating off-policy data in online training, traditional ADP-based methods \reb{can} suffer from the over-estimation of value functions, and the performance gain is limited, whereas dual-RL methods can leverage the same
data to achieve better performance (\reb{Figure~\ref{fig:sac_overestimation_example} and} Appendix~\ref{ap:dual_rl_are_better});
 b) the reward functions learned by \texttt{ReCOIL} are of high quality (Appendix~\ref{ap:recovering_rewards});
 c) the hyperparameter ablation for \fdvl (Appendix~\ref{ap:dvl_sensitivity}) and qualitative results for \texttt{ReCOIL} (Appendix~\ref{ap:qualitative_comparison}).

%% file: 7conclusion.tex
\section{Conclusion}
\label{sec:conclusion}
Our work unifies a significant number of recent developments in RL and IL. The insights gleaned from the unification allow us to identify key gaps and subsequently improve performance and training stability in imitation learning and reinforcement learning. Leveraging this unification, we 
propose 1) a family of \textit{stable offline RL methods} $f$-\texttt{DVL} relying on implicit value function maximization, 2) \texttt{ReCOIL}, a general \textit{off-policy IL method} to learn from arbitrary suboptimal data while being discriminator-free. We show that $f$-\texttt{DVL} and \texttt{ReCOIL} both outperform previous methods in online/offline RL and offline IL domains, respectively.  We demonstrate that Dual-RL algorithms have great potential for developing performant algorithms and warrant further study. We direct interested readers to read our Appendix containing more insights into Dual-RL and potential directions for improvements and new algorithms.
\looseness=-1

%% file: 8acknowledgements.tex
\section*{Acknowledgements}
We thank  Ahmed Touati,  Ben Eysenbach, Siddhant Agarwal, and ICLR reviewers for valuable feedback on this work. This work has taken place in the Safe, Correct, and Aligned Learning and Robotics Lab (SCALAR) at The University of Massachusetts Amherst and Machine Intelligence through Decision-making and Interaction (MIDI) Lab at The University of Texas at Austin. SCALAR research is supported in part by the NSF (IIS-2323384), AFOSR (FA9550-20-1-0077), and ARO (78372-CS, W911NF-19-2-0333), and the Center for AI Safety (CAIS).  This research was also sponsored by the Army Research Office under Cooperative Agreement Number W911NF-19-2-0333. HS and AZ are funded in part by a sponsored research agreement with Cisco Systems Inc. The views and conclusions contained in this document are those of the authors and should not be interpreted as representing the official policies, either expressed or implied, of the Army Research Office or the U.S. Government. The U.S. Government is authorized to reproduce and distribute reprints for Government purposes notwithstanding any copyright notation herein.



%% file: 9appendix.tex
\newpage
\appendix
\onecolumn

\addcontentsline{toc}{section}{Appendix} 
\part{Appendix} 
\parttoc 




\section{Limitations and Negative Societal Impacts}

\textbf{Limitations}: Notably, none of the supervised learning approaches, like Upside-Down RL~\citep{schmidhuber2019reinforcement}, DT~\citep{chen2021decision,zheng2022online} and RvS~\citep{emmons2021rvs}, can be cast as instances of this framework. 
The main reason is that there is no return optimization term in the supervised learning loss for those approaches.
Second, our framework is constrained to the regularized RL setting where the regularization coefficient $\alpha > 0$ (see Eq.~\ref{eq:reg_rl_main} ).
Finally, our framework considers a static regularization distribution $d^O$, and we believe that regularization with dynamically changing distributions, like those in on-policy methods (PPO~\citep{schulman2017proximal}, TRPO~\citep{schulman2015trust}, etc) requires a more careful theoretical treatment. Another limitation of the paper is the assumption that the expert demonstrations used in the imitation learning process are always of high quality and provide the desired behavior. In practice, obtaining high-quality demonstrations can be challenging, especially in complex environments where the behavior of the expert is not always clear. The performance of the proposed approach could be limited in cases where the expert demonstrations are of poor quality or where the behavior of the expert does not correspond to the desired behavior. 


\textbf{Negative Societal Impacts}: As machine learning algorithms continue to grow in sophistication, it is important to consider the potential risks and harms associated with their use. One such area of concern is imitation learning, which involves training a model to imitate a desired behavior by providing it with examples of that behavior. However, this approach can be problematic if the demonstration data includes harmful behaviors, whether intentional or not. Even in cases where the demonstration data is of high quality and desirable behavior is learned, the algorithm may still fall short of providing sufficient guarantees of performance. In high-stakes domains, the use of such algorithms without appropriate safety checks on learned behaviors could lead to serious consequences. As such, it is crucial to carefully consider the potential risks and benefits of imitation learning, and to develop strategies for ensuring safe and effective use of these algorithms in real-world application

\section{Dual Reinforcement Learning}
\label{ap:theory}

\subsection{A Review of Dual-RL}
\label{ap:dual_rl_review}
In this section, we aim to give a self-contained review for Dual Reinforcement Learning. For a more thorough read, refer to~\citep{nachum2020reinforcement}. Our proofs will take a different approach than~\citep{nachum2020reinforcement} which we believe leads to easier exposition and understanding.

\subsubsection{Convex conjugates and $f$-divergence}
\label{ap:convex_conjugation}
We first review the basics of duality in reinforcement learning. Let $f:\R_+\to \mathbb{R}$ be a convex function. The convex conjugate $f^*: \R_+ \rightarrow \R$ of $f$ is defined by:
\begin{equation}
    \label{eq:cx_conjugate_def}
    f^*(y)= \text{sup}_{x \in \R_+}[xy -f(x)].
\end{equation}
The convex conjugates have the important property that $f^*$ is also convex and the convex conjugate of $f^*$ retrieves back the original function $f$. 
We also note an important relation regarding $f$ and $f^*$: $(f^*)^{'}=(f')^{-1}$, where the $'$ notation denotes first derivative.

Going forward, we would be dealing extensively with $f$-divergences. Informally, $f$-divergences~\citep{renyi1961measures} are a measure of distance between two probability distributions. Here's a more formal definition:

Let $P$ and $Q$ be two probability distributions over a space $\mathcal{Z}$ such that $P$ is absolutely continuous with respect to $Q$ \footnote{
Let $z$ denote the random variable. 
For any measurable set $Z \subseteq \mathcal{Z}$, $Q(z \in Z) = 0$ implies $P(z \in Z) = 0$.}. 
For a function $f: \R_+ \to \mathbb{R}$ that is a convex lower semi-continuous and $f(1)=0$,
the $f$-divergence of $P$ from $Q$ is  
\begin{equation}
\f{P}{Q}=\mathbb{E}_{z\sim Q}\left[f\left(\frac{P(z)}{Q(z)}\right)\right].
 \end{equation}
Table~\ref{tbl:div} lists some common $f$-divergences with their generator functions $f$ and the conjugate functions $f^*$.

\begin{table}[t]
	\centering
	\vskip5pt
	\def\arraystretch{1.5}
	\begin{tabular}{l|c|c}
        \toprule
		Divergence Name & Generator $f(x)$ & Conjugate $f^*(y)$  \\ \hline
		Reverse KL & $x\log{x}$ & $e^{(y-1)}$  \\
		Squared Hellinger & $(\sqrt{x} - 1)^2$ & $\frac{y}{1-y}$ \\ 
		Pearson $\chi ^{2}$ & $(x-1)^2$ & $y + \frac{y^2}{4}$ \\
		Total Variation & $\frac{1}{2} | x-1| $ & $y$ if $y \in [-\frac{1}{2}, \frac{1}{2}]$ otherwise $\infty$ \tablefootnote{For our derivations, we will consider smooth extensions of TV conjugate function} \\
		Jensen-Shannon & $-(x+1)\log(\frac{x+1}{2}) + x \log{x}$ & $-\log{(2- e^y)}$ \\
  \bottomrule
	\end{tabular}
 \vskip5pt
	\caption{List of common $f$-divergences.} 
	\label{tbl:div}
\end{table}

\subsubsection{An Overview of Reinforcement Learning via Lagrangian Duality}
\label{ap:dual_rl}


\input{10appendix-dual-lagrangian}

\subsubsection{Deriving \texttt{dual-Q}}
\label{app:derive_dual_q}

We again consider the RL problem as a maximization of a convex program for estimating policy performance $J(\pi)$ by considering optimization over \textit{achievable} state-action visitations (i.e $\max_\pi J(\pi)$):
\begin{align}
\label{eq:primal_rl}
    &\max_{\pi} \bigg[ \max_{d\ge0} \mathbb{E}_{d(s,a)}[r(s,a)]-\textcolor{red}{\alpha}\f{d(s,a)}{d^O(s,a)}\\
    &\text{s.t}~~d(s,a)=(1-\gamma)d_0(s).\pi(a|s)+\gamma \sum_{s',a'} d(s',a')p(s|s',a')\pi(a|s) \bigg],
\end{align}
where $\alpha$ allows us to weigh policy improvement against conservatism from staying close to the state-action distribution $d^O$. \reb{We will assume that $d^O$ has non-zero coverage over $\mathcal{S} \times \mathcal{A}$ for the derivation below.}

A careful reader may notice that the inner problem is overconstrained and overparameterized. The solution to the inner maximization problem with respect to $d$ is uniquely determined by the $|\mathcal{S}|\times |\mathcal{A}|$ linear constraints, and the nonnegativity constraint $d \geq 0$ is not necessary. 
Moreover, given a fixed policy $\pi$, the solution of the inner problem is its visitation distribution $d^\pi$.


The constraints of the inner problem are known as the \textit{Bellman flow equations} that an achievable stationary state-action distribution must satisfy. The next question is how can we solve it? Here is where Lagrangian duality comes into play. First, we form the Lagrangian dual of our original optimization problem, transforming our constrained optimization into an unconstrained form. This introduces additional optimization variables - the Lagrange multipliers $Q$. 

As mentioned before, we can discard the nonnegativity constraint $d \geq 0$ as the other constraints imply a unique solution for $d$. Focusing on the inner optimization problem, we optimize the Lagrangian dual problem: 
\begin{align*}
    & \min_{Q(s,a)} \max_{d} \E{s,a\sim d(s,a)}{r(s,a)}-\alpha \f{d(s,a)}{d^O(s,a)}\\
    &+\sum_{s,a} Q(s,a)\left((1-\gamma)d_0(s).\pi(a|s)+\gamma \sum_{s',a'} d(s',a')p(s|s',a')\pi(a|s)-d(s,a)\right),
\end{align*}
where $Q(s,a)$ are the Lagrange multipliers associated with the equality constraints. 
We can now do some simple algebraic manipulation to further simplify it:
\begin{align}
    & \min_{Q(s,a)} \max_{d}\E{s,a\sim d(s,a)}{r(s,a)}-\alpha \f{d(s,a)}{d^O(s,a)} \nonumber\\
    &+\sum_{s,a} Q(s,a)\left((1-\gamma)d_0(s).\pi(a|s)+\gamma \sum_{s',a'} d(s',a')p(s|s',a')\pi(a|s)-d(s,a)\right)\\
   =& \min_{Q(s,a)}\max_{d} (1-\gamma)\E{d_0(s),\pi(a|s)}{Q(s,a)}  \label{eq:interchange_min_max} \nonumber \\
    &+\E{s,a\sim d}{r(s,a)+\gamma \sum_{s',\reb{a'}} p(s'|s,a)\pi(a'|s')Q(s',a')-Q(s,a)}-\alpha\f{d(s,a)}{d^O(s,a)},
\end{align}
where we swap the maximum and minimum in the last step as strong duality holds for this problem.
This is equivalent to solving the following scaled objective (scaled by $1/\alpha$).
\begin{align}
    &\min_{Q(s,a)}\max_{d} \frac{(1-\gamma)}{\alpha}\E{d_0(s),\pi(a|s)}{Q(s,a)} \nonumber \\
    &+\E{s,a\sim d}{(r(s,a)+\gamma \sum_{s',\reb{a'}} p(s'|s,a)\pi(a'|s')Q(s',a')-Q(s,a))/\alpha}-\f{d(s,a)}{d^O(s,a)}\\
  = & \min_{Q(s,a)} \frac{(1-\gamma)}{\alpha}\E{d_0(s),\pi(a|s)}{Q(s,a)} \nonumber\\
    \label{eq:estimating_performance_dual_Q}
    &+\E{s,a\sim d^O}{f^*((r(s,a)+\gamma \sum_{s',\reb{a'}} p(s'|s,a)\pi(a'|s')Q(s',a')-Q(s,a))/\alpha)},
\end{align}
where we applied the convex conjugate (Eq.~\eqref{eq:cx_conjugate_def}) in the last step.
To see this more clearly, let $y(s,a)= \left( r(s,a)+\gamma \sum_{s',\reb{a'}} p(s'|s,a)\pi(a'|s')Q(s',a')-Q(s,a) 
 \right) / \alpha$. Then, 
under mild conditions that the interchangeability principle~\citep{dai2017learning} is satisfied, 
and $d^O$ has sufficient support
over $\S \times \A$~\citep{nachum2020reinforcement}, it holds that
\begin{align}
\label{eq:f_cvx_conjugate}
      & \max_{d} \E{s,a\sim d}{y(s,a)}-\f{d(s,a)}{d^O(s,a)}\\
    = & \max_{d}\E{s,a\sim d^O}{\frac{d(s,a)}{d^O(s,a)}y(s,a)-f\left(\frac{d(s,a)}{d^O(s,a)}\right)}\\
    = & \E{d^O}{f^*(y(s,a))}.
\end{align}

We have transformed the problem of computing $J(\pi)$ to solving Eq.~\eqref{eq:estimating_performance_dual_Q}. Finally, the policy optimization problem $\max_\pi J(\pi)$ is reduced to solving the following min-max optimization problem, which we will refer to as $\texttt{dual-Q}$:
\begin{equation}
\resizebox{1\hsize}{!}{%
    $\max_{\pi}\min_{Q} \frac{(1-\gamma)}{\alpha}\E{d_0(s),\pi(a|s)}{Q(s,a)} +\E{s,a\sim d^O}{f^*((r(s,a)+\gamma \sum_{s',\reb{a'}} p(s'|s,a)\pi(\reb{a'}|s')Q(s',a')-Q(s,a))/\alpha)}$.
}
\end{equation}

 Table~\ref{tbl:div} lists the corresponding convex conjugates $f^*$ for common $f$-divergences. 

In the case of deterministic policy and deterministic dynamics, the above-obtained optimization takes a  simpler form:
\begin{equation}
    \max_{\pi(a|s)}\min_{Q(s,a)} \frac{(1-\gamma)}{\alpha}\E{\rho_0(s)}{Q(s,\pi(s))}+\E{s,a\sim d^O}{f^*((r(s,a)+\gamma Q(s',\pi(s'))-Q(s,a))/\alpha)}
\end{equation}
Now, we have seen how we can transform a regularized RL problem into its $\texttt{dual-Q}$ form which uses Lagrange variables in the form of state-action functions. Interestingly, we can go further to transform the regularized RL problem into Lagrange variables (V) that only depend on the state, and in doing so we also get rid of the two-player nature (min-max optimization) in the $\texttt{dual-Q}$. 

\subsubsection{Deriving \texttt{dual-V}}
\label{app:derive_dual_v}
One important constraint we have not discussed so far is that the variable $d$ we are optimizing must be nonnegative. 
This constraint is not needed for \primalQ, as for the inner problem~\eqref{eq:primal_rl_q_main},
the solution is uniquely determined by the constraints. Nonetheless, it is important we consider this constraint for \primalV and derive the correct dual problem.

In \primalV, we formulate the visitation constraints to depend solely on states rather than state-action pairs. Note that doing this does not change the solution $\pi^*$ for the regularized RL problem (Eq~\eqref{eq:reg_rl_ap}). We consider $\alpha=1$ for the sake of exposition. Interested readers can derive the result for $\alpha \neq 1$ as in the \texttt{dual-Q} case above.  Recall the formulation of \primalV:
\begin{align}
    \max_{d\ge0} & ~~\mathbb{E}_{d(s,a)}[r(s,a)]-\f{d(s,a)}{d^O(s,a)}\nonumber\\
    &\text{s.t}~~\sum_{a\in\mathcal{A}} d(s,a)=(1-\gamma)d_0(s)+\gamma \sum_{s',a'} d(s',a')p(s|s',a').
\end{align}
As before, we construct the Lagrangian dual to this problem. Note that our constraints now solely depend on $s$.
\begin{align}
   & \min_{V(s)}  \max_{d\ge 0}  \E{\reb{s,a}\sim d(s,a)}{r(s,a)}-\f{d(s,a)}{d^O(s,a)} \nonumber \\
   & \hskip40pt +\sum_sV(s)\left((1-\gamma)d_0(s)+\gamma \sum_{s',a'} d(s',a')p(s|s',a')-\reb{\sum_{a\in\mathcal{A}}}d(s,a)\right)
\end{align}



Using similar algebraic manipulations we used to obtain $\texttt{dual-Q}$ in Section~\ref{app:derive_dual_q}, we have :
\begin{flalign}
    & \min_{V(s)} \max_{d(s,a)\ge0} \E{s,a\sim d(s,a)}{r(s,a)}-\f{d(s,a)}{d^O(s,a)} \nonumber\\
    &\quad +\E{s,a\sim d}{r(s,a)+\gamma \sum_{s'} p(s'|s,a)V(s')-V(s)}-\f{d(s,a)}{d^O(s,a)} \label{eq:dual_v_to_postivity}  \\
         = &\min_{V(s)}\max_{d(s,a)\ge0} (1-\gamma)\E{d_0(s)}{V(s)}\nonumber\\
   &\quad  +\E{s,a\sim d}{r(s,a)+\gamma \sum_{s'} p(s'|s,a)V(s')-V(s)}-\f{d(s,a)}{d^O(s,a)}\\
     = & \min_{V(s)}\max_{d(s,a)\ge0}  (1-\gamma)\E{d_0(s)}{V(s)} \nonumber\\
   &\quad  + \E{s,a\sim d^O}{\tfrac{d(s,a)}{d^O(s,a)}\big(r(s,a)+\gamma \sum_{s'} p(s'|s,a)V(s')-V(s) \big)}- \E{s,a\sim d^O}{f\big(\tfrac{d(s,a)}{d^O(s,a)}\big)}
   \label{eq:dual_v_derivation_is}
\end{flalign}

Let $w(s,a)=\frac{d(s,a)}{d^O(s,a)}$
and $\tdV(s,a) = r(s,a)+\gamma \sum_{s'} p(s'|s,a)V(s')-V(s)$ denote the TD error. The last equation becomes
\begin{align}
     \label{eq:intermediate_primal_v}
     \min_{V(s)}\max_{w(s,a)\ge0}  (1-\gamma)\E{d_0(s)}{V(s)}
    +\E{s,a\sim d^O}{w(s,a)(\tdV(s,a))}- \E{s,a\sim d^O}{f(w(s,a))}.
\end{align}
We now direct the attention to the inner maximization problem and derive a closed-form solution for it. Consider the Lagrangian dual problem of it: 
\begin{align}
    \min_{\lambda \ge 0 } \max_{w(s,a)} \E{s,a\sim d^O}{w(s,a)(\tdV(s,a))}- \E{s,a\sim d^O}{f(w(s,a))} +\sum_{s,a} \lambda(s,a)w(s,a)
\end{align}
where the parameters $\lambda(s, a)$ for all $s \in S$ and $a \in A$ are the Lagrange multipliers. Since strong duality holds, we can use the KKT constraints to find the optimal solutions $w^*(s,a)$ and $\lambda^*(s,a)$:
\begin{enumerate}[leftmargin=*]\itemsep=0pt
    \item \textbf{Primal feasibility} $w^*(s,a)\ge0~~\forall~s,a$\\
    \item \textbf{Dual feasibility} $\lambda^*(s,a)\ge0~~\forall~s,a$\\
    \item \textbf{Stationarity} $d^O(s,a)(-f'(w^*(s,a))+\tdV(s,a)) \reb{+\lambda^*(s,a)}=0~~\forall~s,a$\\
    \item \textbf{Complementary Slackness} $w^*(s,a)\lambda^*(s,a)=0~~\forall~s,a$
\end{enumerate}

Using stationarity we have the following:
\begin{equation}
   \reb{d^O(s,a)} f'(w^*(s,a)) = \reb{d^O(s,a)}\tdV(s,a)+\lambda^*(s,a)~~\forall~s,a
\end{equation}
Now using complementary slackness only two cases are possible $w^*(s,a)\ge0$ or $\lambda^*(s,a)\ge0$.  Combining both cases we arrive at the following solution for this constrained optimization:
\begin{equation}
\label{eq:sol_w_star}
    w^*(s,a) = \max\left(0,{f'}^{-1}(\tdV(s,a)) \right)
\end{equation}
We refer to the resulting function after plugging the solution for $w^*$ back in Eq.~\eqref{eq:intermediate_primal_v} and refer to the closed form solution for $d$ in second and third term as $f^*_p$.
\begin{equation}
\label{eq:sol_f_star_p_in_w_star}
    f^*_p(\tdV(s,a)) = w^*(s,a)(\tdV(s,a))-{f(w^*(s,a))}
\end{equation}
Plugging in $w^*(s,a)$ from Eq.~\eqref{eq:sol_w_star} to Eq.~\eqref{eq:sol_f_star_p_in_w_star}, we get:
\begin{equation}
    \label{eq:f_star_p_def}
    f^*_p(\tdV(s,a)) = \max\left(0,{f'}^{-1}(\tdV(s,a)) \right)(\tdV(s,a))-{f\left(\max\left(0,{f'}^{-1}(\tdV(s,a)) \right)\right)}
\end{equation}

Note that we get the original conjugate $f^*$ back if we do not consider the nonnegativity constraints:
\begin{equation}
    f^*(s,a) ={f'}^{-1}(\tdV(s,a))(\tdV(s,a))- {f({f'}^{-1}(\tdV(s,a)))}.
\end{equation}
Finally, we have the following optimization to solve for \texttt{dual-V} when considering the nonnegativity constraints:
\begin{summarybox}
\centering
$\texttt{dual-V}$: $\;
     \min_{V(s)}  (1-\gamma)\E{s \sim d_0}{V(s)} + \E{(s,a)\sim d^O}{f^*_p(\tdV(s,a))} $
\end{summarybox}

Some works e.g. SMODICE~\citep{ma2022smodice}, ignore the nonnegativity constraints and use the corresponding \dualV formulation
\begin{summarybox}
$\dualV \text{\, (w/o nonneg. constraints)}$: $\min_{V} (1-\gamma)\E{s\sim d_0}{V(s)} +\E{(s,a)\sim d^O}{f^*(\tdV(s,a)}$.
\end{summarybox}

\subsubsection{Discussion on Dual Formulations}

In summary, we have two dual formulations for regularized policy learning:
\begin{summarybox}
$\texttt{dual-Q}$: $
\max_{\pi}\min_{Q} (1-\gamma)\E{d_0(s),\pi(a|s)}{Q(s,a)} \nonumber\\
~~~~~~~~~~~~~~~~~~~~~~~~~~~~~~+\E{s,a\sim d^O}{f^*\left(r(s,a)+\gamma \sum_{s'} p(s'|s,a)\pi(a'|s')Q(s',a')-Q(s,a)\right)}$
\end{summarybox}
and
\begin{summarybox}
\centering
$\texttt{dual-V}$: $\;
     \min_{V(s)}  (1-\gamma)\E{s \sim d_0}{V(s)} + \E{(s,a)\sim d^O}{f^*_p(\tdV(s,a))} $
\end{summarybox}

The above derivations for dual of primal RL formulation - $\texttt{dual-Q}$ and $\texttt{dual-V}$ brings out some important observations
\begin{itemize}
    \item $\texttt{dual-Q}$ and $\texttt{dual-V}$ present off-policy policy optimization solutions for regularized RL problems which requires sampling transitions only from the off-policy distribution the policy state-action visitation is being regularized against. The gradient with respect to policy $\pi$ when $d$ is optimized in \dualQ can be shown to be equivalent to the on-policy policy gradient under a regularized Q-function (see Section 5.1 from \citep{nachum2020reinforcement}).
    \item The above property allows us to solve not only RL problems but also imitation problems by setting the reward function to be zero everywhere and $d^O$ to be the expert dataset, and also offline RL problems where we want to maximize reward with the constraint that our state-action visitation should not deviate too much from the replay buffer ($d^O=\text{replay-buffer}$).
    \item $\texttt{dual-V}$ formulation presents a way to solve the RL problem using a single optimization rather than a min-max optimization of the \primalQ or standard RL formulation. \texttt{dual-V} implicitly subsumes greedy policy maximization.
\end{itemize}



\subsubsection{How to recover the optimal policy in \texttt{dual-V}? }
\label{ap:recovering_policy}
In the above derivations for \texttt{dual-Q} and \texttt{dual-V} we leveraged the fact that the closed form solution for optimizing Eq.~\eqref{eq:cx_conjugate_def} w.r.t $d$ is known. The value of $d^*$ for which Eq.~\eqref{eq:f_cvx_conjugate} is maximized can be found by setting the gradient to zero (stationary point) leading to:
\begin{equation}
\label{ap:optimal_distribution_ratio}
   \frac{d^*(s,a)}{d^O(s,a)}=\max\left(0, (f')^{-1} \left(\frac{y(s,a)}{\alpha}\right)\right)
\end{equation}

This ratio can be utilized in two different ways to recover the optimal policy:

\textbf{Method 1: Maximum likelihood on expert visitation distribution}

Policy learning can be written as maximizing the likelihood of optimal actions under the optimal state-action visitation:
\begin{equation}
    \max \E{s,a\sim d^*}{\log \pi_\theta(a|s)}
\end{equation}
Using importance sampling we can rewrite the optimization above in a form suitable for optimization: 
\begin{equation}
    \max_\theta \E{s,a\sim d^O}{\frac{d^*(s,a)}{d^O(s,a)}\log \pi_\theta(a|s)} = \max_\theta \E{s,a\sim d^O}{w^*(s,a)\log \pi_\theta(a|s)}
\end{equation}
This way of policy learning is similar to weighted behavior cloning or advantage-weighted regression, but suffers from the issue that policy is not optimized at state-actions where the offline dataset $d^O$ has no coverage but $d^*>0$.

\textbf{Method 2: Reverse KL matching on offline data distribution (Information Projection)}

To allow the policy to be optimized at all that states in the offline dataset + actions outside the dataset we consider an alternate objective:
\begin{equation}
    \min_\theta \kl{d^O(s)\pi_\theta(a|s)}{d^O(s)\pi^*(a|s)}
\end{equation}

The objective can be expanded as follows:
\begin{align}
    &\min_\theta \kl{d^O(s)\pi_\theta(a|s)}{d^O(s)\pi^*(a|s)} \\  &=\min_\theta\E{s\sim d^O(s),a\sim\pi_\theta}{\log \frac{\pi_\theta(a|s)}{\pi^*(a|s)}}\\
    &= \min_\theta \E{s\sim d^O(s),a\sim\pi_\theta}{\log \frac{\pi_\theta(a|s)d^*(s)d^O(s)\pi^o(a|s)}{\pi^*(a|s)d^*(s)d^O(s)\pi^o(a|s)}}\\
    &= \min_\theta \E{s\sim d^O(s),a\sim\pi_\theta}{\log\frac{\pi_\theta(a|s)}{\pi^o(a|s)}-\log(w^*(s,a))+\log \frac{d^*(s)}{d^O(s)}}\\
    &=  \min_\theta \E{s\sim d^O(s),a\sim\pi_\theta}{\log(\pi_\theta(a|s))-\log({\pi^o(a|s)})-\log(w^*(s,a))}
\end{align}
This method recovers the optimal policy at the states present in the dataset but has the added complexity of learning another policy $\pi^o(a|s)$. One way of obtaining $\pi^o(a|s)$ is by behavior cloning the replay buffer.

\subsubsection{Semi-gradient and Full-gradient}
\label{ap:semi_gradient_info}
RL algorithms often learn via Bellman backups which minimize an error of the form $L(\theta) = \mathbb{E}_{(s,a,s') \sim \mathcal{D} }[(Q_{\theta}(s,a) - (r(s,a) + \gamma \E{a' \sim \pi}{Q_\theta(s', a'))^2}$. Full stochastic gradient differentiates through the entire objective, whereas semi-gradient methods do not differentiate through the $Q_\theta(s', a')$ term in the bootstrapping target and is generally implemented through a stop-gradient operator. Note that the semi-gradient update still changes the value of the bootstrapping target when used at the next iteration as $\theta$ is updated. Semi-gradient optimization is a common choice in deep RL and often enables significantly faster learning~\citep{sutton2018reinforcement}.

\section{A unified perspective on RL and IL algorithms through duality}
\label{ap:complete_unification}

Figure~\ref{fig:dualRL_main} shows an alternate viewpoint on the landscape of dualRL methods and highlights the gaps in algorithms that we are able to successfully address in this work. Now we discuss in detail how recent algorithms can be unified via duality viewpoint for both RL and IL.

\begin{figure*}[h]
\begin{center}
      \includegraphics[width=1.0\linewidth]{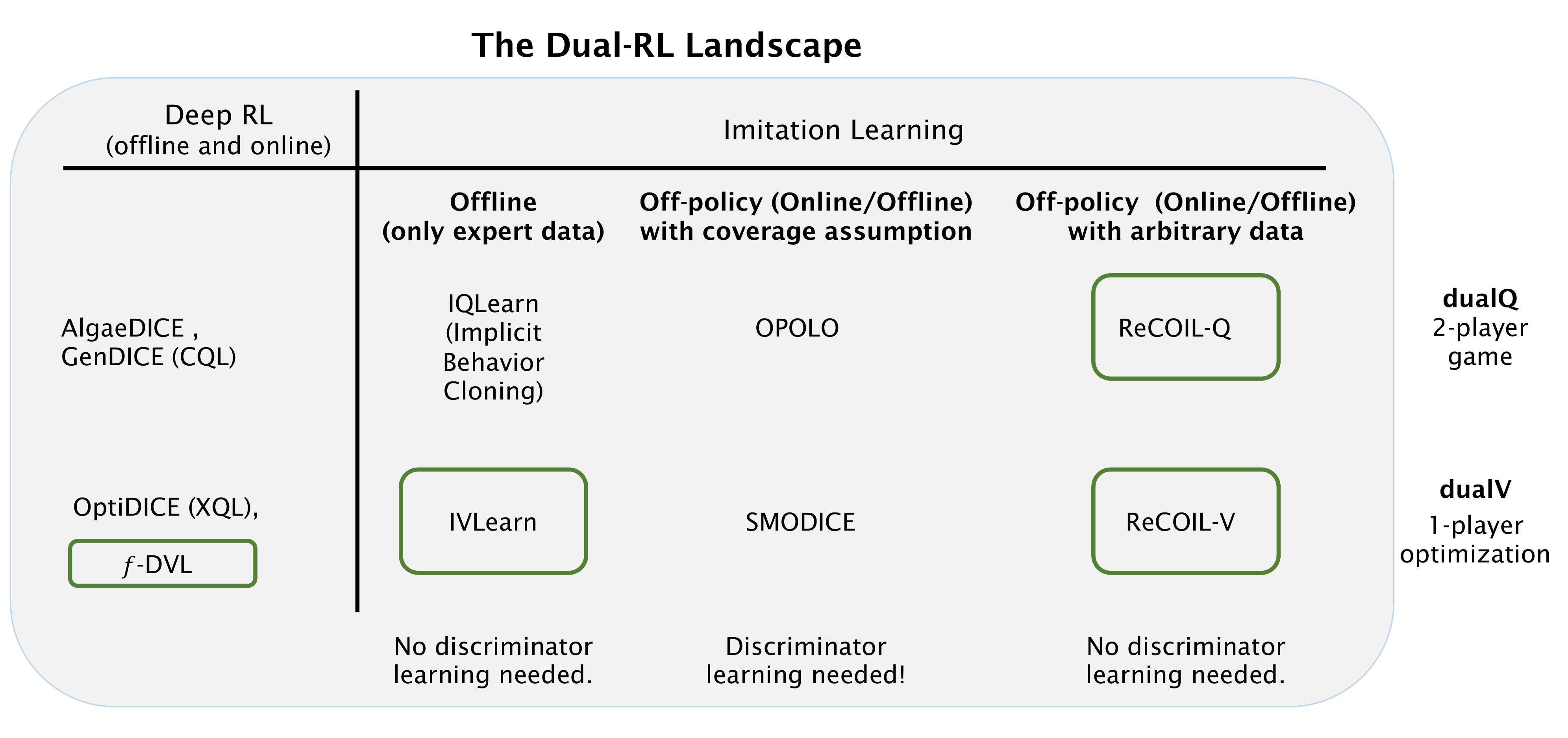}
      \\
      \vspace{-2.0mm}
\end{center}
\caption{We show that a number of prior methods can be understood as a special case of the dual RL framework. Based on this framework, we also propose new methods addressing the shortcomings of previous works (boxed in green).}
\label{fig:dualRL_main}
\end{figure*}

\subsection{Dual Connections to Reinforcement Learning}

We begin by showing reducing popular offline RL class of methods: pessimistic value learning 
 (CQL~\citep{kumar2020conservative}\reb{, ATAC~\citep{cheng2022adversarially}}) and implicit policy improvement (XQL~\citep{garg2021iq}) to the \texttt{dual-Q} and \texttt{dual-V} framework respectively. Then, we show how the \texttt{dual-V} framework under a semi-gradient update rule leads to a family of offline RL algorithms that do not sample OOD actions.

\begin{restatable}[]{proposition}{cql}
\label{thm:CQL}
CQL is an instance of \dualQ  under the semi-gradient update rule,
where the $f$-divergence is the Pearson $\chi^2$ divergence,
and $d^O$ is the offline visitation distribution. 
\vspace{-2mm}
\end{restatable}
\begin{proof}
\reb{We show that CQL~\citep{kumar2020conservative} and ATAC~\citep{cheng2022adversarially}, popular offline RL methods are a special case of $\texttt{dual-Q}$ for offline RL}. Consider the $\chi^2$ $f$-divergence with the generator function $f=(t-1)^2$. The dual function $f^*$ is given by $f^*=(\frac{t^2}{4}+t)$.
With this $f$-divergence the \texttt{dual-Q} optimization can be simplified as:

\begin{align}
&\frac{(1-\gamma)}{\alpha}\E{d_0,\pi(a|s)}{Q(s,a)}+\E{s,a\sim d^O}{\frac{y(s,a,r,s')^2}{4\alpha^2}+\frac{y(s,a,r,s')}{\alpha}}\\
&=\frac{(1-\gamma)}{\alpha}\E{d_0,\pi(a|s)}{Q(s,a)}+\E{s,a\sim d^O}{\frac{y(s,a,r,s')}{\alpha}}+\E{s,a\sim d^O}{\frac{y(s,a,r,s')^2}{4\alpha^2}}
\end{align}

Let's simplify the first two terms:
\begin{align}
    &\frac{1}{\alpha}\left[(1-\gamma)\E{d_0,\pi(a|s)}{Q(s,a)}+ \E{s,a\sim d^O}{r(s,a)+\gamma \sum_{s',a'}p(s'|s,a)\pi(a'|s')Q(s',a')-Q(s,a)}\right]
\end{align}
\begin{small}
\begin{align}
    &= \frac{1}{\alpha}\left[(1-\gamma)\E{d_0,\pi(a|s)}{Q(s,a)}+\E{s,a\sim d^O}{\gamma \sum_{s',a'}p(s'|s,a)\pi(a'|s')Q(s',a')}-\E{s,a\sim d^O}{Q(s,a)}+\cancel{\E{s,a\sim d^O}{r(s,a)}}\right]\label{eq:remove_constant_term_cql}
\end{align}
\end{small}

\begin{small}
\begin{align}
    &= \frac{1}{\alpha}\left[(1-\gamma)\sum_{s,a}d_0(s)\pi(a|s)Q(s,a) + \gamma \sum_{s,a} d^O(s,a) \sum_{s'}p(s'|s,a)\pi(a'|s')Q(s',a') - \E{s,a\sim d^O}{Q(s,a)}\right]
\end{align}
\end{small}

\begin{align}
    &= \frac{1}{\alpha}\left[(1-\gamma)\sum_{s,a}d_0(s)\pi(a|s)Q(s,a) + \gamma \langle d^O,P^\pi Q\rangle - \E{s,a\sim d^O}{Q(s,a)}\right]\\
    &= \frac{1}{\alpha}\left[(1-\gamma)\sum_{s,a}d_0(s)\pi(a|s)Q(s,a) + \gamma \langle P^\pi_* d^O, Q\rangle - \E{s,a\sim d^O}{Q(s,a)}\right]\\
    &=  \frac{1}{\alpha}\left[(1-\gamma)\sum_{s,a}d_0(s)\pi(a|s)Q(s,a) +\gamma \sum_{s,a} \pi(a|s) Q(s,a)\sum_{s',a'}p(s|s',a')d(s',a')  - \E{s,a\sim d^O}{Q(s,a)}\right]
\end{align}
 
\begin{align}
&= \frac{1}{\alpha}\left[\sum_{s,a}(d_0(s)+\gamma \sum_{s'a,'}p(s|s',a')d(s',a') )\pi(a|s)Q(s,a)  - \E{s,a\sim d^O}{Q(s,a)}+\E{s,a\sim d^O}{r(s,a)}\right]\\
    &= \frac{1}{\alpha}\left[\sum_{s,a} d^O(s)\pi(a|s) Q(s,a)  - \E{s,a\sim d^O}{Q(s,a)}+\E{s,a\sim d^O}{r(s,a)}\right]\\
    &= \frac{1}{\alpha}\left[\E{s\sim d^O,a\sim \pi}{Q(s,a)}- \E{s,a\sim d^O}{Q(s,a)}\right]
\end{align}

where $P^\pi$ denotes the policy transition operator, $P^\pi_{*}$ denotes the adjoint policy transition operator. Removing constant terms (Eq.~\eqref{eq:remove_constant_term_cql}) with respect to optimization variables we end up with the following form for \texttt{dual-Q}:

\begin{equation}
    \frac{1}{\alpha}\left[\color{teal}
      \underbrace{\color{black}\E{s\sim d^O,a\sim \pi}{Q(s,a)}}_{\text{reduce Q at OOD actions}}- \color{red}
      \underbrace{\color{black}\E{s,a\sim d^O}{Q(s,a)}}_{\text{increase Q at in-distribution actions}}\color{black} \right]+\color{orange}
      \underbrace{\color{black}\E{s,a\sim d^O}{\frac{y(s,a,r,s')^2}{4\alpha^2}}}_{\text{minimize Bellman Error}}\color{black}
\end{equation}
Hence the \texttt{dual-Q} optimization reduces to:
\begin{equation} 
 \max_\pi \min_Q   \alpha\left[\E{s\sim d^O,a\sim \pi}{Q(s,a)}- \E{s,a\sim d^O}{Q(s,a)}\right]+\E{s,a\sim d^O}{\frac{y(s,a,r,s')^2}{4}}
\end{equation}

Our proposition assumes semi-gradient update, i.e the gradients are not backpropogated through the bootstrapping target $Q(s',\pi(s'))$ when updating $\pi$ or $Q$. The bootstrapping target is regularly updated with the most recent parameters. Thus, maximization with respect to policy just amounts to maximizing the first term $\E{s\sim d^O,a\sim \pi}{Q(s,a)}$. \reb{This update equation matches the unregularized CQL objective (Equation 3 in~\citep{kumar2020conservative}) and the ATAC objective (Equation 1 in~\citep{cheng2022adversarially} when $\beta=0.25$).  One of they key differences between CQL and ATAC is the use of optimization strategy -- CQL uses Gradient Descent Ascent whereas ATAC uses a Stackelberg formulation.}
\end{proof}

\xql*

\begin{proof}
We show that the Extreme Q-Learning~\citep{garg2023extreme} framework for offline and online RL is a special case of the dual framework, specifically the $\texttt{dual-V}$ using the semi-gradient update rule.   

Consider setting the $f$-divergence to be the KL divergence in the \dualV framework, the regularization distribution and the initial state distribution to be the replay buffer distribution ($d^O=d^R$ and $d_0=d^R$). The conjugate of the generating function for KL divergence is given by $f^*(t)=e^{t-1}$.

\begin{equation}
\label{eq:dual-V-kl} 
\min_{V(s)} (1-\gamma)\E{d_0(s)}{V(s)} +\E{s,a\sim d^R}{f^*\left(\left[r(s,a)+\gamma \sum_{s'} p(s'|s,a)V(s')-V(s))\right]/
\alpha\right)}
\end{equation}

\begin{align}
\min_{V(s)} (1-\gamma)\E{d_0(s)}{V(s)} +\E{s,a\sim d^S}{\text{exp}(\left(\left[r(s,a)+\gamma \sum_{s'} p(s'|s,a)V(s')-V(s))\right]/
\alpha-1\right)}
\end{align}

A popular approach for stable optimization in temporal difference learning is the semi-gradient update rule which has been studied in previous works~\citep{sutton2018reinforcement}. In this update strategy, we fix the targets for the temporal difference backup. The target in the above optimization is given by:
\begin{equation}
    \bar{Q}(s,a) = r(s,a) + \gamma \sum_{s'} p(s'|s,a)V(s')
\end{equation}
The update equation for V is now given by:

\begin{align}
\label{eq:reduction_xql_v}
\min_{V(s)} (1-\gamma)\E{d_0(s)}{V(s)} +\E{s,a\sim d^R}{\text{exp}(\left(\left[\bar{Q}(s,a)-V(s))\right]/
\alpha-1\right)}
\end{align}

 where hat denotes the $\texttt{stop-gradient}$ operation. We approximate this target by using mean-squared regression with the single sample unbiased estimate as follows:
 
\begin{equation}
\label{eq:reduction_xql_q}
    \min_Q \E{s,a,s'\sim d^R}{(Q(s,a) - (r(s,a) + V(s')))^2}
\end{equation}

The procedure (alternating Eq.~\eqref{eq:reduction_xql_v} and Eq.~\eqref{eq:reduction_xql_q}) is now equivalent to the Extreme-Q learning and is a special case of the \texttt{dual-V} framework.
\end{proof}

\subsubsection{$f$\texttt{-DVL}: A family of implicit policy improvement algorithms for RL}
\label{ap:implicit_maximizers_family}

\begin{figure}[h]
\begin{center}
    \includegraphics[width=1.0\linewidth]{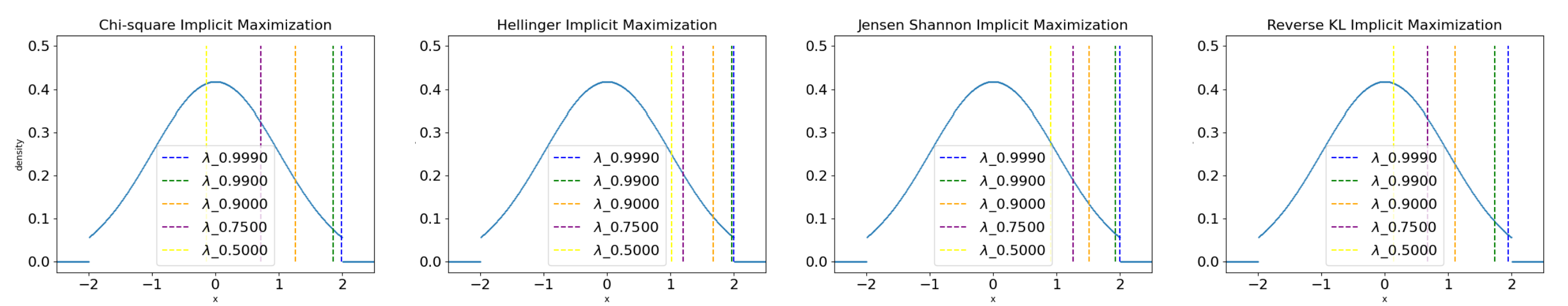}
      \\
\end{center}
\vspace{-1.0mm}
\caption{
Illustration of a family of implicit maximizers corresponding to different $f$-divergences.
The underlying data distribution is a truncated Gaussian \texttt{TN} with mean $0$, variance $1$ and 
a truncation range $(-2, 2)$. 
We sample 10000 data points from \texttt{TN} and compute the solution $v_\lambda$ of Problem~\eqref{eq:implicit_maximization_general}. 
As $\lambda \to 1$, the solution $v_\lambda$ becomes a more accurate estimation for the supremum of the random variable $x$.
}
\label{fig:implicit_maximizers}
\vspace{5pt}
\end{figure}

\implicit*
\begin{proof}
\reb{The behavior of dual-V (Equation~\ref{eq:dual-V}) when $f$ is reverse KL to serve as an implicit maximizer was established in~\citep{garg2023extreme}. In this section we consider other divergences from Table~\ref{tbl:div} under the rewriting of dual-V in terms of temperature $\lambda$ and a surrogate extention for $f^*_p$ (defined below)}.  We analyze the behavior for the following optimization of interest.
\begin{equation}
\label{eq:ap_implicit_maximization_general}
\min_{v} (1-\lambda)\E{x\sim D}{v}+ \lambda\E{x\sim D}{\bar{f^*_p}\left( x-v  \right)}
\end{equation}
$f^*_p(t)$ is given by (using the definition in Eq.~\eqref{eq:f_star_p_def}:
\begin{equation}
    f^*_p(t) = -f\left(\max({f'}^{-1}(t),0)\right) +  t \max\left({f'}^{-1}(t),0 \right)
\end{equation}
Accordingly, the function $f^*_p$ admits two different behaviors given by:
\begin{equation}
\label{eq:f_star_p_extension}
   f^*_p= 
\begin{cases}
   -f({f'}^{-1}(t)) +  t {f'}^{-1}(t) = f^*(t) ,& \text{if } {f'}^{-1}(t)>0\\
    -f(0),              & \text{otherwise}
\end{cases}  
\end{equation}
\reb{where $f^*$ is the convex conjugate of $f$-divergence. We consider all $f$-divergences for which $f^*$ is strictly increasing in $\R^+$, and note that TV  divergence will need special treatment as $(f')^{-1}$ is not well defined. Some properties of note are  $f'$ and $(f')^{-1}$ is non-decreasing, $f(0^+)\ge0$ and $f^*(x)\ge0~\forall x\ge0$ for the divergences we consider. A key limitation of formulating an optimization objective with forms of $f$-divergences is their domain restriction of $\mathbb{R}^+$. First, we note as a result of restriction of $f$ to $\R^+$ that $f':\R^+ \rightarrow [l,\infty)$ for some $l\in\R$ and as a consequence $(f')^{-1}:[l,\infty) \rightarrow \R^+$. Since our objective function depends on $(f')^{-1}(x)$ being well defined on $\R$, We consider an extension that preserves the non-decreasing property of $(f')^{-1}$ such that $(f')^{-1}(x)=0$ for $x\in (-\infty,l]$. We also know from above that for  $x>=l$ , $(f')^{-1}(x)>=0$ as $(f')^{-1}$ is non-decreasing. We define the surrogate $\bar{f^*_p}$ to be a particular extension for $f^*_p$ with $l=0$. Similar extensions can be found in prior work~\citep{picard2022change, goldfieldinfotheory}.}


We analyze the second term in Eq.~\eqref{eq:ap_implicit_maximization_general}. It can be expanded as follows:
\begin{align}
    \lambda \int_{x:(f')^{-1}(x-v)>0} p(x) f^*(x-v) dx - \lambda\int_{x:(f')^{-1}(x-v)\le0} f(0)p(x)dx
\end{align}

From the properties of $f$, we use the fact that $(f')^{-1}(x-v)>0$ when $x-v>0$ or equivalently $x>v$.
\begin{align}
    \lambda \int_{x>v} p(x) f^*(x-v) dx - \lambda\int_{x\le v} f(0)p(x)dx
\end{align}
The first term in the above equation decreases monotonically and the second term increases monotonically (thus the combined terms decrease) as $v$ increases until $v=x^*$ (supremum of the support of the distribution) after which the equation assumes a constant value of $-\lambda f(0)$.

Going back to our original optimization in Eq.~\eqref{eq:ap_implicit_maximization_general}, the first term decreases monotonically with $v$. As $\lambda\to 1$, the minimization of the second term takes precedence, with increasing $v$ until saturation ($v=x^*$). We can go further to characterize the effect of $\lambda$ on solution $v_\lambda$ of the equation. The solution of the optimization can be written in closed form as (using stationarity):
\begin{equation}
    \frac{(1-\lambda)}{\lambda} = \E{x\sim D}{{{f_p^*}'}(x-v)}
\end{equation}
Using the fact that $f^{*'}_p$ is non-decreasing, 
we can show that the right-hand term in the equation above increases\reb{/stays the same} as $v$ decreases. 
This in turn implies that for all $\lambda_1,\lambda_2$ such that $\lambda_1 \le \lambda_2$ we have that $v_{\lambda_1}\le v_{\lambda_2}$ .

\reb{TV divergence require a special treatment as $(f')^{-1}$ is not defined. We construct $f^*_p$ by noting that $f^*$ for TV exists even if  $(f')^{-1}$ does not. A concise proof can be found in Example 8.1 from~\cite{goldfieldinfotheory}. Thus, for TV we consider a smooth extension in $\R^+$ by using $f^*(x)=x$. For Squared Hellinger and RKL, $f^*$ is discontinuous. The problem can be addressed by considering random variable $x\sim D$ upper bounded by 1 and $\ln 2$ for Hellinger and RKL respectively. This can be ensured by rescaling rewards so that the maximum reward is ${1-\gamma}$ and $(1-\gamma)\ln(2)$ for hellinger and RKL respectively. Appendix~\ref{ap:f_star_p_practical} outlines the derivation of surrogate implicit maximizers that we use in practice.
}

\end{proof}

\subsubsection{Connections of CQL to AlgaeDICE and XQL to OptiDICE}

\citet{kumar2020conservative} shows that CQL outperforms a family of behavior-regularized offline RL methods~\citep{fujimoto2018addressing, wu2019behavior, nair2020awac}, which solve different forms of~\primalQ using approximate dynamic programming.
The above result indicates that CQL's better performance is likely due to the choice of $f$-divergence and more amenable optimization afforded by the dual formulation. Moreover, the same \dualQ formulation has been previously studied for online RL in AlgaeDICE~\citep{nachum2019algaedice}, and proposition~\ref{thm:CQL} suggests that CQL is an offline version of AlgaeDICE. 

We also highlight that the full-gradient variant of the \texttt{dual-V} framework for offline RL has been studied extensively in OptiDICE~\citep{lee2021optidice} and proposition~\ref{thm:XQL} highlights that XQL is a special case OptiDICE with a semi-gradient update rule.


\subsection{Dual Connections to Imitation Learning} 
\label{ap:sec_dual_connections_imitation}
This section outlines the reduction of a number of algorithms for Imitation Learning to the dual framework. Most prior methods can either take into account expert-only data for imitation whereas the other methods which do imitation from arbitrary offline data are limited by their assumptions and the form of $f$-divergence they optimize for. We walk through explaining how prior methods can be derived through the unified framework and also why they are limited.

\subsubsection{Offline imitation learning with expert data only}
\label{ap:offline_imitation_learning_expert}

We saw in Section~\ref{main:offline_il_expert}, how using the $\texttt{dual-Q}$ framework directly led to a reduction of IQ-Learn~\citep{garg2021iq} as part of the dual framework. This was accomplished by simple setting the reward function to be 0 uniformly and setting the regularization distribution to the expert. \citet{garg2021iq} uses this method in the online imitation learning setting as well by incorporating the replay data as additional regularization which we suggest is unprincipled, also pointed out by others~\citep{al2023ls} (as only expert data samples can be leveraged in the above optimization) and provide a fix in the Section~\ref{sec:new_il_method}. In this section, we show how the same approach can directly lead to another method for learning to imitate from expert-only data avoiding the alternating min-max optimization of IQ-Learn.

\textbf{IV-Learn: A new method for offline imitation learning: } Analogous to \texttt{dual-Q} (offline imitation), we can leverage the \texttt{dual-V} (offline imitation) setting which avoids the min-max optimization given by:

\texttt{IV-Learn} or \texttt{dual-V}~\text{(offline imitation from expert-only data)}:
\begin{equation}
\label{eq:dual-V-imit}
\min_{V(s)} (1-\gamma)\E{d_0(s)}{V(s)} +\E{s,a\sim d^E}{f^*\left(\left[\mathcal{T}_0 V(s,a)-V(s))\right]/
\alpha\right)}
\end{equation}

We propose \texttt{dual-V} (offline imitation) to be a new method arising out of this framework which we leave for future exploration. This work primarily focuses on imitation learning from general off-policy data.

\textbf{Proofs for this section:}

\begin{restatable}[]{corollary}{ibc}
\label{corollary:ibc}
 IBC~\citep{florence2022implicit} is an instance of \texttt{dual-Q} using the full-gradient update rule, where $r(s,a)=0~~\forall s \in \mathcal{S}, a \in \mathcal{A}$, $d^O=d^E$, and the $f$-divergence is the total variation distance. 
 \vspace{-2mm}
\end{restatable}

Eq.~\eqref{eq:dual-Q-imit} suggests that intuitively IQ-Learn trains an energy-based model in the form of Q where it pushes down the Q-values for actions predicted by current policy and pushes up the Q-values at the expert state-action pairs. This becomes more clear when the divergence $f$ is chosen to be Total-Variation ($f^*=\mathbb{I}$), IQ-Learn for Total-Variation divergence reduces to:

\begin{align}
     &(1-\gamma)\E{d_0(s),\pi(a|s)}{Q(s,a)}+\E{s,a\sim d^E}{\gamma \sum_{s',a'} p(s'|s,a)\pi(a'|s')Q(s',a')-Q(s,a)}\\
    &=  \left[(1-\gamma)\E{d_0(s),\pi(a|s)}{Q(s,a)}+\E{s,a\sim d^E}{\gamma \sum_{s',a'} p(s'|s,a)\pi(a'|s')Q(s',a')}\right]\nonumber\\
    &-\E{s,a\sim d^E}{Q(s,a)} \label{eq:iq_tv}
\end{align} 

First, we simplify the initial two terms:

\begin{align}
    &(1-\gamma)\E{d_0(s),\pi(a|s)}{Q(s,a)}+\E{s,a\sim d^E}{\gamma \sum_{s'} p(s'|s,a)\pi(a'|s')Q(s',a')}\\
    &=(1-\gamma) \sum_{s,a} d_0(s)\pi(a|s)Q(s,a) + \gamma\sum_{s,a} d^E(s,a) \sum_{s',a'}  p(s'|s,a)\pi(a'|s')Q(s',a')
\end{align}
\begin{align}
& = (1-\gamma) \sum_{s,a} d_0(s)\pi(a|s)Q(s,a) + \gamma\sum_{s',a'} \sum_{s,a} d^E(s,a)   p(s'|s,a)\pi(a'|s')Q(s',a')\\
    & = (1-\gamma) \sum_{s,a} d_0(s)\pi(a|s)Q(s,a) + \gamma\sum_{s',a'} \pi(a'|s')Q(s',a')(\sum_{s,a} d^E(s,a)   p(s'|s,a))\\
    & = (1-\gamma) \sum_{s,a} d_0(s)\pi(a|s)Q(s,a) + \gamma\sum_{s',a'} \pi(a'|s')Q(s',a')(\sum_{s,a} d^E(s,a)   p(s'|s,a))
\end{align}

\begin{align}
& = (1-\gamma) \sum_{s,a} d_0(s)\pi(a|s)Q(s,a) + \gamma\sum_{s,a} \pi(a|s)Q(s,a)(\sum_{s',a'} d^E(s',a')   p(s|s',a'))\\
    & =  \sum_{s,a} (1-\gamma) d_0(s)\pi(a|s)Q(s,a) +  \pi(a|s)Q(s,a)(\sum_{s',a'} d^E(s',a')   p(s|s',a'))\\
    & =  \sum_{s,a}\pi(a|s)Q(s,a) \left[(1-\gamma) d_0(s) +  \gamma\sum_{s',a'} d^E(s',a')   p(s|s',a')\right]\\
    & =  \sum_{s,a}\pi(a|s)Q(s,a) d^E(s)
\end{align}

where the last step is due to the steady state property of the MDP (Bellman flow constraint).  


Therefore IQ-Learn/$\texttt{dual-Q}$ for offline imitation (in the special case of TV divergence) simplifies to (from Eq.~\eqref{eq:iq_tv}):
\begin{align}
    &\left[(1-\gamma)\E{d_0(s),\pi(a|s)}{Q(s,a)}+\E{s,a\sim d^E}{\gamma \sum_{s',a'} p(s'|s,a)\pi(a'|s')Q(s',a')}\right]-\E{s,a\sim d^E}{Q(s,a)}\\
    & ~~~~~~~~~~~= \min_Q \E{d_E(s),\pi(a|s)}{Q(s,a)}-\E{s,a\sim d^E}{Q(s,a)} \label{eq:iqlearn_tv}
\end{align}
The update gradient w.r.t for the above optimization matches the gradient update of infoNCE objective in Implicit Behavior Cloning~\citep{florence2022implicit} with $Q$ as the energy-based model.

\subsection{Off-policy imitation learning (under coverage assumption)}
\label{ap:off_policy_imitation_coverage}
 Directly utilizing the dual-RL framework for imitation has its limitation as we see in the previous section -- we cannot leverage off-policy suboptimal data. We first show that it is easy to see why choosing the $f$-divergence to reverse KL makes it possible to get an off-policy objective for imitation learning in the dual framework.  We start with the \primalQ for imitation learning under the reverse KL-divergence regularization ($r(s,a)=0~~\text{and}~~d^O=d^E$):
\begin{align}
\label{eq:primal_imitation_kl}
    &\max_{d(s,a)\ge0,\pi(a|s)} -\kl{d(s,a)}{d^E(s,a)} \nonumber\\
    &\text{s.t}~~d(s,a)=(1-\gamma)\rho_0(s).\pi(a|s)+\gamma \pi(a|s)\sum_{s',a'} d(s',a')p(s|s',a').
\end{align}
\textit{Under the assumption that the suboptimal data visitation (denoted by $d^S$) covers the expert visitation ($d^S>0$ wherever $d^E$>0)}~\citep{ma2022smodice}, which we refer to as the \textbf{coverage assumption}, the reverse KL divergence can be expanded as follows:
\begin{align}
  \kl{d(s,a)}{d^E(s,a)} &= \E{s,a\sim d(s,a)}{\log \frac{d(s,a)}{d^E(s,a)}}=\E{s,a\sim d(s,a)}{\log \frac{d(s,a)}{d^E(s,a)}\frac{d^S(s,a)}{d^S(s,a)}}\\
  &= \E{s,a\sim d(s,a)}{\log \frac{d(s,a)}{d^S(s,a)}+\log \frac{d^S(s,a)}{d^E(s,a)}}\\
  &= \E{s,a\sim d(s,a)}{\log \frac{d^S(s,a)}{d^E(s,a)}}+\kl{d(s,a)}{d^S(s,a)}.
\end{align}
Hence the \primalQ can now be written as:
\begin{align}
    \max_{d(s,a)\ge0,\pi(a|s)} \E{s,a\sim d(s,a)}{-\log \frac{d^S(s,a)}{d^E(s,a)}}-\kl{d(s,a)}{d^S(s,a)}\\
    \text{s.t}~~d(s,a)=(1-\gamma)\rho_0(s).\pi(a|s)+\gamma \sum_{s',a'} d(s',a')p(s|s',a')\pi(a|s).
\end{align}
Now, in the optimization above the first term resembles the reward function and the second term resembles the divergence constraint with a new distribution $d^S(s,a)$ in the original regularized RL primal (Eq.~\eqref{eq:primal_rl}). Hence we can obtain respective $\texttt{dual-Q}$ and $\texttt{dual-V}$ in the setting for off-policy imitation learning using the reward function as $r^\text{imit}(s,a)=-\log \frac{d^S(s,a)}{d^E(s,a)}$ and the new regularization distribution as $d^S(s,a)$. Using $\Timit^\pi$ and $\Timit$ to denote backup operators under a new reward function $r^\text{imit}$, we have

\colorbox{Goldenrod!30}{$\texttt{dual-Q}$ for off-policy imitation (coverage assumption)}:
\begin{align}
    \label{eq:dual_q_imit_coverage}
    &\max_{\pi(a|s)}\min_{Q(s,a)} (1-\gamma)\E{\rho_0(s),\pi(a|s)}{Q(s,a)}+\E{s,a\sim d^S}{f^*(\Timit^\pi Q(s,a)-Q(s,a))}.
\end{align}
This choice of KL divergence leads us to a reduction of  another methods, OPOLO~\citep{zhu2020off} and OPIRL~\citep{hoshino2022opirl} for off-policy imitation learning to \texttt{dualQ} which we formalize in the proposition below:
\begin{restatable}[]{proposition}{opolo}
OPOLO~\citep{zhu2020off} and OPIRL~\citep{hoshino2022opirl} are instances of $\texttt{dual-Q}$ using the semi-gradient update rule, where $r(s,a)=0~\forall \mathcal{S},\mathcal{A}$, $d^O=d^E$,
and the $f$-divergence set to the reverse KL divergence.
\vspace{-2mm}
\end{restatable}
\proof{Proof is sketched in the above section, ie. Eq.~\eqref{eq:dual_q_imit_coverage} is the update equation for OPOLO.}

Analogously we have $\texttt{dual-V}$ for off-policy imitation (coverage assumption):
\begin{equation}
\min_{V(s)} (1-\gamma)\E{\rho_0(s)}{V(s)}+\E{s,a\sim d^S}{\reb{f^*_p}(\mathcal{T}_{r^{\text{imit}}}V(s,a)-V(s))}.
\end{equation}

We note that the \texttt{dual-V} framework for off-policy imitation learning under coverage assumptions was studied in the imitation learning work SMODICE~\citep{ma2022smodice}.

\subsection{Logistic Q-learning and P$^2$IL as dual-QV methods}
\label{ap:LQL_P2IL}
Logistic Q-learning and Proximal Point Imitation Learning (P$^2$IL) uses a modified primal for regularized policy optimization:
\begin{align}
    \max_{d\ge0} & ~~\mathbb{E}_{d(s,a)}[r(s,a)]-\f{d(s,a)}{d^O(s,a)} - H (\mu(s,a)\| \mu^O(s,a)\nonumber\\
    &\text{s.t}~~ d(s,a)=(1-\gamma)d_0(s)+\pi(a|s)\gamma \sum_{s',a'} \mu(s',a')p(s|s',a').\\
     &\text{and}~~ d(s,a)=\mu(s,a)
\end{align}
where $H(\mu(s,a)\|\mu^O(s,a)) = \sum \mu(s,a)\log\frac{\pi_\mu(a|s)}{\pi_{\mu^O}(a|s)}$ denotes the conditional relative entropy and $\mu^O$ is another offline distribution of state-action transitions potentially the same as $d^O$. The optimization is overparameterized (setting $\mu=d$). This trick was popularized via~\citep{mehta2009q} and leads to unbiased estimators and better downstream data driven algorithms. We call these two methods \texttt{dual-QV} as their dual requires estimating both $Q$ and $V$ as shown in~\citep{viano2022proximal,bas2021logistic}

\section{ReCOIL: Off-policy imitation learning without the coverage assumption}
\label{ap:closer}

\begin{figure}
\begin{center}
      \includegraphics[width=1.0\linewidth]{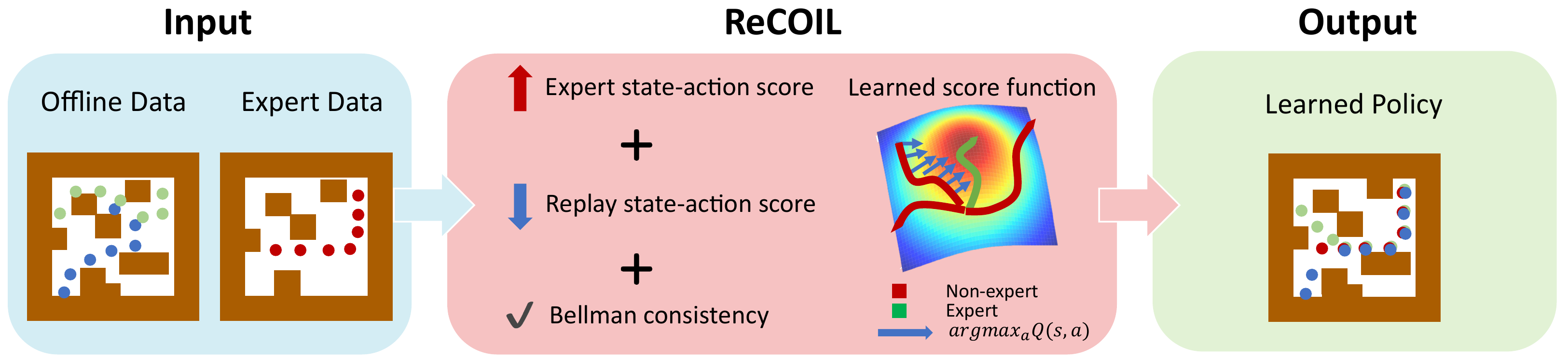}
\end{center}
\caption{Recipe for ReCOIL: Learn a Bellman consistent EBM - A model which increases the score of expert transitions, and decreases the score of replay transitions while maintaining Bellman consistency throughout. }
\label{fig:recoil_main}
\end{figure}

Understanding the limitations of existing imitation learning methods in the dual framework, we now proceed to derive our proposed method for imitation learning with arbitrary 
(off-policy) data. The derivation for the $\texttt{dual-Q}$ setting is shown below. \texttt{dual-V} derivation proceeds analogously.
\closerq*


\begin{proof}

Let's define two mixture distributions that we are going to leverage to formulate the imitation learning problem: $\dmix :=\beta d(s,a)+(1-\beta)d^S(s,a)$ and $\demix := \beta d^E(s,a)+(1-\beta)d^S(s,a)$. $\dmix$ is a mixture between the current agent's visitation distribution with suboptimal transition dataset obtained from a mixture of arbitrary policies and $\demix$ is the mixture between the expert's visitation distribution with arbitrary experience from the offline transition dataset. Minimizing the divergence between these visitation distributions still solves the imitation learning problem, i.e $d=d^E$. We again start with the new modified \primalQ under this mixture divergence regularization:
\begin{align*}
    &\max_{d(s,a)\ge0,\pi(a|s)}  -\f{\dmix(s,a)}{\demix(s,a)} \nonumber\\
    &\text{s.t}~~d(s,a)=(1-\gamma)d_0(s).\pi(a|s)+\gamma \pi(a|s)\sum_{s',a'} d(s',a')p(s|s',a').
\end{align*}
Using the same algebraic machinery of duality as before (Section~\ref{app:derive_dual_q}) to get an unconstrained tractable optimization problem, we obtain:
\begin{align}
    &\max_{\pi,d\ge0} \min_{Q(s,a)} -\f{\dmix}{\demix}\nonumber\\
    &+ \beta \sum_{s,a} Q(s,a)\left((1-\gamma)d_0(s).\pi(a|s)+\gamma \sum_{s',a'} d(s',a') p(s|s',a')\pi(a|s)-d(s,a)\right)\\
    &= \max_{\pi,d\ge0} \min_{Q(s,a)}   \beta (1-\gamma)\E{d_0(s),\pi(a|s)}{Q(s,a)} \nonumber \\
    &+ \beta  \E{s,a\sim d}{\gamma \sum_{s'} p(s'|s,a)\pi(a'|s')Q(s',a')-Q(s,a)}-\f{\dmix}{\demix}
\end{align}
\begin{align}
&= \max_{\pi,d\ge0} \min_{Q(s,a)}   \beta (1-\gamma)\E{d_0(s),\pi(a|s)}{Q(s,a)} \nonumber \\
    &+ \beta \E{s,a\sim d}{\gamma \sum_{s'} p(s'|s,a)\pi(a'|s')Q(s',a')-Q(s,a)}\nonumber\\
    &+(1-\beta) \E{s,a\sim d^S}{\gamma \sum_{s'} p(s'|s,a)\pi(a'|s')Q(s',a')-Q(s,a)}\nonumber\\
    &-(1-\beta) \E{s,a\sim d^S}{\gamma \sum_{s'} p(s'|s,a)\pi(a'|s')Q(s',a')-Q(s,a)}-\f{\dmix}{\demix}
\end{align}

 \reb{Strong duality allows us to swap the order of $\max_d$ and $\min_Q$ in order to arrive at the following result:}
 
\begin{summarybox}
\text{Imitation from Arbitrary data (dualQ)}\\
    \begin{align}
&=\max_{\pi(a|s)}\min_{Q(s,a)}\max_{d\ge0}  \beta (1-\gamma)\E{d_0(s),\pi(a|s)}{Q(s,a)} \nonumber \\
    &+\E{s,a\sim \dmix}{\gamma \sum_{s'} p(s'|s,a)\pi(a'|s')Q(s',a')-Q(s,a)}-\f{\dmix}{\demix}\nonumber\\
    & - (1-\beta ) \E{s,a\sim d^S}{\gamma \sum_{s'} p(s'|s,a)\pi(a'|s')Q(s',a')-Q(s,a)}
    \label{eq:imitation_mixture_approach2_no_constraint}
\end{align}
\end{summarybox}

We can ignore the constraints ($d\ge0$) as the primal-Q is overparameterized and the constraints uniquely determine the distribution $d$. Therefore, ignoring this constraint ($d\ge0$) results in the following dual-optimization for imitation from arbitrary data.

\begin{align}
    &\max_{\pi(a|s)}\min_{Q(s,a)}  \reb{\beta} (1-\gamma)\E{d_0(s),\pi(a|s)}{Q(s,a)} \nonumber\\
    &+\E{s,a\sim \demix}{f^*(\gamma \sum_{s'} p(s'|s,a)\pi(a'|s')Q(s',a')-Q(s,a))}\nonumber\\
    & - (1-\reb{\beta}) \E{s,a\sim d^S}{\gamma \sum_{s'} p(s'|s,a)\pi(a'|s')Q(s',a')-Q(s,a)}
\end{align}

\subsection{ReCOIL-V}
A similar derivation can be done in $V$-space to obtain an analogous result for \texttt{ReCOIL-V}, although extra care has to be taken to ensure the non-negativity constraints similar to proof for in section~\ref{app:derive_dual_v}.
\begin{align}
\label{eq:primal_imitation_f_mixture}
\colorbox{RoyalBlue!15}{\primalV} &~~\max_{d(s,a)\ge0} -\f{d^S_\text{mix}(s,a)}{d^{E,S}_\text{mix}(s,a)} \nonumber\\
    &~~\text{s.t}~\textstyle \sum_{a\in\mathcal{A}} d(s,a)=(1-\gamma)d_0(s)+\gamma \sum_{(s',a') \in \S \times \A } d(s',a') p(s|s',a'), \; \forall s \in \S.
\end{align}
The dual of \primalV form for mixture distribution matching is given by:
\begin{equation}
\label{eq:recoil_v}
\resizebox{0.935\textwidth}{!}{
\colorbox{Goldenrod!30}{\texttt{ReCOIL-V}}
$\min_{V} \beta (1-\gamma)\E{s \sim d_0}{V(s)}  +\E{(s,a)\sim d_\text{mix}^{E,S}}{f^*_p(\mathcal{T}_0V(s,a)-V(s))} - (1-\beta) \E{(s,a)\sim d^S}{\mathcal{T}_0V(s,a)-V(s)}$.
}
\end{equation}

\end{proof}

\input{11appendix-recoil-perf-bound}

\subsection{ReCOIL with $\chi^2$ divergence}
\label{ap:recoil_with_chi_square}
In this section, we derive the objective for \texttt{ReCOIL} under the chosen  $\chi^2$ divergence.Starting with the core ReCOIL objective for Q-update:

\begin{equation}
  \max_{\pi}\min_{Q}  \beta (1-\gamma)\mathbb{E}_{d_0,\pi}[Q(s,a)] +\mathbb{E}_{s,a\sim d_\text{mix}^{E,S}}[f^*(\mathcal{T}^{\pi}_0 Q(s,a )-Q(s,a))]  -(1-\beta)\mathbb{E}_{s,a\sim d^S}[\mathcal{T}^{\pi}_0 Q(s,a )-Q(s,a)]
\end{equation}

 Under the $\chi^2$ divergence ($f^* = x^2/4+x$), we can simplify the Recoil objective as follows. Let $Q(s',\pi)=\mathbb{E}_{a'\sim\pi(s')}[Q(s',a')]$, then:

\begin{multline}
        \max_{\pi}\min_{Q}  \beta (1-\gamma)\mathbb{E}_{d_0,\pi}[Q(s,a)] +\reb{0.25} \mathbb{E}_{s,a\sim d_\text{mix}^{E,S}}[(\gamma Q(s',\pi)-Q(s,a))^2] + \beta\mathbb{E}_{s,a\sim d_{E}}[(\gamma Q(s',\pi)-Q(s,a))]\\ + (1-\beta)\mathbb{E}_{s,a\sim d_{S}}[(\gamma Q(s',\pi)-Q(s,a))] - (1-\beta)\mathbb{E}_{s,a\sim d_{S}}[(\gamma Q(s',\pi)-Q(s,a))]
\end{multline}

The last term of the ReCOIL objective cancels, simplifying to:
\begin{multline}
        \max_{\pi}\min_{Q}  \beta (1-\gamma)\mathbb{E}_{d_0,\pi}[Q(s,a)] +\reb{0.25} \mathbb{E}_{s,a\sim d_\text{mix}^{E,S}}[(\gamma Q(s',\pi)-Q(s,a))^2] + \beta\mathbb{E}_{s,a\sim d_{E}}[(\gamma Q(s',\pi)-Q(s,a))]
\end{multline}
Finally, rearranging terms we get:
\begin{multline}
    \label{eq:simplified_recoil_chi_square}
        \max_{\pi}\min_{Q}  \beta (1-\gamma)\mathbb{E}_{d_0}[Q(s,\pi)]+\gamma\beta\mathbb{E}_{s\sim d_{E}}[(Q(s',\pi)] -\beta\mathbb{E}_{s,a\sim d_{E}}[Q(s,a))] +\reb{0.25}\mathbb{E}_{s,a\sim d_\text{mix}^{E,S}}[(\gamma Q(s',\pi)-Q(s,a))^2] 
\end{multline}
Equivalently:
\reb{
\begin{multline}
    \label{eq:simplified_recoil_chi_square}
        \max_{\pi}\min_{Q}   (1-\gamma)\mathbb{E}_{d_0}[Q(s,\pi)]+\gamma\mathbb{E}_{s\sim d_{E}}[(Q(s',\pi)] -\mathbb{E}_{s,a\sim d_{E}}[Q(s,a))] + \frac{0.25}{\beta}\mathbb{E}_{s,a\sim d_\text{mix}^{E,S}}[(\gamma Q(s',\pi)-Q(s,a))^2] 
\end{multline}
}

Substituting the initial distribution as the dataset distribution \reb{$d_0=d^S$} (similar to~\citep{kostrikov2019imitation} and common practice in off-policy RL), and combining the first two terms which decrease $Q$ values at a mixture of dataset states(offline and expert) under the current policy we obtain the intuitive definition of ReCOIL from the paper which indicates to increase Q at expert-state actions and decrease Q at dataset states under current policy. But this can lead to unbounded Q functions. Finally, Eq.~\ref{eq:simplified_recoil_chi_square} in the practical algorithm (Algorithm~\ref{algo:recoil_algorithm_final}) implements $\max_\pi$ by performing an implicit maximization.

\section{Taking dual-RL from offline to online setting}
\label{ap:offline_to_online}
\textbf{Imitation Learning } Problem~\eqref{eq:dual_q_il_coverage} naturally extends to online IL, as the suboptimal data does not need to be static---it can be the replay buffer during online training. The corresponding algorithms generalize as well, since their key component is estimating the $Q^\pi$ function using \emph{off-policy} data. It is worth noting that $d^S$ is dynamically changing for online IL. In contrast, Eq.~\eqref{eq:dual-Q-imit} cannot be extended to online IL. \citet{garg2021iq} uses IQLearn in the online setting where they add additional regularization using bellman backups on $d^S$.
Our results suggest this to be unprincipled (also pointed out by \citet{al2023ls}), as only expert data samples can be leveraged in this formulation. \looseness=-1

\textbf{Reinforcement Learning} Again, all the above-discussed offline methods naturally extend to online settings~\citep{kostrikov2021offline, garg2023extreme, nakamoto2023cal}, as their off-policy nature extends beyond the offline setup. Our analysis still holds, where the regularization distribution $d^O$ becomes the visitation distribution of the replay buffer $d^R$. It is worth noting that $d^R$ is dynamically changing over the course of training.

\section{Implementation and Experiment Details}
\label{ap:experiment_details}

\subsection{Offline IL: ReCOIL algorithm and implementation details}

The algorithm for ReCOIL can be found in  Algorithm~\ref{algo:recoil_algorithm_final}.  We base the ReCOIL implementation on the official implementation of XQL~\citep{garg2023extreme} and IQL~\citep{kostrikov2021offline}. Our network architecture mimics theirs and uses the same data preprocessing techniques. 

In our set of environments, we keep the same hyper-parameters (except $\lambda$ - parameter that intuitively controls the pessimism or the upper expectile of $Q$ function) across tasks - locomotion, adroit manipulation, and kitchen manipulation.  For each environment, the values of $\lambda$ are searched between [2.5,5,10]. We keep a constant batch size of 256 across all environments. For all tasks we average mean returns over 10 evaluation trajectories and 7 random seeds. We add Layer Normalization~\citep{lei2016layer} to the value networks for all environments. Full hyper-parameters we used for experiments are given in Table \ref{tab:recoil-hp}. Although there might be better alternatives for implicit maximization, we found the implicit maximizer from~\citep{garg2023extreme} to be especially performant in the imitation setting. For policy update, using Advantage weighted regression, we use the temperature $\alpha$ to be 3 for MuJoCo locomotion environments and to be 0.5 for kitchen environments. The resembles prior work~\citep{kostrikov2021offline}.

\paragraph{Numerical Stability: } In practice a naive implementation of ReCOIL update for $Q_\phi$ in equation~\ref{eq:recoil_qphi_update} suffers from numerical instability to learning $Q$-functions that are unbounded and since our objective maximizes $Q$-values at expert distribution, the $Q$ values can be arbitrarily large without any grounding. The equation for $Q$-update from ReCOIL is given by:
\begin{equation}
    \resizebox{0.92\textwidth}{!}{$ \mathcal{L}(\phi) = \beta (\E{d^S,\pi(a|s)}{Q_\phi(s,a)}-\E{d^E(s,a)}{Q_\phi(s,a)}) + \E{s,a\sim \demix}{(\gamma V_\psi(s')-Q_\phi(s,a))^2}$},
\end{equation}
To avoid this numerical instability we make a small modification to the objective, that upper bounds the $Q$-function regression target as follows:

\begin{equation}
    \resizebox{0.92\textwidth}{!}{$ \mathcal{L}(\phi) = \beta (\E{d^S,\pi(a|s)}{Q_\phi(s,a)}-\E{d^E(s,a)}{(Q_\phi(s,a)-Q_{max})^2} + \E{s,a\sim \demix}{(\gamma V_\psi(s')-Q_\phi(s,a))^2}$},
\end{equation}
Such modifications are inspired by~\citep{sikchi2022ranking, al2023ls} which have found that bounding targets can make learning significantly more stable.


Hyperparameters for our proposed off-policy imitation learning method \texttt{ReCOIL} are shown in Table~\ref{tab:recoil-hp}. 
\begin{table}[h!]
  \begin{center}
    \begin{tabular}{l|c}
      \toprule 
      \textbf{Hyperparameter} & \textbf{Value}\\
      \midrule 
      Policy learning rate & 3e-4\\
      Value learning rate & 3e-4\\
      $f$-divergence &   $\chi^2$\\
      max-clip (loss clipping) &  7 \\     
      MLP layers &  (256,256)\\
      LR decay schedule & cosine\\
       $Q_{max}$ & 200\\
      \bottomrule 
    \end{tabular}
  \end{center}
  \caption{Hyperparameters for \texttt{ReCOIL}. }
      \label{tab:recoil-hp}
\end{table}

\subsection{Offline Imitation Learning Experiments}
\label{ap:offline_il_experiment_details}

\textbf{Environments:} For the offline imitation learning experiments we focus on 10 locomotion and manipulation environments from the MuJoCo physics engine \citep{todorov2012mujoco}. These environments include Hopper, Walker2d, HalfCheetah, Ant, Kitchen, Pen, Door, Hammer, and Relocate. The MuJoCo environments used in this work are \href{https://github.com/deepmind/mujoco/blob/main/LICENSE}{licensed under CC BY 4.0} and the datasets used from D4RL are also \href{https://github.com/Farama-Foundation/D4RL/blob/master/LICENSE}{licensed under Apache 2.0}.

\textbf{Suboptimal Datasets:} For the offline imitation learning task, we utilize offline datasets consisting of environment interactions from the D4RL framework \citep{fu2020d4rl}. Specifically, we construct suboptimal datasets following the composition approach introduced in SMODICE \citep{ma2022smodice}. The suboptimal datasets, denoted as 'random+expert', 'random+few-expert', 'medium+expert', and  'medium+few-expert' combine expert trajectories with low-quality trajectories obtained from the "random-v2" and "medium-v2" datasets, respectively. For locomotion tasks, the 'x+expert' dataset (where x is 'random' or 'medium') contains a mixture of some number of expert trajectories ($\le$ 200) and $\approx$1 million transitions from the "x" dataset. The 'x+few-expert' dataset is similar to `x+expert,' but with only 30 expert trajectories included. For manipulation environments we consider only 30 expert trajectories mixed with the complete 'x' dataset of transitions obtained from D4RL. 

\textbf{Expert Dataset:} To enable imitation learning, an offline expert dataset is required. In this work, we use 1 expert trajectory obtained from the "expert-v2" dataset for each respective environment.

\textbf{Baselines:} To benchmark and analyze the performance of our proposed methods for offline imitation learning with suboptimal data, we consider four representative baselines in this work: SMODICE \citep{ma2022smodice}, RCE \citep{eysenbach2021replacing}, ORIL \citep{zolna2020offline}, and IQLearn \citep{garg2021iq}. We exclude DEMODICE \citep{kim2022demodice} from the comparison, as SMODICE has been shown to be competitive \citep{ma2022smodice}. SMODICE is an imitation learning method based on the dual framework, assuming a restrictive coverage. ORIL adapts the generative adversarial imitation learning (GAIL) \citep{ho2016generative} algorithm to the offline setting, employing an offline RL algorithm for policy optimization. The RCE baseline combines RCE, an online example-based RL method proposed by~\citet{eysenbach2021replacing}.
RCE also uses a recursive discriminator to test the proximity of the policy visitations to successful examples. \citep{eysenbach2021replacing}, with TD3-BC \citep{fujimoto2021minimalist}. Both ORIL and RCE utilize a state-action based discriminator similar to SMODICE, and TD3-BC serves as the offline RL algorithm. All the compared approaches only have access to the expert state-action trajectory.

The open-source implementations of the baselines SMODICE, RCE, and ORIL provided by the authors \citep{ma2022smodice} are employed in our experiments. We use the hyperparameters provided by the authors, which are consistent with those used in the original SMODICE paper \citep{ma2022smodice}, for all the MuJoCo locomotion and manipulation environments.

\subsection{Calculation of $f^*_p$ for $f$-DVL under practical considerations}
\label{ap:f_star_p_practical}
The main practical consideration when optimizing $f^*_p$ is that the function is not well defined in $\mathbb{R}$. We extend the domain of $f^*_p$ from a semi-closed interval $[l, \infty)$ for certain $l \in \R$ to  the set of real numbers $\R$. 
Such extension and the behavior of $f^*_p$ is described in the proof of proposition~\ref{thm:implicit_maximizer}, but we will discuss it in more detail in this section.

Let us start from Eq.\ref{eq:f_star_p_def}. Here, the domain of $f^*_p$ is the same as $f^{'-1}$.
Recall that $f$ only admits a domain of $\mathbb{R}^+ = [0, \infty)$. As a consequence, the function ${f'}^{-1}$ 
has a limited domain $[l,\infty)$ for certain $l \in \R$ (To see this, first note that $f'$ is non-decreasing as $f$ is convex; further,
since the domain of $f$ is bounded from below, the range of $f'$ is also bounded from below.).
The behavior of $f^*_p: [l, \infty) \mapsto \R_+$ is then described as:
\[
   f^*_p(x)= 
\begin{cases}
    f^*(x) ,& \text{if } {f'}^{-1}(x)>0\\
    C=-f(0),              & \text{otherwise}
\end{cases}
\]
Next, we extend the domain of $f^*_p$ to $\R$. We use $\bar{f^*_p}$ to denote the extended function.
A natural choice is to take Eq~\ref{eq:f_star_p_extension} and extend it to $\R$. We also note that similar extensions can be found in prior work~\citep{goldfieldinfotheory,wiebelfdivergence}. 

\paragraph{$\chi^2$ divergence}
In practice, using the definition of $f^*$ for $\chi^2$, we use a smoother surrogate objective that still maintains the property of implicit maximization for Equation~\ref{eq:ap_implicit_maximization_general} and concisely write:
\begin{equation}
    \bar{f^*_p} = \max(C, x^2/4+x).
\end{equation}


\paragraph{Total Variation}
For the special case of total variation divergence, note that the convex conjugate $f^*(y)$ exists and is given by $f^*(y)$=$y$ if $y \in [-\frac{1}{2}, \frac{1}{2}]$ otherwise $\infty$ , even if $f'^{-1}$ does not exist. A concise proof can be found in Example 8.1 from~\citep{goldfieldinfotheory}. The reason is basically that a closed-form solution for convex conjugate does not exist as equation~\ref{eq:sol_w_star} in our paper no longer follows ($f'$ is not invertible). We recover $f^*_p$ for total-variation divergence using the definition in Eq~\ref{eq:f_star_p_extension} similar to $\chi^2$, as follows:
\begin{equation}
     f^*_p(x)= 
\begin{cases}
    x ,& \text{if } 0.5>x>0\\
    \infty ,& \text{if } x>0.5\\
    C=-f(0),              & \text{otherwise}
\end{cases}
\end{equation}
\reb{In practice we use a smooth extension of $f^*_p$ for TV divergence given by $\max (-f(0),x)$}
\paragraph{Practical Choice of $C$}
While Eq~\ref{eq:f_star_p_extension} suggests setting $C$ dependent on the corresponding $f$-divergence, we found a single value of $C=0$ to be a robust choice across our experiments. A comparison between using $f^*_p=\{-f(0)~\text{if}~ x<-4,x^2/4+x ~\text{otherwise}\})$ and $f^*_p=\max(0,x^2/4+x)$ for $\chi^2$ divergence can be found in Table~\ref{tab:d4rl_fstarp_ablation} below. Choosing $C=0$ instead of $-f(0)$ led to performance improvements.

\begin{table}[H]
\centering
\setlength\tabcolsep{3pt}
\renewcommand{\arraystretch}{1.0}
\vspace{5pt}
\resizebox{0.95\textwidth}{!}{
\begin{tabular}{l||ccc}
\multicolumn{1}{c||}{Dataset} &  \bf{$f$-DVL} ($\chi^2$, $f^*_p=\max(0,x^2/4+x)$) & \bf{$f$-DVL} (TV) & \bf{$f$-DVL} ($\chi^2$, $f^*_p=\{-f(0)~\text{if}~ x<-4,x^2/4+x ~\text{otherwise}\})$ )\\\hline
halfcheetah-medium-v2 &  \highlight{47.7} & 47.5 & 46.19\\ 
hopper-medium-v2 & 63.0& 64.1 & \highlight{78.66}\\ 
walker2d-medium-v2 &  80.0 & \highlight{81.5} & 76.85\\ 
halfcheetah-medium-replay-v2  & 42.9 & \highlight{44.7} & 42.91\\ 
hopper-medium-replay-v2  & 90.7&\highlight{98.0} & 97.73\\ 
walker2d-medium-replay-v2  &52.1&68.7&\highlight{73.5}  \\ 
halfcheetah-medium-expert-v2 &89.3 &\highlight{91.2} & 89.3 \\ 
hopper-medium-expert-v2 &\highlight{105.8} & 93.3 & 94.5\\ 
walker2d-medium-expert-v2  &\highlight{110.1} & 109.6 &  106.54\\ \hline 

kitchen-complete-v0 & \highlight{67.5} &65.71 & 67.14 \\ 
kitchen-partial-v0 &58.8 & \highlight{70.0} & 48.2\\ 
kitchen-mixed-v0  & \highlight{53.75}&52.5 &  52.4 \\ \hline \hline
\end{tabular}
}
\caption{The normalized return of offline RL methods on D4RL tasks. Shows comparison of setting the cutoff constant for $f^*_p$ to be $C=-f(0)$ vs $C=0$}
\label{tab:d4rl_fstarp_ablation}
\end{table}

\subsection{Online and Offline RL: $f$-\texttt{DVL} Algorithm and implementation details}

\paragraph{Rewriting of \dualV using temperature parameter $\lambda$ instead of $\alpha$:} We found rewriting dualV using temperature parameter $\lambda$ instead of $\alpha$ to be particularly useful  in reducing the number of hyperparameters to tune in order to obtain strong learning performance. We replace the temperature parameter from $\alpha$ to $\lambda$. Notice that our initial \dualV formulation used the temperature parameter $\alpha$ as follows:
\begin{align}
\colorbox{Goldenrod!30}{\dualV} \;
\min_{V} {(1-\gamma)}\E{s \sim d_0}{V(s)} +{\alpha}\E{(s,a)\sim d^O}{f^*_p\left(\left[\mathcal{T}V(s,a)-V(s))\right]/
\alpha\right)},
\end{align}
The temperature parameter $\alpha$ captures the tradeoff between the first term which seeks to minimize V vs the second term which seeks to maximize V and set it to the maximum value possible when taking various actions from that state onwards. Depending on different $f$ generator functions we would require tuning this parameter as it has a non-linear dependence on the entire optimization problem through the function $f$. Instead we consider a simpler objective, that we observe to empirically reduce hyperparameter tuning significantly by trading off linear between the first term and the second term using parameter $\lambda$. This modification is used in all of our experiments for RL and IL. 
\begin{align}
\colorbox{Goldenrod!30}{\dualV (\text{rewritten})} \;
\min_{V} {(1-\lambda)}\E{s \sim d_0}{V(s)} +{\lambda}\E{(s,a)\sim d^O}{f^*_p\left(\left[\mathcal{T}V(s,a)-V(s))\right]\right)},
\end{align}

\label{ap:xql++algorithm_implementation}
\begin{algorithm}[t]
\algsetup{linenosize=\tiny}
\caption{\fdvl (Under Stochastic Dynamics)}
\label{algo:meta_algorithm}
\begin{algorithmic}[1]
    \STATE Initialize $Q_\phi$, $V_\theta$, and $\pi_\psi$, temperature $\alpha$, weight $\lambda$
    \STATE Let $\mathcal{D} =( s, a, r, s')$ be data from $\pi_{\mathcal{D}}$ (offline) or replay buffer (online)
    \FOR{$t=1..T$ iterations}
        \STATE  Train $Q_\phi$ using $\min_{\phi} \mathcal{L}(\phi)$:\\
        \begin{equation*}
            \mathcal{L}(\phi) =  \E{s,a,s'\sim \mathcal{D}}{(Q_\phi(s,a) - (r(s,a) + V(s')))^2}
        \end{equation*}
        \STATE Train $V_\theta$ using $\min_{\theta} \mathcal{J}(\theta)$\\
          
\begin{equation*}
\resizebox{0.91\hsize}{!}{$\mathcal{J}(\theta)= \begin{cases}
 (1-\lambda)\E{s,a\sim D}{V_{\theta}(s)}+ \lambda\E{s,a\sim D}{\max(\bar{Q}_\phi(s,a)-V_{\theta}(s),0)} &\text{TV}\\
 (1-\lambda)\E{s,a\sim D}{V_{\theta}(s)}+ \lambda\E{s,a\sim D}{\max((\bar{Q}_\phi(s,a)-V_{\theta}(s))+ 0.5(\bar{Q}_\phi(s,a)-V_{\theta}(s))^2,0)} &\text{$\chi^2$}\\
(1-\lambda)\E{s,a\sim D}{V_\theta(s)} +\lambda \E{s,a\sim D}{\text{exp}(\left(\left[\bar{Q}_\phi(s,a)-V_\theta(s))\right]-1\right)} &\text{RKL/XQL}
\end{cases}$}
\end{equation*}
         \STATE Update $\pi_\psi$ via $\max_\psi \mathcal{M}(\psi)$:
        \begin{equation}
             \mathcal{M}(\psi) =  \mathbb{E}_{s,a\sim\mathcal{D}}[e^{\alpha(Q_\phi(s, a) - V_\theta(s))} \log \pi_\psi(s|a)].
        \end{equation}
        \vspace{0.025in}
    \ENDFOR
\end{algorithmic}
\end{algorithm}

\textbf{Offline RL}: Algorithm~\ref{ap:xql++algorithm_implementation} gives the algorithm for \fdvl. This section provides additional offline RL experimences along with complete hyper-parameter and implementation details. Figure \ref{fig:offline_rl_plots} shows learning curves for all the environments. $f$-\texttt{DVL} exhibits as fast convergence as XQL but avoids the numerical instability of XQL with one hyperparameter across each set of environments.
We base our implementation of $f$-\texttt{DVL} off the official implementation of XQL~\citep{garg2023extreme} and IQL from \citet{kostrikov2021offline}. Our network architecture mimics theirs and uses the same data preprocessing techniques. 

In our set of environments, we keep the same hyper-parameter across sets of tasks - locomotion, adroit manipulation, kitchen-manipulation, and antmaze. Contrary to XQL, we find no need to use tricks like gradient clipping to stabilize learning. For each set of environment, the values of $\lambda$ were tuned via hyper-parameter sweeps over a fixed set of values $[0.65, 0.7, 0.75, 0.8, 0.9]$. We keep a constant batch size of 256 across all environments. For MuJoCo locomotion tasks we average mean returns over 10 evaluation trajectories and 7 random seeds. For the AntMaze tasks, we average over 1000 evaluation trajectories. We add Layer Normalization \citep{lei2016layer} to the value networks for all environments. For policy update, using Advantage weighted regression, we use the temperature $\alpha$ to be 3 for MuJoCo locomotion environments and to be 0.5 for kitchen environments. The resembles prior work~\citep{kostrikov2021offline}. Full hyper-parameters we used for experiments are given in Table \ref{tab:offline_hyperparameters}.

\subsection{Online RL Experiments with \texttt{$f$-DVL}}

\textbf{Online RL}: We base the implementation of SAC on \href{https://github.com/denisyarats/pytorch_sac}{pytorch\_sac} and XQL~\citep{garg2023extreme}.  Like in offline experiments, hyper-parameters were left as default except for $\lambda$, which we tuned between $[0.6, 0.7, 0.8]$ and found a single value to work best across all environments. This was in contrast to XQL's finding which required per environment different hyperparameter. Also, as opposed to XQL we required no clipping of the loss function. We test our method on 7 random seeds for each environment.

\begin{table}[]
\centering
\scriptsize
\begin{tabular}{l|lll}
Env                          & Lambda $\lambda$ &   Batch Size & v\_updates \\ \midrule
halfcheetah-medium-v2        & 0.7       & 256   & 1 \\
hopper-medium-v2             & 0.7       & 256    & 1 \\
walker2d-medium-v2           & 0.7       & 256    & 1 \\
halfcheetah-medium-replay-v2 & 0.7        & 256    & 1 \\
hopper-medium-replay-v2      & 0.7         & 256    & 1 \\
walker2d-medium-replay-v2    & 0.7       & 256    & 1 \\
halfcheetah-medium-expert-v2 & 0.7        & 256   & 1 \\
hopper-medium-expert-v2      & 0.7       & 256  & 1 \\
walker2d-medium-expert-v2    & 0.7       & 256   & 1 \\
antmaze-umaze-v0             & 0.8        & 256   & 1 \\ 
antmaze-umaze-diverse-v0     & 0.8     & 256   & 1 \\
antmaze-medium-play-v0       & 0.8    & 256   & 1 \\
antmaze-medium-diverse-v0    & 0.8    & 256    & 1\\
antmaze-large-play-v0        &0.8     & 256  & 1 \\
antmaze-large-diverse-v0     & 0.8    & 256   & 1 \\ 
kitchen-complete-v0          & 0.8    & 256    & 1 \\
kitchen-partial-v0           & 0.8    & 256   & 1 \\
kitchen-mixed-v0             & 0.8     & 256 & 1 \\

pen-human-v0                 & 0.8      & 256    & 1  \\
hammer-human-v0              & 0.8     & 256   & 1  \\
door-human-v0                & 0.8       & 256     & 1 \\
relocate-human-v0            & 0.8    & 256    & 1  \\
pen-cloned-v0                & 0.8     & 256    & 1  \\
hammer-cloned-v0             &0.8      & 256     & 1  \\
door-human-v0                & 0.8       & 256    & 1  \\
relocate-human-v0            & 0.8       & 256     & 1  \\
\bottomrule
\end{tabular}
\caption{Offline RL Hyperparameters used for $f$-\texttt{DVL}. Lambda $\lambda$ is the value that controls the strength of the implicit maximizer. V-updates gives the number of value updates per Q updates.}
\label{tab:offline_hyperparameters}
\end{table}

\begin{table}[h!]
  \begin{center}
    \begin{tabular}{l|c}
      \toprule 
      \textbf{Hyperparameter} & \textbf{Value}\\
      \midrule 
      Policy updates $n_{pol}$ & 1\\
      Policy learning rate & 3e-4\\
      Value learning rate & 3e-4\\     
      MLP layers &  (256,256)\\
      LR decay schedule & cosine\\
      \bottomrule 
    \end{tabular}
  \end{center}
  \caption{Common hyperparameters for $f$-\texttt{DVL}. }
      \label{tab:fDVL-hp}

\vspace{10pt}
%
\centering
\begin{tabular}{l|l}
\toprule
\textbf{Hyperparameter} & \textbf{Value}\\
\midrule
Batch Size            & 1024                     \\
Learning Rate         & 0.0001                   \\
Critic Freq           & 1                            \\
Actor Freq            & 1                         \\
Actor and Critic Arch & 1024, 1024          \\
Buffer Size           & 1,000,000           \\
Actor Noise           & Auto-tuned  \\
Target Noise          & --             \\
\bottomrule
\end{tabular}
\caption{Hyperparameters for SAC.}
\label{tab:online_hparams}
\end{table}

\paragraph{Compute}

We ran all our experiments on a machine with AMD EPYC 7J13 64-Core Processor and NVIDIA A100 with a GPU memory consumption of <1000 MB per experiment. Our offline RL and IL experiments for locomotion tasks take 10-20 min and the online IL experiments took around 5-6 hours for 1 million timesteps.

\section{Additional Experimental Results}
\label{ap:additional_experiments}

\subsection{Why Dual-RL Methods are a Better Alternative to Traditional Off-Policy Algorithms}
\label{ap:dual_rl_are_better}

Our experimental evaluation aims to illustrate the benefits of the dual RL framework and analyze our proposed method for off-policy imitation learning. In the RL setting, we first present a case study on the failure of ADP-based methods like SAC~\citep{haarnoja2018soft} to make the most when bootstrapped with additional (helpful) data. This setting is what motivates the use of off-policy algorithms in the first place and is invaluable in domains like robotics~\cite {uchendu2022jump,nair2020awac}. Our results validate the benefit of utilizing the dual RL framework for off-policy learning. 

\textbf{The limitations of classical off-policy algorithms:} \reb{In this section, we test the sensitivity of an ADP method (SAC~\citep{haarnoja2018soft}) vs dual-RL methods in the case when we initialize the replay buffer of both styles of off-policy algorithms with expert or human demonstrated trajectories.}  At the beginning of training, each learning agent is provided with expert or human-demonstrated trajectories for completing the task. We add 1000 transitions from this dataset to the replay buffer for the off-policy algorithm to bootstrap from. SAC is able to leverage this helpful data and shows improved performance in Hopper-v2, where the action dimension is small. As the action dimension increases, \reb{the instability of SAC} becomes more apparent (see SAC+off policy data and SACfD plots in Figure~\ref{fig:off-policy-main}). We hypothesize that this failure in the online RL setting is primarily due to the training instabilities caused by TD-backups resulting in overestimation in regions where the agent's current policy does not visit. In Figure~\ref{fig:off-policy-reason}, we observe that overestimation indeed happens in environments with larger action dimensions and these overestimations take longer to get corrected and in the process destabilize the training. 
\begin{figure}[h]
\begin{center}
          \includegraphics[width=0.9\linewidth]{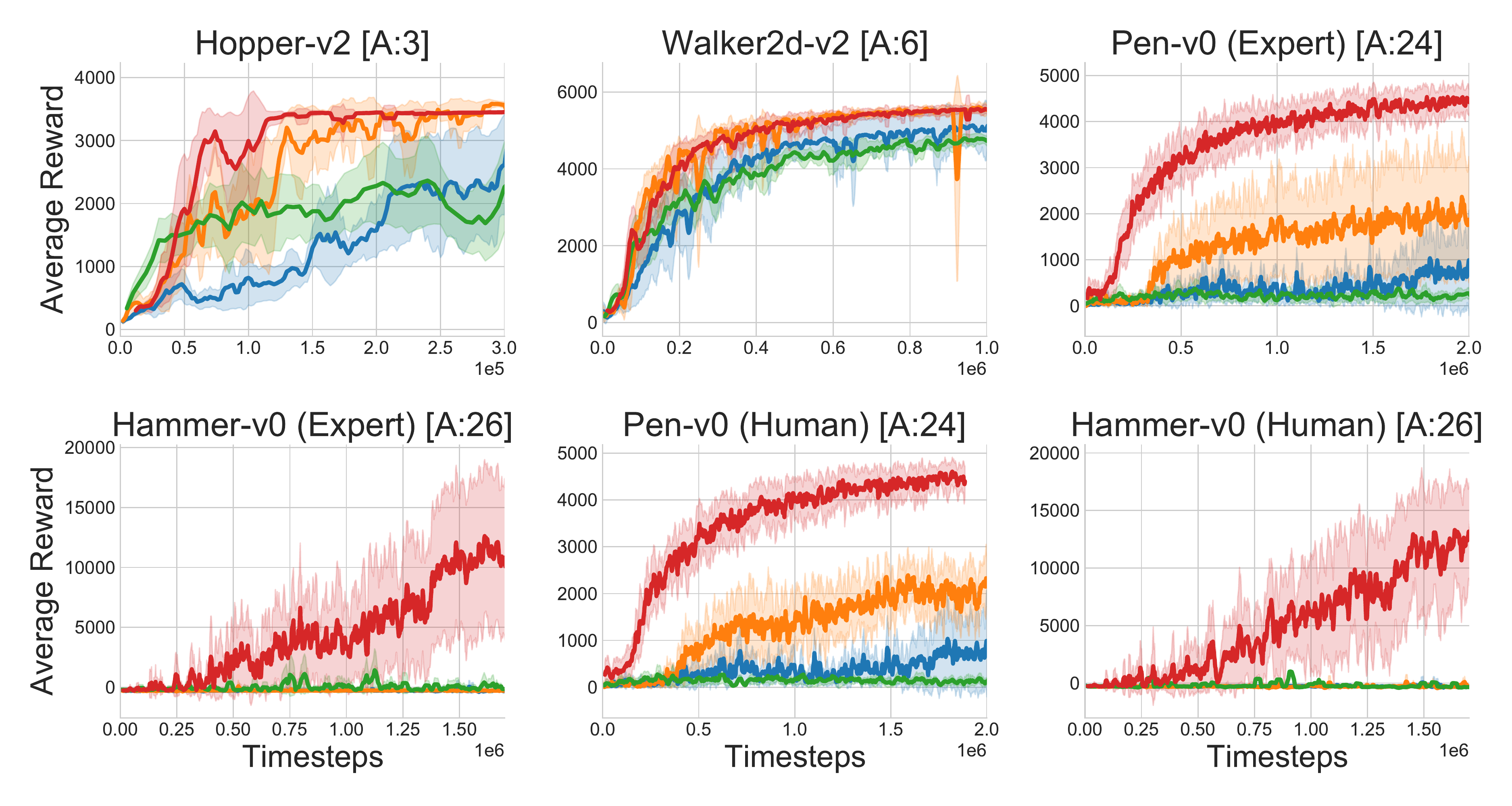}
      \\
      \vspace{-2.0mm}
    \includegraphics[width=1.0\linewidth]{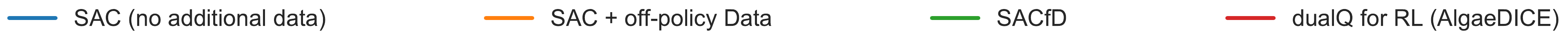}  
\end{center}
\caption{Despite the promise of off-policy methods, current methods based on ADP such as SAC fail when the dimension of action space, denoted by A, increases even when helpful data is added to their replay buffer. On other hand, dual-Q methods are able to leverage off-policy data to increase their learning performance}
\label{fig:off-policy-main}
\end{figure}

\begin{figure}[h]
\begin{center}
\includegraphics[width=0.6\linewidth]{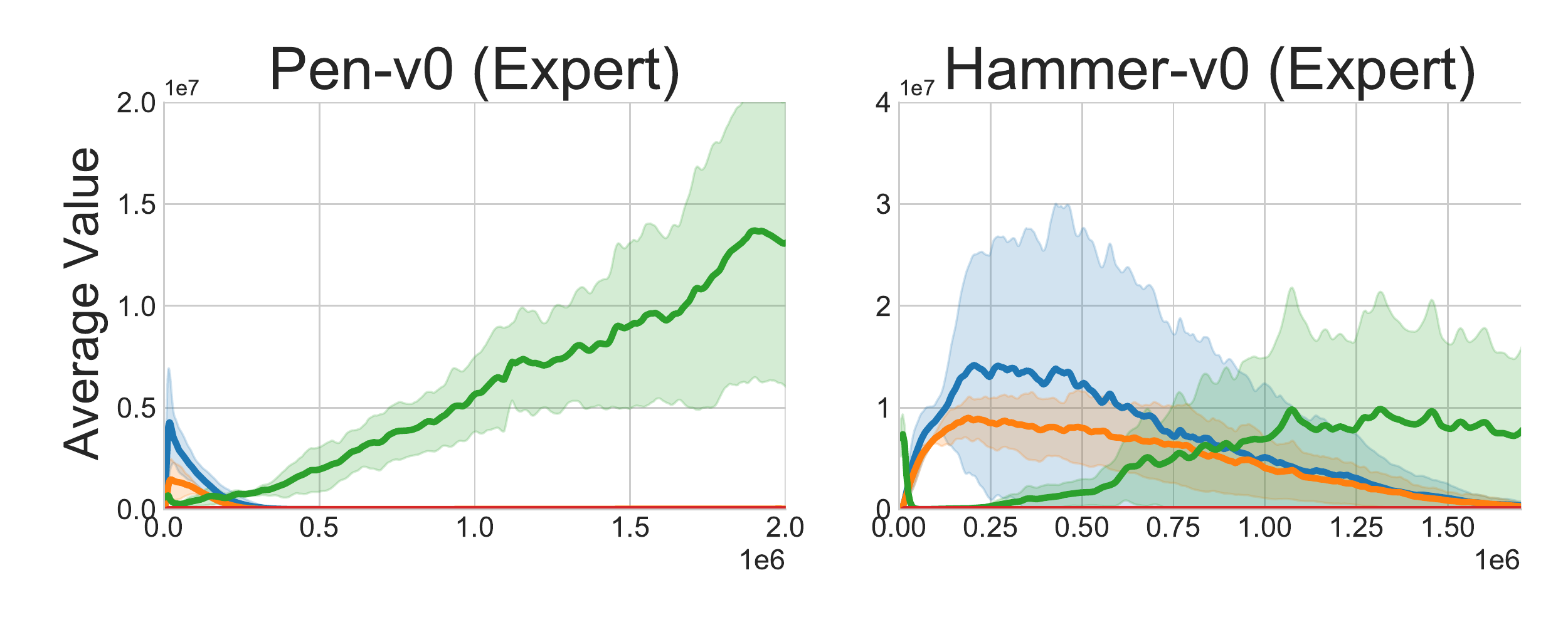}
      \\
\includegraphics[width=1.0\linewidth]{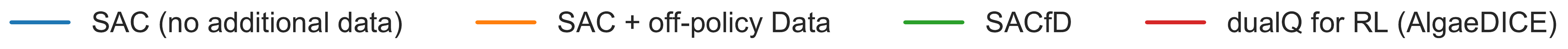}  
\end{center}
\caption{SAC and SACfD suffer from overestimation when off-policy data is added to the replay buffer. We hypothesize this to cause instabilities during training while dualQ has no overestimation.}
\label{fig:off-policy-reason}
\end{figure}
Figure~\ref{fig:off-policy-main} shows that the dual-RL method (AlgaeDICE) is able to leverage off-policy data to increase learning performance without any signs of destabilization. This can be attributed to the distribution correction estimation property of dual RL methods which updates the current policy using the corrected on-policy policy visitation~\citep{nachum2019algaedice}. Note, that we set the temperature $\alpha$ to a low value (0.001) to disentangle the effect of pessimism which is an alternate way to avoid overestimation.

\subsection{Training Curves for \texttt{ReCOIL} on MuJoCo tasks}
We show learning curves for \texttt{ReCOIL} in Figure~\ref{fig:offline_il} for locomotion tasks and Figure~\ref{fig:offline_il_adroit} for manipulation tasks below. \texttt{ReCOIL} training curves are reasonably stable while also being performant, especially in the manipulation setting where other methods completely fail.

\begin{figure}
\begin{center}
    \includegraphics[width=1.0\linewidth]{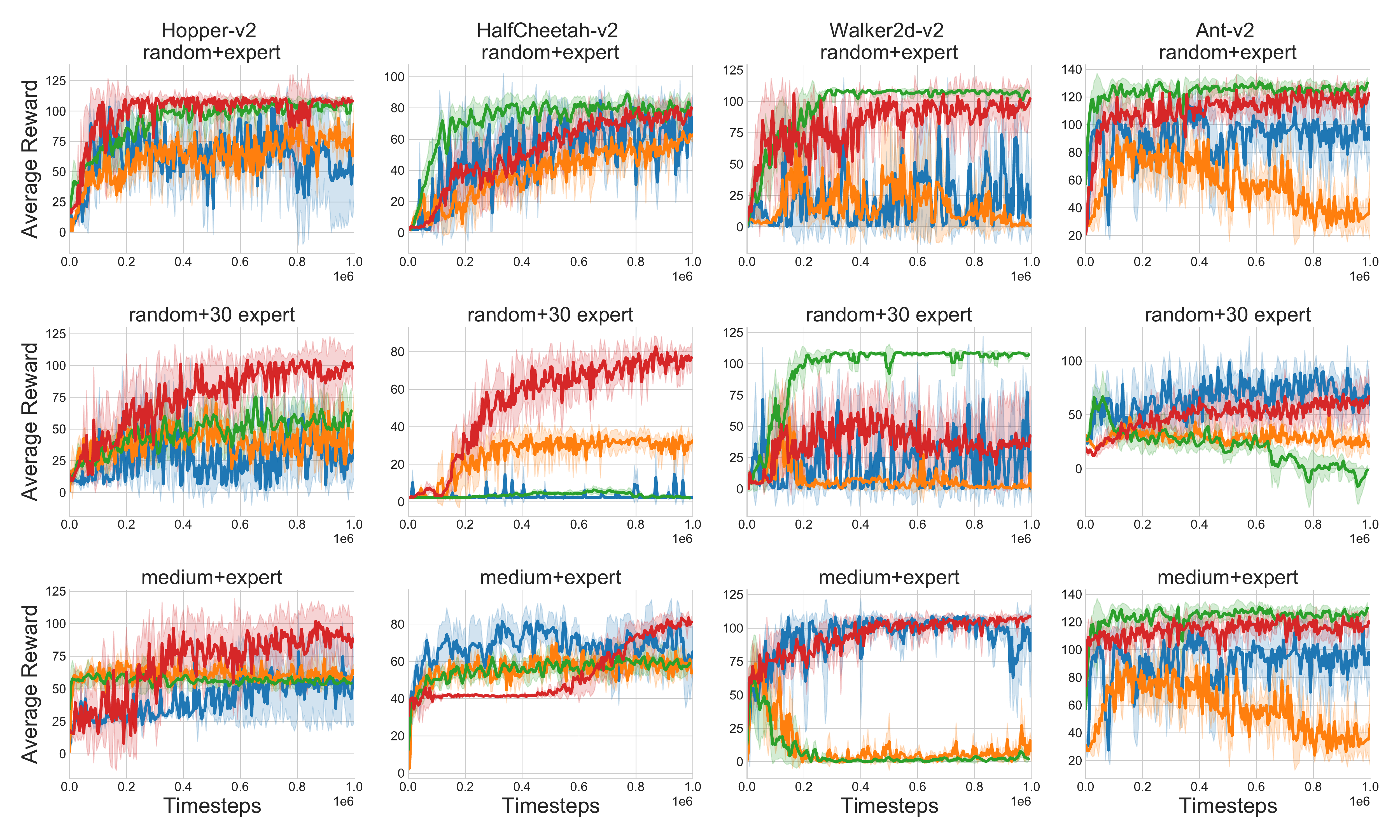}
      \\
       \includegraphics[width=0.9\linewidth]{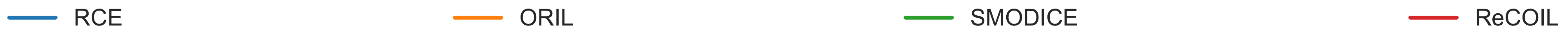}
      \vspace{-2.0mm}
\end{center}
\caption{Learning curves for ReCOIL showing that it outperforms baselines in the setting of learning to imitate from diverse offline data. The results are averaged over 7 seeds}
\label{fig:offline_il}
\end{figure}
\begin{figure}[h]
\begin{center}
    \includegraphics[width=1.0\linewidth]{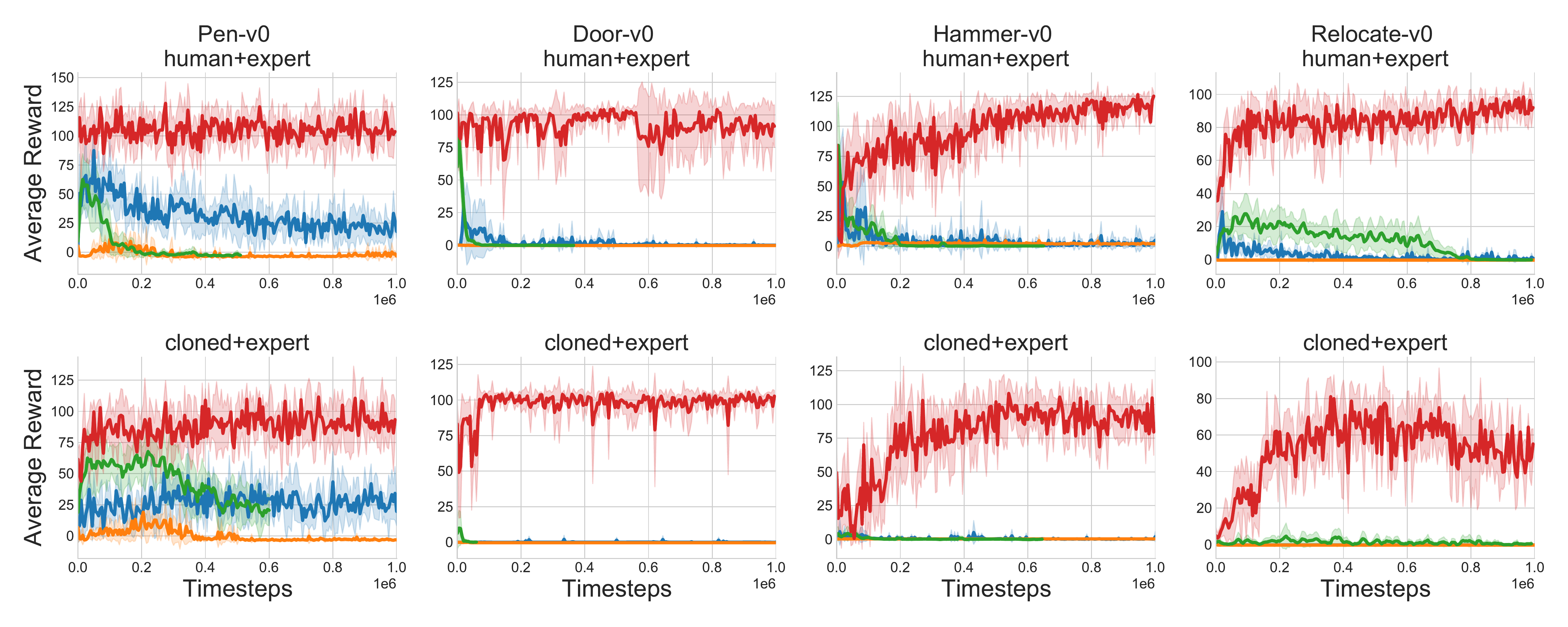}
      \\
       \includegraphics[width=0.9\linewidth]{figures/offline_il/offline-il-legend.pdf}
      \vspace{-2.0mm}
\end{center}
\caption{Learning curves for ReCOIL showing that it outperforms baselines in the setting of learning to imitate from diverse offline data. The results are averaged over 7 seeds}
\label{fig:offline_il_adroit}
\end{figure}
\subsection{Does \texttt{ReCOIL} Allow for Better Estimation of Agent Visitation Distribution?}
\label{ap:density_ratio_estimation}
\reb{
In this section, we consider the experiment of visitation estimation for a prespecified policy given expert data and suboptimal/replay data. Our experiments are tabular, so we can have an accurate estimate of the visitation of the policy by running rollouts in the MDP. We call this estimate the ground-truth policy visitation. We will estimate the accuracy of dual-RL methods to estimate visitation density by measuing MSE error against the ground truth agent/policy visitation. The property of Dual-RL methods to implicitly estimate density ratios (eg. Eq~\ref{ap:optimal_distribution_ratio}) has been studied before~\citep{nachum2020reinforcement}. We explain below how to extract visitation density ratios for any policy $\pi$ with \texttt{ReCOIL}. First, we discuss the setup for the experiment.}

\reb{We consider two settings in our experiments: (1) a 2-timestep MDP where agent states from state $s_0$ and transitions to one of the states $\{s_1,s_2,s_3,s_4,s_5\}$ (from left to right in Figure~\ref{fig:distribution_ratio_estimation1}) which are absorbing. In this setting, the replay buffer perfectly covers the unknown ground truth agent visitation. (2) a 2-D gridworld (Figure~\ref{fig:distribution_ratio_estimation2}) where the agent can move cardinally and the replay buffer distribution does not cover the unknown ground truth agent visitation. In both environments we have access to a policy at training time -- the task is to estimate this policy's visitation using the offline dataset of expert and suboptimal quality. Both figures also demonstrate the ground truth policy visitation that is used to compute the mean-squared evaluation loss and test the quality of our predictions.}

\reb{We consider the proposed \texttt{ReCOIL} method and investigate its ability to estimate distribution ratios correctly. We consider the inner optimization for \texttt{ReCOIL-Q}:}
\begin{align}
&\min_{Q(s,a)}  \beta (1-\gamma)\E{d_0(s),\pi(a|s)}{Q(s,a)} +\E{s,a\sim d_{mix}^{E,S}}{f^*(\gamma \sum_{s'} p(s'|s,a)\pi(a'|s')Q(s',a')-Q(s,a))}\nonumber \\
& - (1-\beta) \E{s,a\sim d^S}{\gamma \sum_{s'} p(s'|s,a)\pi(a'|s')Q(s',a')-Q(s,a)}
\end{align}

\reb{The following holds for the inner optimization for \texttt{ReCOIL-Q} when $Q$ is optimized:}
\begin{equation}
    f^{*'}(\gamma \sum_{s'} p(s'|s,a)\pi(a'|s')Q(s',a')-Q(s,a))=\frac{\beta d^\pi(s,a) + (1-\beta)d^S(s,a)}{\beta d^E(s,a)+(1-\beta)d^S(s,a)}
    \label{eq:mixture_density_ratio_1}
\end{equation}
\reb{Thus, given the visitation distribution of the replay buffer $d^R$, expert $d^E$ and the policy $\pi$, the inner optimization implicitly learns the distribution ratio in Eq~\ref{eq:mixture_density_ratio_1}, allowing us to infer agent visitation $d^\pi$.}

\reb{Our results (Figure~\ref{fig:distribution_ratio_estimation1} and \ref{fig:distribution_ratio_estimation2}) demonstrate that in the perfect coverage setting, \texttt{ReCOIL} is able to infer the agent policy visitation perfectly, and in the case of imperfect coverage is able to significantly outperform other methods (IQLearn and SMODICE). Note that we modify SMODICE with a $Q$ objective instead of $V$ objective to incorporate state-action expert data rather than relying on state-only expert data. IQLearn does not leverage replay data information, relying only on expert data to infer agent visitation. SMODICE's reward function $-\log \frac{d^S(s,a)}{d^E(s,a)}$, arising from its coverage assumption is ill-defined in parts of state space where the expert has no support leading to poor downstream density ratio estimation.}

\reb{Our results on the} 2-D gridworld environment that demonstrate the failures of a method that either do not utilize all available suboptimal data (IQ-Learn) or relies on a coverage assumption (SMODICE). We saw that \texttt{ReCOIL} is able to perfectly infer the agent's visitation when the replay buffer covers agent ground truth visitation perfectly (Fig~\ref{fig:distribution_ratio_estimation1}) and here we see that \texttt{ReCOIL}  is able to outperform baselines when the replay buffer has imperfect coverage over the agent's ground truth visitation (Fig~\ref{fig:distribution_ratio_estimation2}). In this task, the agent starts at (0,0) which is the top-left corner. The agent can only move in cardinal directions with deterministic dynamics. The agent has access to two sources of off-policy data - expert visitation and replay visitation. The problem is to estimate the agent's visitation distribution given access to the agent's policy using all the available transition data. IQLearn and SMODICE predict an agent's visitation that wildly differs from Agent's ground truth visitation distribution. While \texttt{ReCOIL} is not perfect as the coverage of the offline data is limited, we can estimate some visitation which is qualitatively very similar to the agent's ground truth visitation.

\begin{figure}[h]
\begin{center}
    \includegraphics[width=0.9\linewidth]{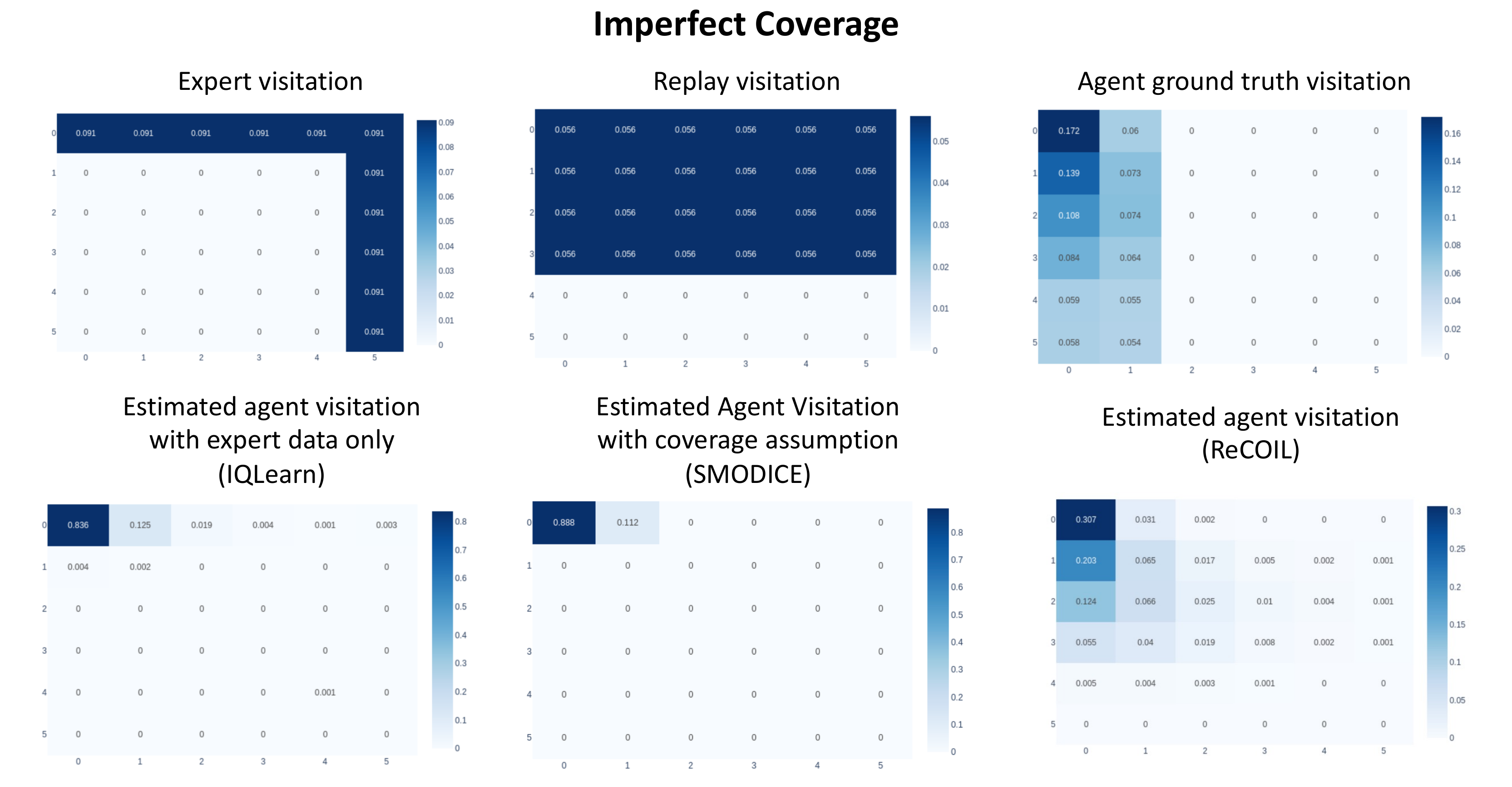}
      \\
      \vspace{-2.0mm}
\end{center}
\caption{Replay buffer consists of data that visits near the initial state (0,0), a setting commonly observed when training RL agents. We estimate the agent's policy visitation and observe \texttt{ReCOIL} to outperform both methods which rely on expert data only or use the replay data with coverage assumption}
\label{fig:distribution_ratio_estimation2}
\end{figure}

\subsection{\texttt{ReCOIL}: Qualitative Comparison with a Baseline}
\label{ap:qualitative_comparison}
In Figure~\ref{fig:imitation_errors}, we investigate qualitatively why other baselines fail where \texttt{ReCOIL} succeeds in high-dimensional tasks. A surprising finding is that the baseline we consider 'SMODICE' almost learns to imitate. It follows nearly the same actions as an expert but makes small mistakes along the way - eg. 'gripping the hammer too loose' or 'picking up the ball at a slightly wrong location'. SMODICE is unable to recover from such mistakes and ends up having low performance. \texttt{ReCOIL}, on the other hand, learns a performant task-solving policy from the same data.

\begin{figure}[h]
\begin{center}
      \includegraphics[width=0.8\linewidth]{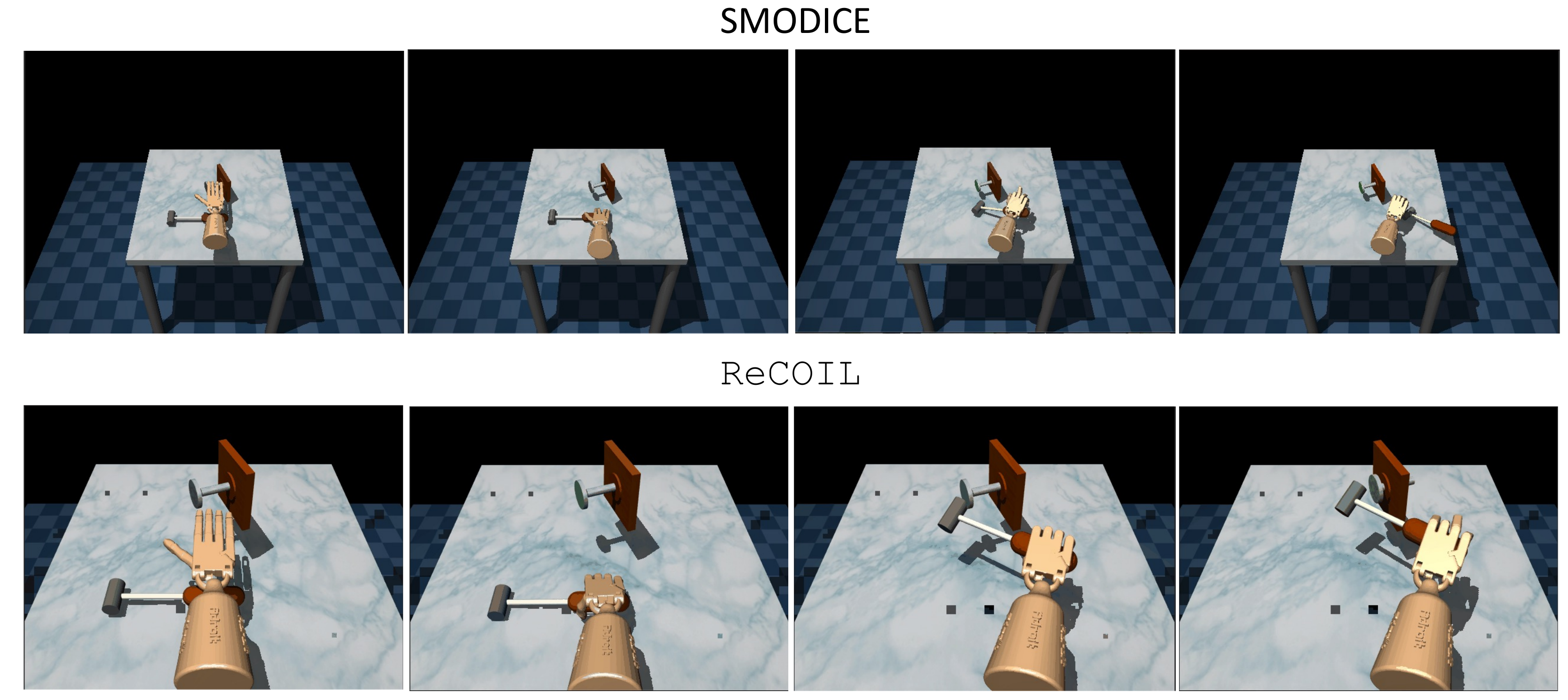}
      \\
      \includegraphics[width=0.8\linewidth]{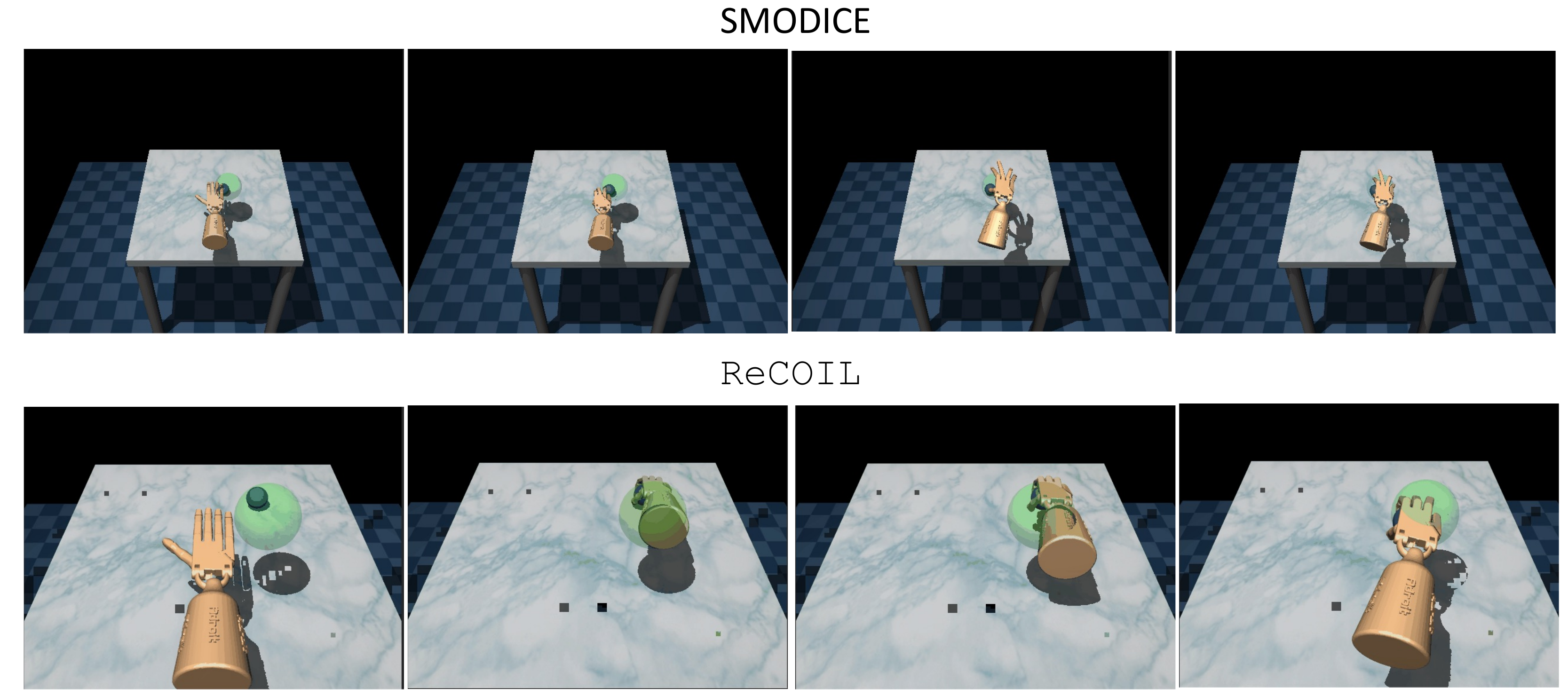}
      \vspace{-2.0mm}
\end{center}
\caption{Errors compound in imitation learning and recovery is of crucial importance. Figure demonstrate how SMODICE 'almost' imitates, figures out roughly what actions to take but does not realise once it has made a mistake. In Hammer environment, it grips the hammer too loose causing it to get thrown away and for relocate picks up just beside the ball missing the original task the expert intended to solve.}
\label{fig:imitation_errors}
\end{figure}

\subsection{Evaluation of \fdvl for Online RL}

Fig~\ref{fig:online_rl} shows that \fdvl \reb{is competitive to performant off-policy RL methods} in the online RL benchmarks.
\begin{figure*}[t]
\begin{center}
          \includegraphics[width=0.85\linewidth]{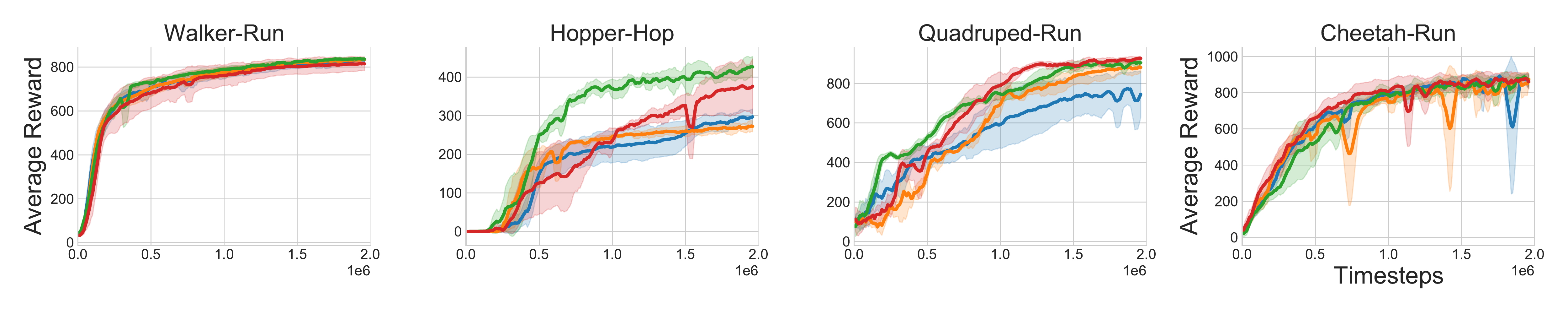}
      \\
      \vspace{-2.0mm}
    \includegraphics[width=0.8\linewidth]{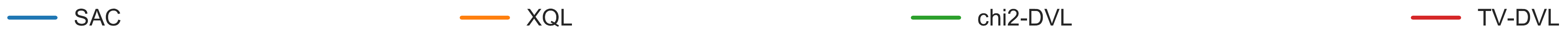}  
\end{center}
\vspace{-3.0mm}
\caption{Online RL: $f$-\texttt{DVL} \reb{is competitive to} \texttt{SAC} and \texttt{XQL}, particularly for \texttt{Hopper-Hop} and \texttt{Quadruped-Run} tasks. \looseness=-1}
\label{fig:online_rl}
\end{figure*}

\subsection{Training Curves for \fdvl on MuJoCo Tasks (Offline)}
Figure~\ref{fig:offline_rl_plots} shows the learning curves during training for \fdvl. \fdvl is able to leverage low-order conjugate $f$-divergences to give offline RL algorithms that more stable compared to XQL. XQL frequently crashes in the antmaze environment.

\begin{figure}[h]
\begin{center}
    \includegraphics[width=1.0\linewidth]{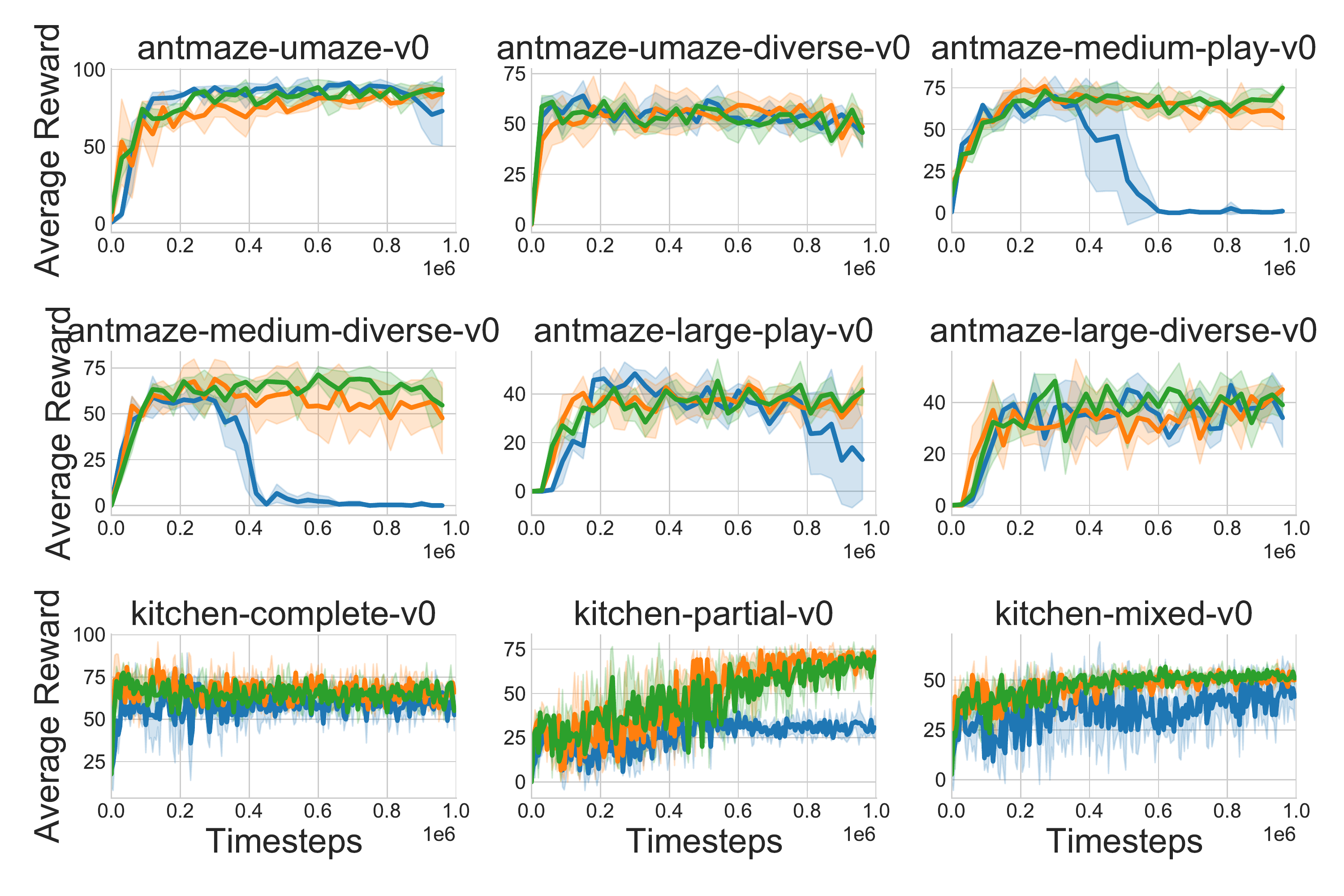}
      \\
       \includegraphics[width=0.9\linewidth]{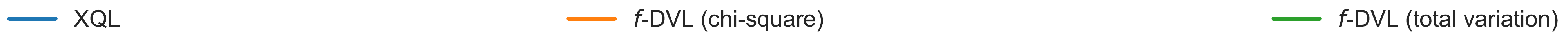}
      \vspace{-2.0mm}
\end{center}
\caption{Learning curves for \fdvl showing that it is able to leverage low-order conjugate $f$-divergences to give offline RL algorithms that more stable compared to XQL. The results are averaged over 7 seeds}
\label{fig:offline_rl_plots}
\end{figure}

\subsection{$f$-\texttt{DVL}: Complete Offline RL Results}

Table~\ref{tab:d4rl_complete} and Table~\ref{tab:franca_adroit} show complete results for benchmarking \fdvl on MuJoCo D4RL environments. Here we also show the author-reported results for XQL and the reproduced results (XQL(r)) using the metric of taking the average of the last iterate performance across seeds.

\begin{table}[h]
\centering
\scriptsize
\setlength\tabcolsep{3pt}
\renewcommand{\arraystretch}{1.0}
\caption{\small Averaged normalized scores on MuJoCo locomotion and Ant Maze tasks. XQL(r) denotes the reproduced results with author's implementation. We do not highlight XQL due to an incorrect evaluation strategy used in the work. Standard deviations for our proposed offline RL methods can be found below. }
\begin{tabular}{c|l||rrrrrrrr|cccc}
& \multicolumn{1}{c||}{Dataset} & BC & 10\%BC & DT  & TD3+BC & CQL & IQL & {XQL} & {XQL}(r) & \fdvl $\chi^2$ & \fdvl TV\\\hline
\parbox[t]{2mm}{\multirow{9}{*}{\rotatebox[origin=c]{90}{\textbf{Gym}}}} & ~~ 
halfcheetah-medium-v2 & 42.6 & 42.5 & 42.6  & \textbf{48.3} &44.0 & \textbf{47.4} & 47.7 & 47.4 & \textbf{47.7}$\pm${\scriptsize0.23} & 47.5$\pm${\scriptsize0.20}\\ & ~~ 
hopper-medium-v2 & 52.9 & 56.9 & 67.6 &  59.3 & 58.5 & 66.3 & 71.1 & \textbf{68.5} & 63.0$\pm${\scriptsize4.64}& 64.1$\pm${\scriptsize3.46}\\ &~~ 
walker2d-medium-v2 & 75.3 &75.0 &74.0 & \textbf{83.7} & 72.5 & 78.3 & 81.5 & 81.4 & 80.0$\pm${\scriptsize6.75} & {81.5}$\pm${\scriptsize3.47}\\ & ~~ 
halfcheetah-medium-replay-v2 & 36.6 & 40.6 & 36.6 &  \textbf{44.6} & \textbf{45.5} &{ 44.2} & 44.8 &44.1  & 42.9$\pm${\scriptsize1.79} & \textbf{44.7}$\pm${\scriptsize0.44}\\ & ~~ 
hopper-medium-replay-v2 & 18.1 & 75.9 & 82.7 & 60.9 & 95.0 &  94.7 & 97.3 &95.1  & 90.7$\pm${\scriptsize6.13}&\textbf{98.0}$\pm${\scriptsize4.62}\\ & ~~ 
walker2d-medium-replay-v2 & 26.0 & 62.5 & 66.6 &  \textbf{81.8} & 77.2 & 73.9 & 75.9 &58.0 &52.1$\pm${\scriptsize12.15}&68.7$\pm${\scriptsize7.20}  \\ & ~~ 
halfcheetah-medium-expert-v2 & 55.2 & \textbf{92.9} & 86.8 &   90.7 & {91.6} & 86.7 & 89.8 & 90.8 &89.3$\pm${\scriptsize2.42} &{91.2}$\pm${\scriptsize2.27} \\ & ~~ 
hopper-medium-expert-v2 &52.5 & \textbf{110.9} & {107.6}  & 98.0 & 105.4 & 91.5 & 107.1 & 94.0 &\textbf{105.8}$\pm${\scriptsize5.79} & 93.3$\pm${\scriptsize14.04} \\ & ~~ 
walker2d-medium-expert-v2 & 107.5 & 109.0 & 108.1  & \textbf{110.1} & 108.8 & \textbf{109.6} & 110.1 & \textbf{110.1} &\textbf{110.1}$\pm${\scriptsize0.29} & 109.6$\pm${\scriptsize1.46}\\ \hline 
 \parbox[t]{2mm}{\multirow{6}{*}{\rotatebox[origin=c]{90}{\textbf{AntMaze}}}} & ~~ 
antmaze-umaze-v0 & 54.6 & 62.8 & 59.2 & 78.6 & 74.0 & \textbf{87.5} & 87.2 & 47.7&{83.7}$\pm${\scriptsize5.90} &\textbf{87.7}$\pm${\scriptsize3.07} \\  & ~~ 
antmaze-umaze-diverse-v0 & 45.6 & 50.2 & 53.0 &  71.4 & \textbf{84.0} & 62.2 & 69.17 & 51.7   &  50.4$\pm${\scriptsize2.44}&48.4$\pm${\scriptsize9.95}\\  & ~~ 
antmaze-medium-play-v0 & 0.0 & 5.4 & 0.0 &  10.6 & 61.2 & \textbf{71.2} & 73.5 &  31.2 & 56.7$\pm${\scriptsize13.82}&\textbf{71.0}$\pm${\scriptsize5.90} \\  & ~~
antmaze-medium-diverse-v0 & 0.0 & 9.8 & 0.0 &  3.0 & 53.7 & \textbf{70.0} &67.8 &  0.0& 48.2$\pm${\scriptsize8.85}& 60.2$\pm${\scriptsize7.99}\\  & ~~ 
antmaze-large-play-v0 &0.0 &0.0 &0.0 &0.2 &15.8 & 39.6 & 41 &10.7 &36.0$\pm${\scriptsize5.82} & \textbf{41.7}$\pm${\scriptsize9.43}\\  & ~~ 
antmaze-large-diverse-v0 & 0.0 &6.0 & 0.0 & 0.0 & 14.9 & \textbf{47.5} & 47.3&31.28 &{44.5}$\pm${\scriptsize7.66} & 39.3$\pm${\scriptsize11.84}\\ \hline 
\parbox[t]{2mm}{\multirow{3}{*}{\rotatebox[origin=c]{90}{\textbf{Franka}}}} & ~~ 
kitchen-complete-v0 & 65.0 & -  & -  & - &43.8 & 62.5 & 72.5 & 56.7 &\textbf{67.5}$\pm${\scriptsize6.68} &61.3$\pm${\scriptsize7.95} \\ & ~~ 
kitchen-partial-v0 &38.0 & - & - & - & 49.8 & 46.3 & 73.8  & 48.6 &58.8$\pm${\scriptsize9.60} & \textbf{70.0}$\pm${\scriptsize1.82} \\ & ~~ 
kitchen-mixed-v0 & 51.5 & - & - &  - & 51.0 & 51.0 & 54.6 & 40.4  & \textbf{53.75}$\pm${\scriptsize5.32}&52.5$\pm${\scriptsize5.15} \\ \hline \hline
\end{tabular}
\label{tab:d4rl_complete}
\end{table}

\begin{table}[h]
\centering
\caption{\small Evaluation on Adroit tasks from D4RL.XQL-C (r) denotes the reproduced results with author's implementation. }
\scriptsize
\begin{tabular}{l||rrrrrrrr|cc}
Dataset &BC &BRAC-p &BEAR &Onestep RL &CQL & IQL & {XQL} & XQL(r) & \fdvl ($\chi^2$) & \fdvl (TV) \\ \hline
pen-human-v0 &63.9 &8.1 &-1.0 &- &37.5 & 71.5 & 85.5 & 63.5&67.1&64.1\\
hammer-human-v0 &1.2 &0.3 &0.3 &- & 4.4 &1.4 & 2.2 &1.4&2.6& 1.8 \\
door-human-v0 &2 &-0.3 &-0.3 &- & 9.9 &4.3 & 11.5 & 6.63 & 5.7&6.77 \\
relocate-human-v0 &0.1 &-0.3 &-0.3 &- &{0.2} &0.1 & 0.17 & 0.2 &{0.37} &0.12  \\
pen-cloned-v0 &37 &1.6 &26.5 & {60.0} &39.2 &37.3 & 38.6 & 25.25 &36.1&38.1 \\
hammer-cloned-v0 &0.6 &0.3 &0.3 & 2.1 & 2.1 & 2.1  & 4.3 &1.58 &1.64& 1.65 \\
door-cloned-v0 &0.0 &-0.1 &-0.1 &0.4 &0.4 & 1.6 & 5.9 & 0.69 &0.45&0.87 \\
relocate-cloned-v0 &-0.3 &-0.3 &-0.3 &-0.1 &-0.1 &-0.2 & -0.2 & -0.24 &-0.24 &-0.24\\ 
\end{tabular}
\label{tab:franca_adroit}
\end{table}

\subsection{Sensitivity of $f$-\texttt{DVL} (offline) with varying $\lambda$  on MuJoCo tasks}
\label{ap:dvl_sensitivity}

We ablate the temperature parameter, $\lambda$ for offline RL experiments using \fdvl in Figure~\ref{fig:offline_rl_chi_sensitivity_plots} and Figure~\ref{fig:offline_rl_tv_sensitivity_plots}. The temperature $\lambda$ controls the strength of KL penalization between the learned policy and the dataset behavior policy, and a small $\lambda$ is beneficial for datasets with lots of random noisy actions. In contrast, a high $\lambda$ favors more expert-like datasets. We observe that significantly less hyperparameter tuning is required compared to XQL as a single temperature value works well across a broad range of experiments.
 
\begin{figure}[h]
\begin{center}
    \includegraphics[width=1.0\linewidth]{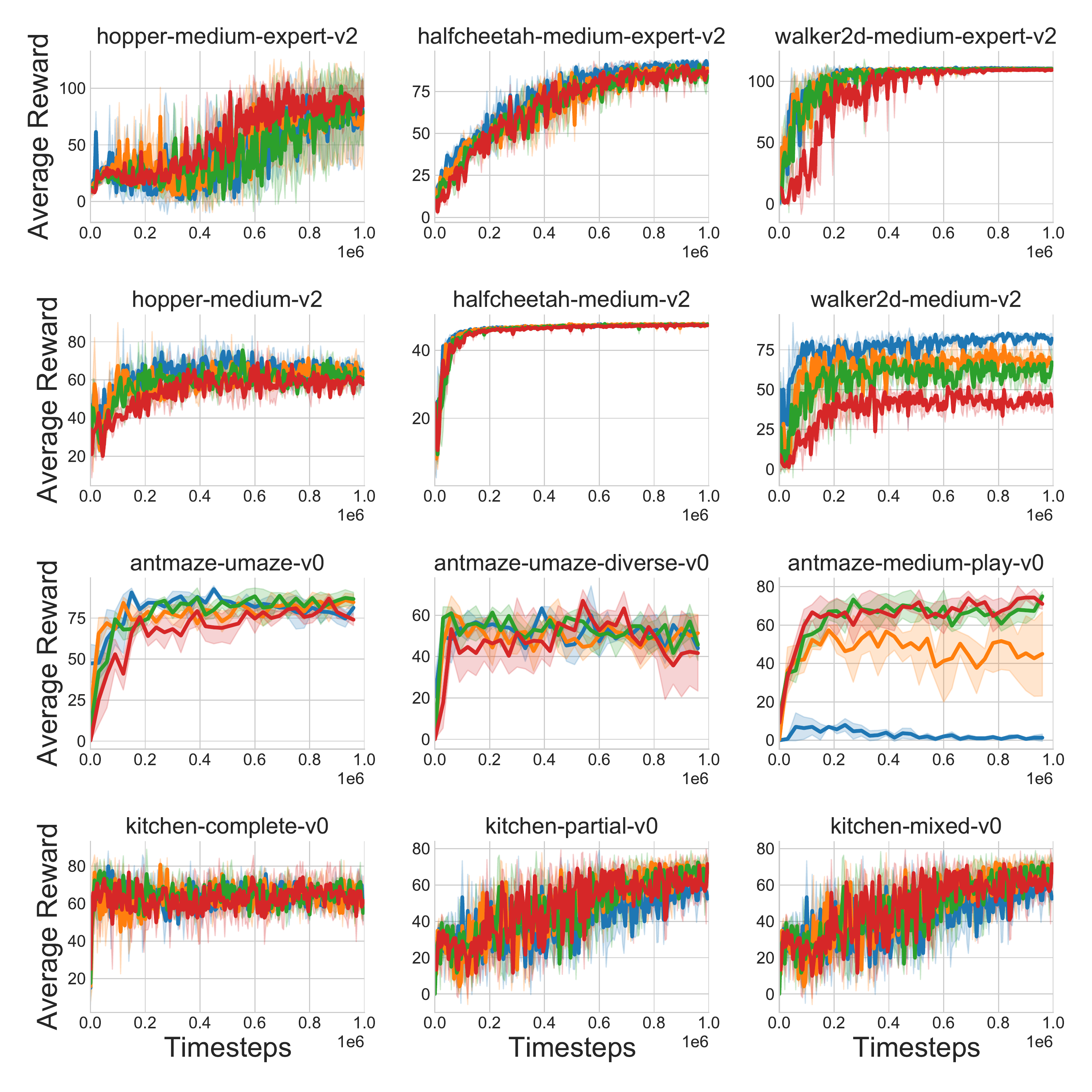}
      \\
       \includegraphics[width=0.9\linewidth]{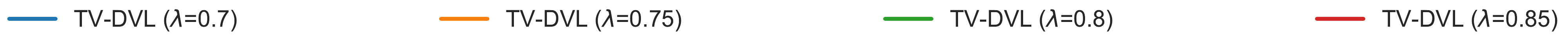}
      \vspace{-2.0mm}
\end{center}
\caption{Offline RL: Ablating the temperature parameter for \fdvl (Total variation). The plot shows the effect of temperature parameters on learning performance.}
\label{fig:offline_rl_tv_sensitivity_plots}
\end{figure}

\begin{figure}[h]
\begin{center}
    \includegraphics[width=1.0\linewidth]{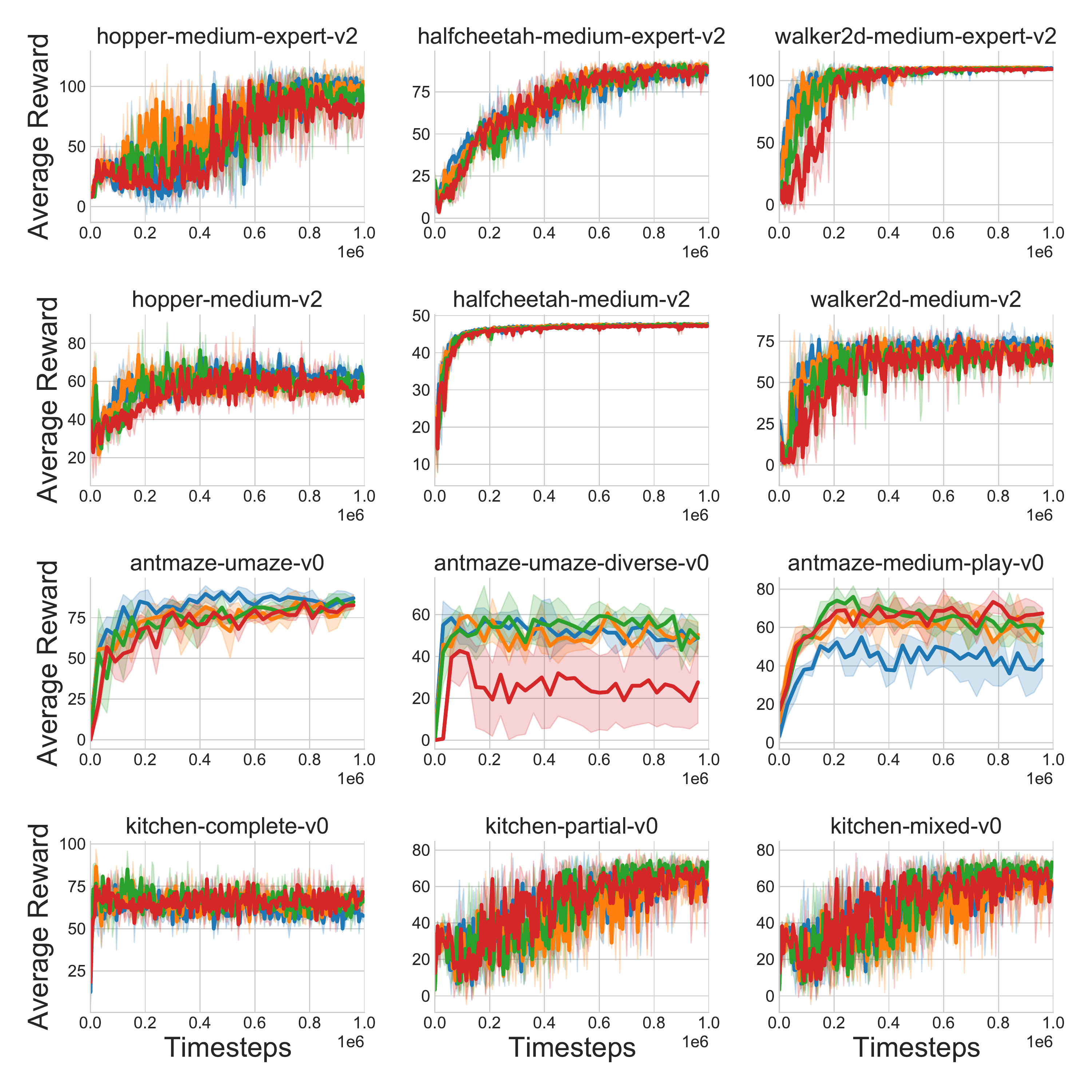}
      \\
       \includegraphics[width=0.9\linewidth]{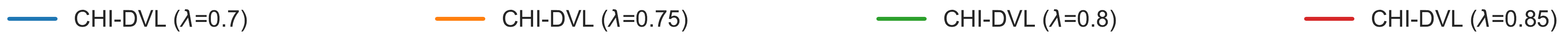}
      \vspace{-2.0mm}
\end{center}
\caption{Offline RL: Ablating the temperature parameter for \fdvl (Chi-square). The plot shows the effect of temperature parameters on learning performance.}
\label{fig:offline_rl_chi_sensitivity_plots}
\end{figure}
\subsection{Sensitivity of $f$-\texttt{DVL} (online) with varying $\lambda$  on MuJoCo tasks}
We ablate the temperature parameter $\lambda$ for online RL experiments using \fdvl in Figure~\ref{fig:online_rl_chi_sensitivity_plots} (chi-square) and Figure~\ref{fig:online_rl_tv_sensitivity_plots} (TV). We observe that significantly less hyperparameter tuning is required compared to XQL as a single temperature value works well across a broad range of experiments.
\begin{figure}[h]
\begin{center}
    \includegraphics[width=1.0\linewidth]{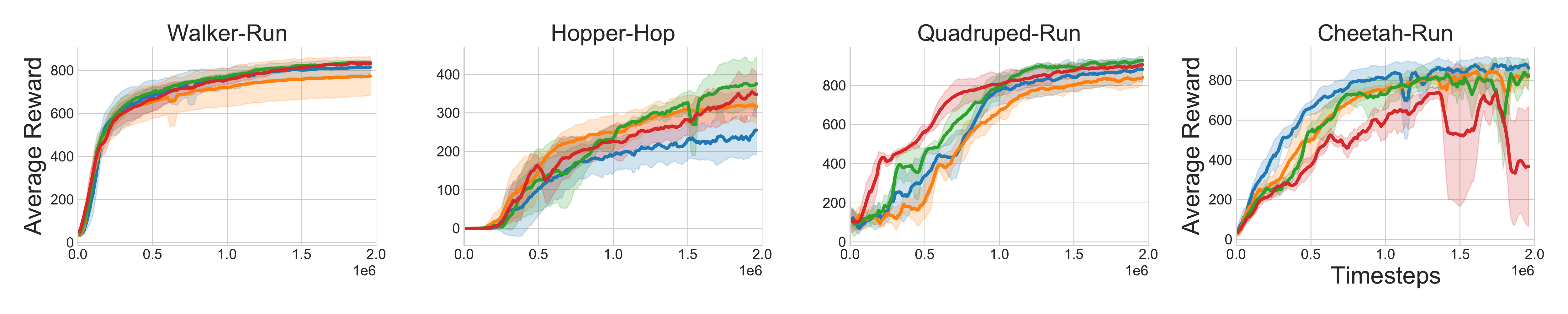}
      \\
       \includegraphics[width=0.9\linewidth]{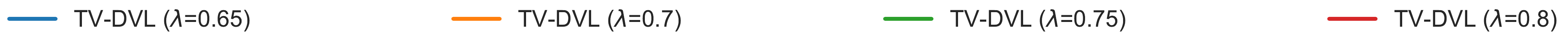}
      \vspace{-2.0mm}
\end{center}
\caption{Online RL: Ablating the temperature parameter for \fdvl (Total variation). The plot shows the effect of temperature parameters on learning performance.}
\label{fig:online_rl_tv_sensitivity_plots}
\end{figure}

\begin{figure}[h]
\begin{center}
    \includegraphics[width=1.0\linewidth]{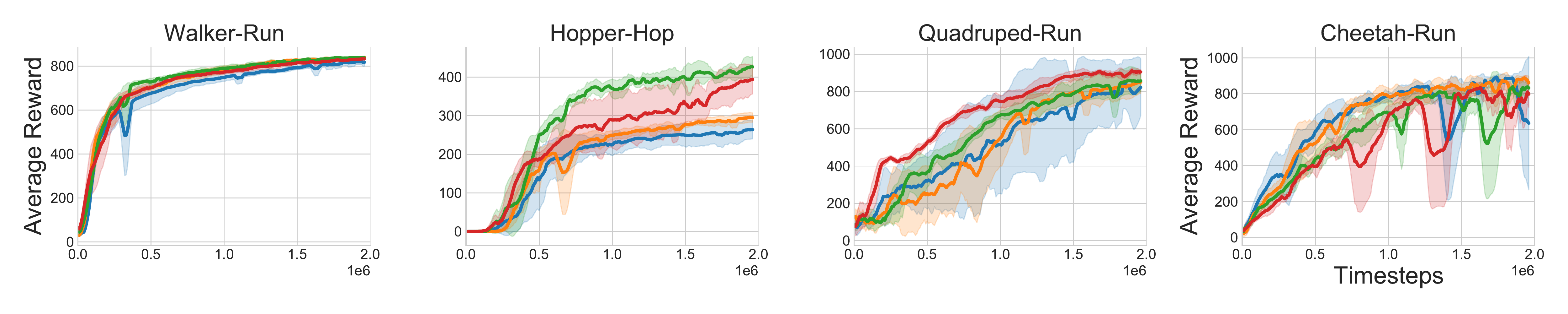}
      \\
       \includegraphics[width=0.9\linewidth]{figures/online_rl_sensitivity/online-rl-tv-sensitivity-legend.pdf}
      \vspace{-2.0mm}
\end{center}
\caption{Online RL: Ablating the temperature parameter for \fdvl (Chi-square). The plot shows the effect of temperature parameters on learning performance.}
\label{fig:online_rl_chi_sensitivity_plots}
\end{figure}

\subsection{Recovering Reward functions from \texttt{ReCOIL}}
\label{ap:recovering_rewards}
We study the quality of reward functions recovered from \texttt{ReCOIL} using the hopper-medium-expert and Walker2d-medium-expert datasets and the setup described in Section~\ref{sec:expr_offline_il}. For all trajectories in this dataset, we calculate the ground truth return (sum of rewards) and the predicted cumulative reward using \texttt{ReCOIL}. The scatter plot in figure~\ref{fig:recovering_mujoco_rewards} shows the correlation between predicted rewards. We note that \texttt{ReCOIL} is an IRL method and suffers from the reward ambiguity problems as rest of the IRL methods--- we can only expect a reward function that induces an optimal policy whose visitation is close to an expert and cannot guarantee that we recover the expert's exact reward function. To test the quality of rewards functions output by IRL methods, Pearson correlation is not the accurate metric and metrics like EPIC~\citep{gleave2020quantifying} might be used instead.

\begin{figure}[H]
\begin{center}
    \includegraphics[width=0.8\linewidth]{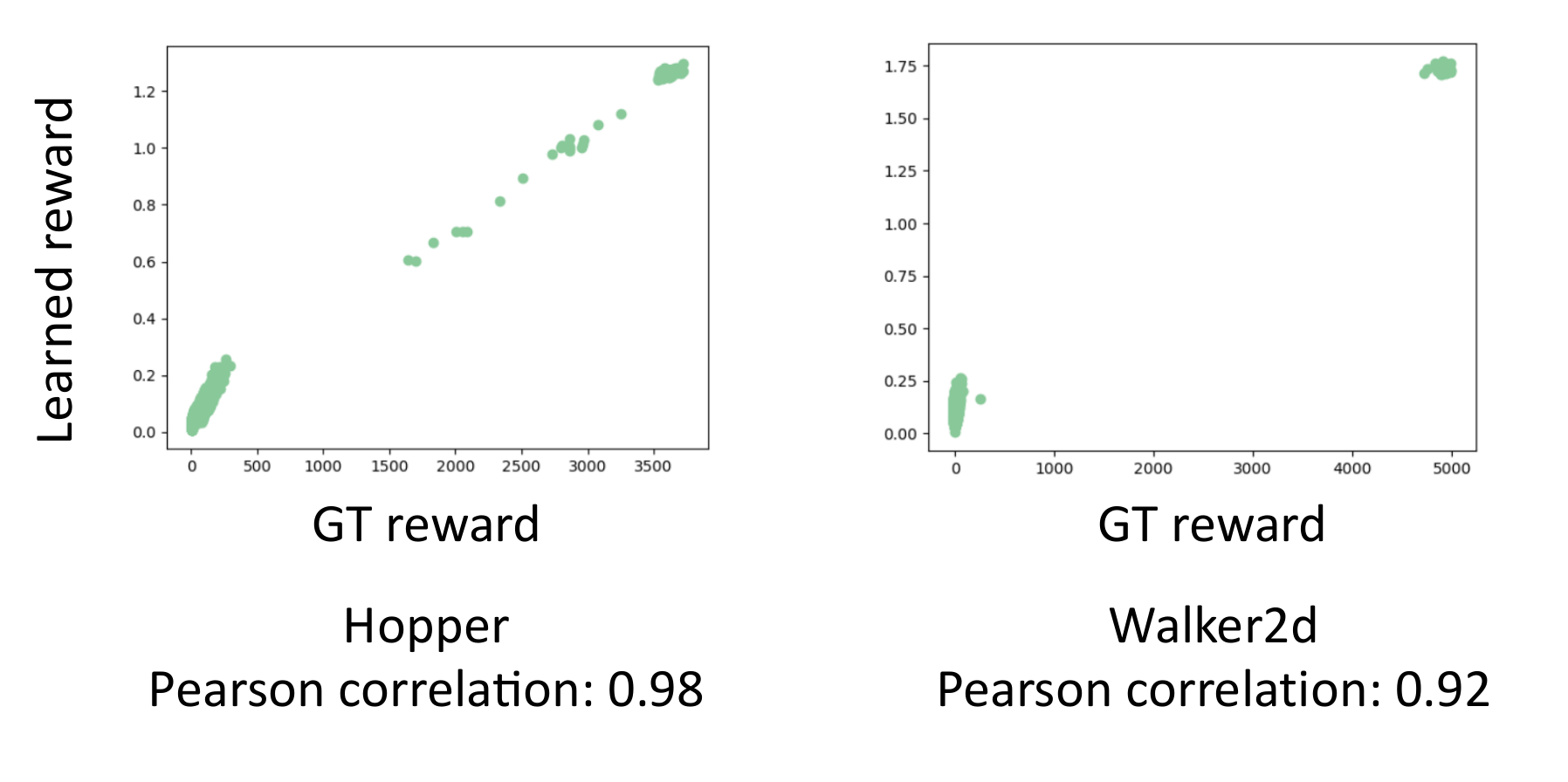}
\end{center}
\caption{Correlation of the rewards inferred by \texttt{ReCOIL} with respect to the ground truth reward function of the expert.}
\label{fig:recovering_mujoco_rewards}
\end{figure}

%% file: 10appendix-dual-lagrangian.tex
In this section, we give a more detailed review than we are able to in the main text due to space constraints. We consider RL problems with their average return considered in the form of a convex program with linear constraints~\citep{manne1960linear}, to which we apply Lagrangian duality to obtain corresponding constraint-free problems. This framework was first introduced in the work of~\citet{nachum2020reinforcement}, which obtains the same formulations as ours via Fenchel-Rockfeller duality. Here we use Lagrangian duality for its simplicity and popularity. 

Consider the following regularized policy learning problem 
\begin{equation}
    \label{eq:reg_rl_ap}
    \max_\pi \, J(\pi)=\mathbb{E}_{d^\pi(s,a)}[r(s,a)] -\alpha \f{d^\pi(s,a)}{d^O(s,a)},
\end{equation}
where 
$\f{d^\pi(s,a)}{d^O(s,a)}$ is a conservatism regularizer that encourages the visitation distribution of $\pi$ to stay close to some distribution $d^O$, and
$\alpha$ is a temperature parameter that balances the expected return and the conservatism.

An interesting fact is that $J(\pi)$ can be rewritten as a convex problem that searches for an \textit{achievable} visitation distribution that satisfies the \textit{Bellman-flow} constraints:
\begin{equation}
 \begin{aligned}
\label{eq:primal_rl_q_ap}
             J(\pi) & = \max_{d} \; \mathbb{E}_{d(s,a)}[r(s,a)]-{\alpha}\f{d(s,a)}{d^O(s,a)}\\
    & \text{s.t} \resizebox{0.85 \hsize}{!}{$\; d(s,a)=(1-\gamma)d_0(s).\pi(a|s)+\gamma \textstyle \sum_{s',a'} d(s',a')p(s|s',a')\pi(a|s), \; \forall s \in \S, a \in \A.$}
\end{aligned}    
\end{equation}
Applying Lagrangian duality and convex conjugate~\eqref{eq:cx_conjugate_def} to this problem, we can convert it to an unconstrained problem with dual variables $Q(s, a)$ defined for all $s, a \in \S \times \A$:
\begin{equation}
\label{eq:dual-Q-inner_ap}
 \min_{Q} (1-\gamma) \E{s\sim d_0,a \sim \pi(s)}{Q(s,a)} + {\alpha}\E{(s,a)\sim d^O}{f^*\left(\left[\bellman Q(s,a)-Q(s,a)\right]/
\alpha \right)},
\end{equation}
where $f^*$ is the convex conjugate of $f$. We defer the derivation to the next section.
As problem \eqref{eq:primal_rl_q_ap} is convex, strong duality holds and problems \eqref{eq:primal_rl_q_ap}
and \eqref{eq:dual-Q-inner_ap} have the same optimal objective value up to a constant scaling\footnote{We scaled the dual problem by $1/\alpha$ for derivation simplicity.}. We refer to the nested policy learning problem where $J(\pi)$ is of form \eqref{eq:primal_rl_q_ap} as \primalQ and the joint problem with scaled $J(\pi)$ of form \eqref{eq:dual-Q-inner_ap}
as \dualQ.
\begin{align}
\colorbox{RoyalBlue!15}{\primalQ} &\; \max_\pi \, [J(\pi) \, \text{in the form Eq.~\eqref{eq:primal_rl_q_main}]}, \\
\label{eq:dual-Q}
\colorbox{Goldenrod!30}{\dualQ} & \; 
\resizebox{0.8\hsize}{!}{$\max_{\pi} \min_{Q} (1-\gamma)\E{s\sim d_0,a \sim \pi(s)}{Q(s,a)} +{\alpha}\E{(s,a)\sim d^O}{f^*\left(\left[ \bellman Q(s,a)-Q(s,a)\right]/ \alpha \right)}$.}
 \end{align}
In fact, problem~\eqref{eq:primal_rl_q_ap} is overconstrained -- the maximization w.r.t $d$ is unnecessary, as for a fixed $\pi$ the $|\mathcal{S}| \times |\mathcal{A}|$ equality constraints already uniquely determine a solution $d^\pi$~\citep{puterman2014markov}. 
Let $\pi^*, d^*$ be the optimal policy and corresponding visitation distribution.
In fact, we can relax the constraints to get another problem~\citep{agarwal2019reinforcement} with the same optimal solution $d^*$, which we call \primalV below:
\begin{equation}
  \begin{aligned}
\label{eq:primal_rl_v_ap}
\colorbox{RoyalBlue!15}{\primalV} &~~
    \max_{d \geq 0}  \; \mathbb{E}_{d(s,a)}[r(s,a)]-\alpha\f{d(s,a)}{d^O(s,a)}\\
   &~~ \text{s.t} \; \textstyle \sum_{a\in\mathcal{A}} d(s,a)=(1-\gamma)d_0(s)+\gamma \sum_{(s',a') \in \S \times \A } d(s',a') p(s|s',a'), \; \forall s \in \S.
\end{aligned}  
\end{equation}
Comparing with problem~\eqref{eq:primal_rl_q_ap}, the constraints are relaxed and there is no policy $\pi$ in this formulation. 
In fact, as opposed to \primalQ, which needs to solve nested inner problems, \primalV solves a single problem to obtain $d^*$, from which we can recover $\pi^*$ via Eq.~\eqref{eq:recover_pi_from_d}\footnote{Eq.~\eqref{eq:recover_pi_from_d} can be easily computed for discrete actions, yet it is difficult for continuous actions. While our analysis focuses on the tabular case,
we discuss two methods for recovering $\pi^*$ for continuous actions in Appendix~\ref{ap:recovering_policy}.}:
\begin{equation}
    \label{eq:recover_pi_from_d}
    \pi(a|s) = d^\pi(s,a) / \textstyle \sum_{a \in \A}d^\pi(s, a).
\end{equation}
Similarly, we consider the Lagrangian dual of \eqref{eq:primal_rl_v_ap}, with dual variables $V(s)$ defined for all $s \in S$:
\begin{align}
\label{eq:dual-V_ap}
\colorbox{Goldenrod!30}{\dualV} \;
\min_{V} {(1-\gamma)}\E{s \sim d_0}{V(s)} +{\alpha}\E{(s,a)\sim d^O}{f^*_p\left(\left[\mathcal{T}V(s,a)-V(s))\right]/
\alpha\right)},
\end{align}
where $f^*_p$ is a variant of $f^*$ defined in Eq.~\eqref{eq:f_star_p_def}. Such modification is to cope with the nonnegativity constraint $d(s, a) \geq 0$ in \primalV. This constraint is ignored in \primalQ because the constraints of the inner problem~\eqref{eq:primal_rl_q_ap} already uniquely identify the solution. See Appendix~\ref{app:derive_dual_v} for the derivation.
As before, strong duality holds here (up to a factor of $1/\alpha$), and we can compute the optimal policy $\pi^*$ after obtaining $V^*$.
We discuss this in detail in Appendix~\ref{ap:recovering_policy}.

\emph{Remark 1.}  The above formulations generalizes to the popular MaxEnt RL framework, where the objective $J(\pi)$ contains an extra policy entropy regularizer. One only needs to replace
the Bellman operator $\mathcal{T}^\pi_r$ by its soft variant: $\mathcal{T}^\pi_{r, \text{soft}} Q(s,a) = r(s,a) + \gamma \E{s', a'}{Q(s',a') - \log \pi(a'|s')}$.

\emph{Remark 2.} We derive the dual problems via the Lagrangian duality. 
Taking the \primalQ problem as an example, the key step which bridges its Lagrangian dual problem $\min_Q \max_d L(Q, d)$ and the final formulation \dualQ is that the maximizer $d^*$ of the inner problem has a closed form solution. Equivalently, we can rewrite the inner problem $\max_d L(Q, d)$ via the convex conjugate~\eqref{eq:f_cvx_conjugate}, which eliminates the variable $d$. 
The Fenchel-Rockerfeller duality provides an alternative way to directly reach the same formulation, where one first rewrites the linear constraints as part of the objective using the Dirac delta function~\citep{nachum2020reinforcement}.\looseness=-1

\emph{Remark 3.} The dual formulations have a few appealing properties. (a)  They allow us to transform constrained distribution-matching problems,  w.r.t previously logged data, into unconstrained forms.
(b)  One can show that the gradient of \dualQ w.r.t $\pi$, when $Q$ is optimized for the inner problem, is the on-policy policy gradient computed by off-policy data. This key property relieves the instability or divergence issue in off-policy learning. (c)  The dual framework can be extended to the max-entropy RL setting, where $J(\pi)$  consists of additional entropy regularization, by replacing Bellman-operator with their soft Bellman counterparts~\citep{haarnoja2017reinforcement}.

%% file: 11appendix-recoil-perf-bound.tex
\subsection{Suboptimality Bound for \texttt{ReCOIL-V}}

\begin{restatable}[Suboptimality Bound for Offline \texttt{ReCOIL}]{theorem}{recoil2}
\label{thm:recoil}
 Let $S^J$ denote the joint support of $d^S$ and $d^E$. Let $r(s,a) = V(s)- \mathcal{T}_0V(s,a)$ be the pseudo-reward implied by \texttt{ReCOIL} and $R_{\max} = \max_{s,a} r(s,a)$.
Let $D_\delta = \set{ d \, | \Pr_d\big((s,a) \in S^J \big) \geq 1 - \delta}$ be the set of visitation distributions that have $1 - \delta$ coverage of $S^J$.
Let $\pi^*_{\delta}$ be the best policy over all policies whose visitation distribution is in $D_\delta$.
Let $g(d,V) = (1-\gamma)\E{d_0(s)}{V(s)} + \E{d}{ \mathcal{T}_0V(s, a)-V(s)} - \f{d(s,a)}{d^E(s,a)}$ be the imitation learning objective, and
$h(V) = \max_{d \in D_\delta} g(d,V)$.
Suppose that we can solve \texttt{ReCOIL} with the constraint $d \in D_\delta$, $h$ is $\kappa$-strongly convex in $V$ and $\beta \rightarrow 1$, then the output policy $\hat{\pi}$ satisfies that
 $J(\pi^*_{\delta}) - J(\hat{\pi}) \leq \tfrac{4}{1-\gamma} \sqrt{ 2\delta R_{\max} / \kappa}$.
\vspace{-2mm}
\end{restatable}

We provide a suboptimality bound by analyzing \texttt{ReCOIL-V} in this section. An important note is that we consider the setting $\beta \to 1$, which implies that we study the behavior of the optimization when $\beta$ is a number close to 1 and not exactly 1. This allows us to incorporate suboptimal data in off-policy imitation learning setting.

\label{sec:recoil_proof}
Recall that \texttt{ReCOIL-V} admits a \dualV form~\eqref{eq:recoil_v}.
When deriving \dualV, there is one step (Eq.~\eqref{eq:dual_v_derivation_is}) where we assumed the importance sampling is exact, i.e., 
\begin{equation}
   \E{(s,a)\sim d}{\mathcal{T}V(s,a) - V(s)} = \E{(s,a)\sim d^O}{\tfrac{d(s,a)}{d^O(s,a)}(\mathcal{T}V(s,a) - V(s))}. 
\end{equation}
However, this assumption does not hold in general and is not practical, because
$d^O$ and $d$ might have different support. The gap between the two terms greatly affects the performance of dual RL approaches.
We shall bound the approximation error introduced by importance sampling for \texttt{ReCOIL-V} in Section~\ref{app:approx_err_is}, and then bound the suboptimality
of the learned policy in Section~\ref{app:perf_bound_recoil_v}, under mild conditions. This analysis also results in the suboptimality bound of \texttt{IV-Learn} and \texttt{IQ-Learn} methods. 

Let $S^J$ denote the joint support of $d^S$ and $d^E$. 
Let $r(s,a) = V(s)-\gamma \mathcal{T}_0V(s,a)$ be the pseudo-reward implied by \texttt{ReCOIL} and $R_{\max} = \max_{s,a} |r(s,a)|$.
Let $D_\delta = \set{ d \, | \Pr_d\big((s,a) \in S^J \big) \geq 1 - \delta}$ be the set of visitation distributions that have $1 - \delta$ coverage of $S^J$,
where $\Pr_d\big((s,a) \in S^J \big)$ is the probabily that $(s,a)$ lies in $S^J$ when sampling $(s,a)$ from $d$.

We make the following assumptions for our proof:
    \item[\textbf{Assumption 1}] We consider imitation learning under the constraint $d \in D_\delta$. This is similar to pessimism assumption when learning from fixed datasets in offline RL~\cite{levine2020offline}.
    \item[\textbf{Assumption 2}] The hyperparameter $\beta$ for defining $\dmix$ and $\demix$ goes to $1$: $\beta \to 1$. 
    \item[\textbf{Assumption 3}] The function $h(V)$ defined in Section~\ref{app:perf_bound_recoil_v} is $\kappa$-strongly convex.

For Assumption 1, \texttt{ReCOIL-V}  is able to find a policy under the visitation constraint as a result of a combination of implicit maximization, which prevents overestimation and thus choosing OOD action, and weighted behavior cloning (Advantage-weighted regression), which keeps the output policy close to the dataset policy.


\subsubsection{Approximation Error of the Imitation Learning Objective}
\label{app:approx_err_is}



The imitation learning problem can be written in the Lagrangian form of \primalV where 
$r(s,a)=0$ everywhere:
\begin{equation}
\label{eq:original_objective_imitation}
    \min_V \max_{d \in D_\delta} \, (1-\gamma)\E{d_0(s)}{V(s)} + \E{d}{ \mathcal{T}_0V(s, a)-V(s)} - \f{d(s,a)}{d^E(s,a)},
\end{equation}
where we have a constraint $d \in D_\delta$ due to Assumption 1.
\texttt{ReCOIL-V} optimizes a surrogate objective of Problem~\eqref{eq:original_objective_imitation}.
To derive \texttt{ReCOIL-V}, consider the corresponding \primalV in its Lagrangian form 
\begin{equation}
    \min_V \max_{d \in D_\delta} \, (1-\gamma)\E{d_0(s)}{V(s)} + \E{d}{ \mathcal{T}_0V(s, a)-V(s)} - \f{\dmix}{\demix}.
\end{equation}
Rewriting the second term, we obtain
\begin{align}
\label{eq:recoil_objective_imitation}
& \min_V \max_{d \in D_\delta} \, (1-\gamma)\E{d_0(s)}{V(s)} + \frac{1}{\beta}\E{s,a\sim \dmix}{ \mathcal{T}_0V(s, a)-V(s)} \nonumber\\ 
& \hskip40pt - \f{\dmix}{\demix} - \frac{1-\beta}{\beta} \E{d^S}{ \mathcal{T}_0V(s, a)-V(s)}.
\end{align}
Now we \emph{approximate} the second term via importance sampling, which leads to
\begin{align}
& \min_V \max_{d \in D_\delta} \, (1-\gamma)\E{d_0(s)}{V(s)} + \frac{1}{\beta}\E{s,a \sim \demix}{ \frac{\dmix}{\demix}(\mathcal{T}_0V(s, a)-V(s))} \nonumber\\ 
& \hskip40pt - \E{\demix}{f(\frac{\dmix}{\demix})} - \frac{1-\beta}{\beta} \E{d^S}{ \mathcal{T}_0V(s, a)-V(s)}.
\end{align}
By expanding $\dmix = \beta d(s,a) + (1-\beta)d^S(s,a)$, we obtain
\begin{align}
& \min_V \max_{d \in D_\delta} \, (1-\gamma)\E{d_0(s)}{V(s)} + \frac{1}{\beta}\E{s,a \sim \demix}{ \frac{\dmix}{\demix}\left(\mathcal{T}_0V(s, a)-V(s)\right)} \nonumber\\ 
& \hskip40pt - \E{\demix}{f\left(\frac{\dmix}{\demix}\right)} - \frac{1-\beta}{\beta} \E{d^S}{\mathcal{T}_0V(s, a)-V(s)},
\end{align}
This can be further simplified to
\begin{align}
\label{eq:recoil_v_obj_compare}
& \min_V \max_{d \in D_\delta} \, (1-\gamma)\E{d_0(s)}{V(s)} + \E{s,a\sim \demix}{\frac{d(s,a)}{\demix}\left(\gamma \mathcal{T}_0V(s, a)-V(s)\right)}  \nonumber\\ 
& \hskip40pt - \E{\demix}{f\left(\frac{\dmix}{\demix}\right)},
\end{align}
where we used the fact
$$\E{s,a\sim \demix}{\frac{d^S(s,a)}{\demix}( \mathcal{T}_0V(s, a)-V(s))} = \E{s,a \sim d^S}{\mathcal{T}_0V(s, a)-V(s)}$$
as the support of $\demix$ contains the support of $d^S$.

Let $g(d, V)$ and $\hat{g}_\texttt{ReCOIL}(d,V)$ be the objective functions of Problem~\eqref{eq:original_objective_imitation}
and~\eqref{eq:recoil_v_obj_compare}.
$g(d,V)$ is the original IL objective we want to solve, and $\hat{g}_\texttt{ReCOIL}(d,V)$ is an approximation (with importance sampling) of $g(d, V)$ used by \texttt{ReCOIL-V}.
To simplify the analysis, we consider the case when mixture ratio $\beta \to 1$ (Assumption 2), so that \emph{the approximation error of
the objective function reduces to the approximation error of importance sampling.} That is,
\begin{equation}
    \abs{g(d,V) - \hat{g}_\texttt{ReCOIL}(d,V)} \to  \abs{\E{d}{ \mathcal{T}_0V(s, a)-V(s)} - \E{\demix}{\tfrac{d(s,a)}{\demix}\left( \mathcal{T}_0V(s, a)-V(s)\right)}}.
\end{equation}
For any visitation distribution $d \in D_\delta$, it holds that
\begin{align}
  & \abs{\E{d}{( \mathcal{T}_0V(s, a)-V(s))}-\E{\demix}{\tfrac{d(s,a)}{{\demix}}( \mathcal{T}_0V(s, a)-V(s))}}\nonumber \\
 & \leq \E{ s,a \in S^d \backslash S^J}{\abs{ \mathcal{T}_0V(s, a)-V(s)}} \le \max \delta \abs{ \mathcal{T}_0V(s, a) - V(s)}
 \le \delta R_{\max},
\end{align}
where $S^d$ is the support of $d$, and the second inequality follows from the definition of $D_\delta$. As a consequence, we can bound
the approximation error
\begin{align}
    \epsilon_\texttt{ReCOIL} =  \max_{d \in D_\delta, V} \abs{ g(d,V) - \lim_{\beta \to 1} \hat{g}_\texttt{ReCOIL}(d,V)} \leq \delta R_{\max}.
\end{align}
Similarly, one can show that for \texttt{IV-Learn}, we have
\begin{equation}
    \abs{g(d,V) - \hat{g}_\texttt{IVLearn}(d,V)} \to \abs{
    \E{d}{ \mathcal{T}_0V(s, a)-V(s)} - \E{s,a\sim d^E}{\tfrac{d(s,a)}{d^E(s,a)}} \left( \mathcal{T}_0V(s, a)-V(s)\right)}.
\end{equation}
Let $S^E$ be the support of $d^E$. Unlike \texttt{ReCOIL-V}, the objective of \texttt{IVLearn} suffers from the following worst-case estimation error 
\begin{align}
    & \abs{ \E{d}{( \mathcal{T}_0V(s, a)-V(s))}-\E{d^E}{\tfrac{d(s,a)}{d^E(s,a)}(\mathcal{T}_0V(s, a)-V(s))}  } \nonumber \\
   &  \leq \E{(s,a) \in S^d \backslash S^E}{\abs{ \mathcal{T}_0V(s, a)-V(s)}} \le \max \abs{ \mathcal{T}_0V(s, a) - V(s)} \le R_{\max},
\end{align}
and consequently
\begin{equation}
    \epsilon_\texttt{IVLearn} =  \max_{d \in D_\delta, V} \abs{ g(d,V) - \lim_{\beta \to 1} \hat{g}_\texttt{IVLearn}(d,V)} \leq R_{\max}.
\end{equation}

We note that the same approximation error bounds hold similarly for \texttt{IQLearn} as that of \texttt{IVLearn}. Thus \texttt{ReCOIL} has a smaller upper bound for the approximation error than \texttt{IQLearn} which we will see in the next sections leads to a better performance guarantee than \texttt{IQLearn}. 

\subsubsection{Performance Bound of the Learned Policy}
\label{app:perf_bound_recoil_v}

Recall that $\epsilon_\texttt{ReCOIL}$ denotes the approximation error of the objective function by \texttt{ReCOIL-V}:
\begin{equation}
\label{eq:bounded_approx_error}
    \epsilon_\texttt{ReCOIL} = \max_{d \in D_\delta, V} \abs{g(d,V)- \lim_{\beta \to 1}\hat{g}(d,V)}.
\end{equation}

Let $h(V) = \max_{d \in D_\delta} g(d,V)$ and $\hat{h}(V) = \max_{d \in D_\delta} \lim_{\beta \to 1} \hat{g}(d,V)$. It directly follows from Eq.~\eqref{eq:bounded_approx_error}
that 
\begin{equation}
\label{eq:abs_bound_h}
    |\hat{h}(V) - h(V)| \le 2\epsilon_\texttt{ReCOIL}, \; \forall V.
\end{equation}
We note that $\max_{d} g(d,V)$ (without the $d \in D_\delta$ constraint) is the standard \dualV form for imitation learning, but $h(V)$ here is defined as the same optimization under a constrained set $d \in D_\delta$.


Let $\hat{V} = \argmin_V \hat{h}(V)$ and $ V^* = \argmin_V h(V) $. We are interested
in bounding the gap $h(\hat{V}) - h(V^*)$. It holds that
\begin{align}
        h(\hat{V}) - h(V^*) 
    & = h(\hat{V}) - \hat{h}(\hat{V}) + \hat{h}(\hat{V}) - h(V^*) \\
    & = h(\hat{V}) - \hat{h}(\hat{V}) + \hat{h}(\hat{V}) - \hat{h}(V^*) + \hat{h}(V^*) - h(V^*) \\
    & \leq 2\epsilon_\texttt{ReCOIL} + 0 + 2\epsilon_\texttt{ReCOIL} \\
    & = 4 \epsilon_\texttt{ReCOIL},
\end{align}
where the inequality follows from Eq.~\eqref{eq:abs_bound_h} and the fact $\hat{V} = \argmin_V \hat{h}(V)$.

As a consequence, we have
\begin{align}
    4 \epsilon_\texttt{ReCOIL}
& \geq h(\hat{V}) - h(V^*) \\
& \geq h(V^*) + (V^*-\hat{V})\nabla h(V^*) + \frac{\kappa}{2}\|V^*-\hat{V}\|^2_F - h(V^*) \\
& = \frac{\kappa}{2}\|V^*-\hat{V}\|^2_F,
\end{align}
where the second inequality comes from the fact that the function $h(V)$ is $\kappa$-strongly convex (Assumption 3) and
$\nabla h(V^*) = 0$. It directly follows
that
\begin{align}
    \|V^*-\hat{V}\|_\infty \leq \|V^*-\hat{V}\|_F \leq  2\sqrt{\tfrac{2}{\kappa} \epsilon_\texttt{ReCOIL}}.
\end{align}

 Let $\pi^*_\delta$ be the policy that acts greedily with value function $V^*$, which is an optimal policy over all policies whose visitation distribution is within $D_\delta$. Let $\hat{\pi}$ denote the policy that acts greedily with value function $\hat{V}$, i.e., the output policy of \texttt{ReCOIL-V}. We then use the results in~\citet{singh1994upper} to bound the performance gap between $\pi^*_\delta$ and $\hat{\pi}$:
\begin{equation}
    J^{\pi^*_\delta} - J^{\hat{\pi}} \le \frac{4}{1-\gamma} \sqrt{\frac{2\epsilon_\texttt{ReCOIL}}{\kappa}} \leq  \frac{4}{1-\gamma} \sqrt{\frac{2\delta R_{\max}}{\kappa}}.
\end{equation}

The above results demonstrate that \texttt{ReCOIL} is able to leverage suboptimal data with an approximate in-distribution policy improvement and results in a policy close to the best policy with visitation almost in-support of the dataset.

%% file: main.bbl
\begin{thebibliography}{82}
\providecommand{\natexlab}[1]{#1}
\providecommand{\url}[1]{\texttt{#1}}
\expandafter\ifx\csname urlstyle\endcsname\relax
  \providecommand{\doi}[1]{doi: #1}\else
  \providecommand{\doi}{doi: \begingroup \urlstyle{rm}\Url}\fi

\bibitem[Agarwal et~al.(2019)Agarwal, Jiang, Kakade, and Sun]{agarwal2019reinforcement}
A.~Agarwal, N.~Jiang, S.~M. Kakade, and W.~Sun.
\newblock Reinforcement learning: Theory and algorithms.
\newblock \emph{CS Dept., UW Seattle, Seattle, WA, USA, Tech. Rep}, pages 10--4, 2019.

\bibitem[Agarwal et~al.(2020)Agarwal, Sikchi, Gulino, Wilkinson, and Gautam]{agarwal2020imitative}
S.~Agarwal, H.~Sikchi, C.~Gulino, E.~Wilkinson, and S.~Gautam.
\newblock Imitative planning using conditional normalizing flow.
\newblock \emph{arXiv preprint arXiv:2007.16162}, 2020.

\bibitem[Al-Hafez et~al.(2023)Al-Hafez, Tateo, Arenz, Zhao, and Peters]{al2023ls}
F.~Al-Hafez, D.~Tateo, O.~Arenz, G.~Zhao, and J.~Peters.
\newblock Ls-iq: Implicit reward regularization for inverse reinforcement learning.
\newblock \emph{arXiv preprint arXiv:2303.00599}, 2023.

\bibitem[Baird(1995)]{baird1995residual}
L.~Baird.
\newblock Residual algorithms: Reinforcement learning with function approximation.
\newblock In \emph{Machine Learning Proceedings 1995}, pages 30--37. Elsevier, 1995.

\bibitem[Ball et~al.(2023)Ball, Smith, Kostrikov, and Levine]{ball2023efficient}
P.~J. Ball, L.~Smith, I.~Kostrikov, and S.~Levine.
\newblock Efficient online reinforcement learning with offline data.
\newblock \emph{arXiv preprint arXiv:2302.02948}, 2023.

\bibitem[Bas-Serrano et~al.(2021)Bas-Serrano, Curi, Krause, and Neu]{bas2021logistic}
J.~Bas-Serrano, S.~Curi, A.~Krause, and G.~Neu.
\newblock Logistic q-learning.
\newblock In \emph{International Conference on Artificial Intelligence and Statistics}, pages 3610--3618. PMLR, 2021.

\bibitem[Bertsekas and Tsitsiklis(1995)]{bertsekas1995neuro}
D.~P. Bertsekas and J.~N. Tsitsiklis.
\newblock Neuro-dynamic programming: an overview.
\newblock In \emph{Proceedings of 1995 34th IEEE conference on decision and control}, volume~1, pages 560--564. IEEE, 1995.

\bibitem[Borkar(1988)]{borkar1988convex}
V.~S. Borkar.
\newblock A convex analytic approach to markov decision processes.
\newblock \emph{Probability Theory and Related Fields}, 78\penalty0 (4):\penalty0 583--602, 1988.

\bibitem[Chen et~al.(2021{\natexlab{a}})Chen, Lu, Rajeswaran, Lee, Grover, Laskin, Abbeel, Srinivas, and Mordatch]{chen2021decision}
L.~Chen, K.~Lu, A.~Rajeswaran, K.~Lee, A.~Grover, M.~Laskin, P.~Abbeel, A.~Srinivas, and I.~Mordatch.
\newblock Decision transformer: Reinforcement learning via sequence modeling.
\newblock \emph{Advances in neural information processing systems}, 34:\penalty0 15084--15097, 2021{\natexlab{a}}.

\bibitem[Chen et~al.(2021{\natexlab{b}})Chen, Wang, Zhou, and Ross]{chen2021randomized}
X.~Chen, C.~Wang, Z.~Zhou, and K.~Ross.
\newblock Randomized ensembled double q-learning: Learning fast without a model.
\newblock \emph{arXiv preprint arXiv:2101.05982}, 2021{\natexlab{b}}.

\bibitem[Cheng et~al.(2022)Cheng, Xie, Jiang, and Agarwal]{cheng2022adversarially}
C.-A. Cheng, T.~Xie, N.~Jiang, and A.~Agarwal.
\newblock Adversarially trained actor critic for offline reinforcement learning.
\newblock In \emph{International Conference on Machine Learning}, pages 3852--3878. PMLR, 2022.

\bibitem[Dai et~al.(2017)Dai, He, Pan, Boots, and Song]{dai2017learning}
B.~Dai, N.~He, Y.~Pan, B.~Boots, and L.~Song.
\newblock Learning from conditional distributions via dual embeddings.
\newblock In \emph{Artificial Intelligence and Statistics}, pages 1458--1467. PMLR, 2017.

\bibitem[de~Ghellinck and Eppen(1967)]{de1967linear}
G.~T. de~Ghellinck and G.~D. Eppen.
\newblock Linear programming solutions for separable markovian decision problems.
\newblock \emph{Management Science}, 13\penalty0 (5):\penalty0 371--394, 1967.

\bibitem[Denardo(1970)]{denardo1970linear}
E.~V. Denardo.
\newblock On linear programming in a markov decision problem.
\newblock \emph{Management Science}, 16\penalty0 (5):\penalty0 281--288, 1970.

\bibitem[Emmons et~al.(2021)Emmons, Eysenbach, Kostrikov, and Levine]{emmons2021rvs}
S.~Emmons, B.~Eysenbach, I.~Kostrikov, and S.~Levine.
\newblock Rvs: What is essential for offline rl via supervised learning?
\newblock \emph{arXiv preprint arXiv:2112.10751}, 2021.

\bibitem[Eysenbach et~al.(2021)Eysenbach, Levine, and Salakhutdinov]{eysenbach2021replacing}
B.~Eysenbach, S.~Levine, and R.~R. Salakhutdinov.
\newblock Replacing rewards with examples: Example-based policy search via recursive classification.
\newblock \emph{Advances in Neural Information Processing Systems}, 34:\penalty0 11541--11552, 2021.

\bibitem[Florence et~al.(2022)Florence, Lynch, Zeng, Ramirez, Wahid, Downs, Wong, Lee, Mordatch, and Tompson]{florence2022implicit}
P.~Florence, C.~Lynch, A.~Zeng, O.~A. Ramirez, A.~Wahid, L.~Downs, A.~Wong, J.~Lee, I.~Mordatch, and J.~Tompson.
\newblock Implicit behavioral cloning.
\newblock In \emph{Conference on Robot Learning}, pages 158--168. PMLR, 2022.

\bibitem[Fu et~al.(2019)Fu, Kumar, Soh, and Levine]{fu2019diagnosing}
J.~Fu, A.~Kumar, M.~Soh, and S.~Levine.
\newblock Diagnosing bottlenecks in deep q-learning algorithms.
\newblock In \emph{International Conference on Machine Learning}, pages 2021--2030. PMLR, 2019.

\bibitem[Fu et~al.(2020)Fu, Kumar, Nachum, Tucker, and Levine]{fu2020d4rl}
J.~Fu, A.~Kumar, O.~Nachum, G.~Tucker, and S.~Levine.
\newblock D4rl: Datasets for deep data-driven reinforcement learning.
\newblock \emph{arXiv preprint arXiv:2004.07219}, 2020.

\bibitem[Fujimoto and Gu(2021)]{fujimoto2021minimalist}
S.~Fujimoto and S.~S. Gu.
\newblock A minimalist approach to offline reinforcement learning.
\newblock \emph{Advances in neural information processing systems}, 34:\penalty0 20132--20145, 2021.

\bibitem[Fujimoto et~al.(2018)Fujimoto, Hoof, and Meger]{fujimoto2018addressing}
S.~Fujimoto, H.~Hoof, and D.~Meger.
\newblock Addressing function approximation error in actor-critic methods.
\newblock In \emph{International conference on machine learning}, pages 1587--1596. PMLR, 2018.

\bibitem[Fujimoto et~al.(2019)Fujimoto, Meger, and Precup]{fujimoto2019off}
S.~Fujimoto, D.~Meger, and D.~Precup.
\newblock Off-policy deep reinforcement learning without exploration.
\newblock In \emph{International conference on machine learning}, pages 2052--2062. PMLR, 2019.

\bibitem[Garg et~al.(2021)Garg, Chakraborty, Cundy, Song, and Ermon]{garg2021iq}
D.~Garg, S.~Chakraborty, C.~Cundy, J.~Song, and S.~Ermon.
\newblock Iq-learn: Inverse soft-q learning for imitation.
\newblock \emph{Advances in Neural Information Processing Systems}, 34:\penalty0 4028--4039, 2021.

\bibitem[Garg et~al.(2023)Garg, Hejna, Geist, and Ermon]{garg2023extreme}
D.~Garg, J.~Hejna, M.~Geist, and S.~Ermon.
\newblock Extreme q-learning: Maxent rl without entropy.
\newblock \emph{arXiv preprint arXiv:2301.02328}, 2023.

\bibitem[Ghasemipour et~al.(2020)Ghasemipour, Zemel, and Gu]{ghasemipour2020divergence}
S.~K.~S. Ghasemipour, R.~Zemel, and S.~Gu.
\newblock A divergence minimization perspective on imitation learning methods.
\newblock In \emph{Conference on Robot Learning}, pages 1259--1277. PMLR, 2020.

\bibitem[Gleave et~al.(2020)Gleave, Dennis, Legg, Russell, and Leike]{gleave2020quantifying}
A.~Gleave, M.~Dennis, S.~Legg, S.~Russell, and J.~Leike.
\newblock Quantifying differences in reward functions.
\newblock \emph{arXiv preprint arXiv:2006.13900}, 2020.

\bibitem[Goldfeld()]{goldfieldinfotheory}
Z.~Goldfeld.
\newblock Ece 5630: Information theory for data transmission, security and machine learning, lecture 8: Duality for f-divergences.

\bibitem[Haarnoja et~al.(2017)Haarnoja, Tang, Abbeel, and Levine]{haarnoja2017reinforcement}
T.~Haarnoja, H.~Tang, P.~Abbeel, and S.~Levine.
\newblock Reinforcement learning with deep energy-based policies.
\newblock In \emph{International conference on machine learning}, pages 1352--1361. PMLR, 2017.

\bibitem[Haarnoja et~al.(2018)Haarnoja, Zhou, Abbeel, and Levine]{haarnoja2018soft}
T.~Haarnoja, A.~Zhou, P.~Abbeel, and S.~Levine.
\newblock Soft actor-critic: Off-policy maximum entropy deep reinforcement learning with a stochastic actor.
\newblock In \emph{International conference on machine learning}, pages 1861--1870. PMLR, 2018.

\bibitem[Hafner et~al.(2023)Hafner, Pasukonis, Ba, and Lillicrap]{hafner2023mastering}
D.~Hafner, J.~Pasukonis, J.~Ba, and T.~Lillicrap.
\newblock Mastering diverse domains through world models.
\newblock \emph{arXiv preprint arXiv:2301.04104}, 2023.

\bibitem[Ho and Ermon(2016)]{ho2016generative}
J.~Ho and S.~Ermon.
\newblock Generative adversarial imitation learning.
\newblock \emph{Advances in neural information processing systems}, 29:\penalty0 4565--4573, 2016.

\bibitem[Hoshino et~al.(2022)Hoshino, Ota, Kanezaki, and Yokota]{hoshino2022opirl}
H.~Hoshino, K.~Ota, A.~Kanezaki, and R.~Yokota.
\newblock Opirl: Sample efficient off-policy inverse reinforcement learning via distribution matching.
\newblock In \emph{2022 International Conference on Robotics and Automation (ICRA)}, pages 448--454. IEEE, 2022.

\bibitem[Janner et~al.(2019)Janner, Fu, Zhang, and Levine]{janner2019trust}
M.~Janner, J.~Fu, M.~Zhang, and S.~Levine.
\newblock When to trust your model: Model-based policy optimization.
\newblock \emph{Advances in neural information processing systems}, 32, 2019.

\bibitem[Ji et~al.(2023)Ji, Luo, Sun, Zhan, Zhang, and Xu]{ji2023seizing}
T.~Ji, Y.~Luo, F.~Sun, X.~Zhan, J.~Zhang, and H.~Xu.
\newblock Seizing serendipity: Exploiting the value of past success in off-policy actor-critic.
\newblock \emph{arXiv preprint arXiv:2306.02865}, 2023.

\bibitem[Kim et~al.(2022{\natexlab{a}})Kim, Lee, Jang, Yang, and Kim]{kim2022lobsdice}
G.-H. Kim, J.~Lee, Y.~Jang, H.~Yang, and K.-E. Kim.
\newblock Lobsdice: offline imitation learning from observation via stationary distribution correction estimation.
\newblock \emph{arXiv preprint arXiv:2202.13536}, 2022{\natexlab{a}}.

\bibitem[Kim et~al.(2022{\natexlab{b}})Kim, Seo, Lee, Jeon, Hwang, Yang, and Kim]{kim2022demodice}
G.-H. Kim, S.~Seo, J.~Lee, W.~Jeon, H.~Hwang, H.~Yang, and K.-E. Kim.
\newblock Demodice: Offline imitation learning with supplementary imperfect demonstrations.
\newblock In \emph{International Conference on Learning Representations}, 2022{\natexlab{b}}.

\bibitem[Kostrikov et~al.(2018)Kostrikov, Agrawal, Dwibedi, Levine, and Tompson]{kostrikov2018discriminator}
I.~Kostrikov, K.~K. Agrawal, D.~Dwibedi, S.~Levine, and J.~Tompson.
\newblock Discriminator-actor-critic: Addressing sample inefficiency and reward bias in adversarial imitation learning.
\newblock \emph{arXiv preprint arXiv:1809.02925}, 2018.

\bibitem[Kostrikov et~al.(2019)Kostrikov, Nachum, and Tompson]{kostrikov2019imitation}
I.~Kostrikov, O.~Nachum, and J.~Tompson.
\newblock Imitation learning via off-policy distribution matching.
\newblock \emph{arXiv preprint arXiv:1912.05032}, 2019.

\bibitem[Kostrikov et~al.(2021)Kostrikov, Nair, and Levine]{kostrikov2021offline}
I.~Kostrikov, A.~Nair, and S.~Levine.
\newblock Offline reinforcement learning with implicit q-learning.
\newblock \emph{arXiv preprint arXiv:2110.06169}, 2021.

\bibitem[Kumar et~al.(2019)Kumar, Fu, Soh, Tucker, and Levine]{kumar2019stabilizing}
A.~Kumar, J.~Fu, M.~Soh, G.~Tucker, and S.~Levine.
\newblock Stabilizing off-policy q-learning via bootstrapping error reduction.
\newblock \emph{Advances in Neural Information Processing Systems}, 32, 2019.

\bibitem[Kumar et~al.(2020)Kumar, Zhou, Tucker, and Levine]{kumar2020conservative}
A.~Kumar, A.~Zhou, G.~Tucker, and S.~Levine.
\newblock Conservative q-learning for offline reinforcement learning.
\newblock \emph{Advances in Neural Information Processing Systems}, 33:\penalty0 1179--1191, 2020.

\bibitem[Lee et~al.(2021)Lee, Jeon, Lee, Pineau, and Kim]{lee2021optidice}
J.~Lee, W.~Jeon, B.~Lee, J.~Pineau, and K.-E. Kim.
\newblock Optidice: Offline policy optimization via stationary distribution correction estimation.
\newblock In \emph{International Conference on Machine Learning}, pages 6120--6130. PMLR, 2021.

\bibitem[Lei~Ba et~al.(2016)Lei~Ba, Kiros, and Hinton]{lei2016layer}
J.~Lei~Ba, J.~R. Kiros, and G.~E. Hinton.
\newblock Layer normalization.
\newblock \emph{ArXiv e-prints}, pages arXiv--1607, 2016.

\bibitem[Levine et~al.(2020)Levine, Kumar, Tucker, and Fu]{levine2020offline}
S.~Levine, A.~Kumar, G.~Tucker, and J.~Fu.
\newblock Offline reinforcement learning: Tutorial, review, and perspectives on open problems.
\newblock \emph{arXiv preprint arXiv:2005.01643}, 2020.

\bibitem[Ma et~al.(2022)Ma, Shen, Jayaraman, and Bastani]{ma2022smodice}
Y.~J. Ma, A.~Shen, D.~Jayaraman, and O.~Bastani.
\newblock Smodice: Versatile offline imitation learning via state occupancy matching.
\newblock \emph{arXiv preprint arXiv:2202.02433}, 2022.

\bibitem[Malek et~al.(2014)Malek, Abbasi-Yadkori, and Bartlett]{malek2014linear}
A.~Malek, Y.~Abbasi-Yadkori, and P.~Bartlett.
\newblock Linear programming for large-scale markov decision problems.
\newblock In \emph{International Conference on Machine Learning}, pages 496--504. PMLR, 2014.

\bibitem[Manne(1960)]{manne1960linear}
A.~S. Manne.
\newblock Linear programming and sequential decisions.
\newblock \emph{Management Science}, 6\penalty0 (3):\penalty0 259--267, 1960.

\bibitem[Mehta and Meyn(2009)]{mehta2009q}
P.~Mehta and S.~Meyn.
\newblock Q-learning and pontryagin's minimum principle.
\newblock In \emph{Proceedings of the 48h IEEE Conference on Decision and Control (CDC) held jointly with 2009 28th Chinese Control Conference}, pages 3598--3605. IEEE, 2009.

\bibitem[Nachum and Dai(2020)]{nachum2020reinforcement}
O.~Nachum and B.~Dai.
\newblock Reinforcement learning via fenchel-rockafellar duality.
\newblock \emph{arXiv preprint arXiv:2001.01866}, 2020.

\bibitem[Nachum et~al.(2019)Nachum, Dai, Kostrikov, Chow, Li, and Schuurmans]{nachum2019algaedice}
O.~Nachum, B.~Dai, I.~Kostrikov, Y.~Chow, L.~Li, and D.~Schuurmans.
\newblock Algaedice: Policy gradient from arbitrary experience.
\newblock \emph{arXiv preprint arXiv:1912.02074}, 2019.

\bibitem[Nair et~al.(2020)Nair, Gupta, Dalal, and Levine]{nair2020awac}
A.~Nair, A.~Gupta, M.~Dalal, and S.~Levine.
\newblock Awac: Accelerating online reinforcement learning with offline datasets.
\newblock \emph{arXiv preprint arXiv:2006.09359}, 2020.

\bibitem[Nakamoto et~al.(2023)Nakamoto, Zhai, Singh, Mark, Ma, Finn, Kumar, and Levine]{nakamoto2023cal}
M.~Nakamoto, Y.~Zhai, A.~Singh, M.~S. Mark, Y.~Ma, C.~Finn, A.~Kumar, and S.~Levine.
\newblock Cal-ql: Calibrated offline rl pre-training for efficient online fine-tuning.
\newblock \emph{arXiv preprint arXiv:2303.05479}, 2023.

\bibitem[Ni et~al.(2021)Ni, Sikchi, Wang, Gupta, Lee, and Eysenbach]{ni2021f}
T.~Ni, H.~Sikchi, Y.~Wang, T.~Gupta, L.~Lee, and B.~Eysenbach.
\newblock f-irl: Inverse reinforcement learning via state marginal matching.
\newblock In \emph{Conference on Robot Learning}, pages 529--551. PMLR, 2021.

\bibitem[Peters and Schaal(2007)]{peters2007reinforcement}
J.~Peters and S.~Schaal.
\newblock Reinforcement learning by reward-weighted regression for operational space control.
\newblock In \emph{Proceedings of the 24th international conference on Machine learning}, pages 745--750, 2007.

\bibitem[Picard-Weibel and Guedj(2022{\natexlab{a}})]{picard2022change}
A.~Picard-Weibel and B.~Guedj.
\newblock On change of measure inequalities for $ f $-divergences.
\newblock \emph{arXiv preprint arXiv:2202.05568}, 2022{\natexlab{a}}.

\bibitem[Picard-Weibel and Guedj(2022{\natexlab{b}})]{wiebelfdivergence}
A.~Picard-Weibel and B.~Guedj.
\newblock On change of measure inequalities for f -divergences.
\newblock 2022{\natexlab{b}}.

\bibitem[Precup(2000)]{precup2000eligibility}
D.~Precup.
\newblock Eligibility traces for off-policy policy evaluation.
\newblock \emph{Computer Science Department Faculty Publication Series}, page~80, 2000.

\bibitem[Precup et~al.(2001)Precup, Sutton, and Dasgupta]{precup2001off}
D.~Precup, R.~S. Sutton, and S.~Dasgupta.
\newblock Off-policy temporal-difference learning with function approximation.
\newblock In \emph{ICML}, pages 417--424, 2001.

\bibitem[Puterman(2014)]{puterman2014markov}
M.~L. Puterman.
\newblock \emph{Markov decision processes: discrete stochastic dynamic programming}.
\newblock John Wiley \& Sons, 2014.

\bibitem[R{\'e}nyi(1961)]{renyi1961measures}
A.~R{\'e}nyi.
\newblock On measures of entropy and information.
\newblock In \emph{Proceedings of the Fourth Berkeley Symposium on Mathematical Statistics and Probability, Volume 1: Contributions to the Theory of Statistics}, pages 547--561. University of California Press, 1961.

\bibitem[Schmidhuber(2019)]{schmidhuber2019reinforcement}
J.~Schmidhuber.
\newblock Reinforcement learning upside down: Don't predict rewards--just map them to actions.
\newblock \emph{arXiv preprint arXiv:1912.02875}, 2019.

\bibitem[Schulman et~al.(2015)Schulman, Levine, Abbeel, Jordan, and Moritz]{schulman2015trust}
J.~Schulman, S.~Levine, P.~Abbeel, M.~Jordan, and P.~Moritz.
\newblock Trust region policy optimization.
\newblock In \emph{International conference on machine learning}, pages 1889--1897. PMLR, 2015.

\bibitem[Schulman et~al.(2017)Schulman, Wolski, Dhariwal, Radford, and Klimov]{schulman2017proximal}
J.~Schulman, F.~Wolski, P.~Dhariwal, A.~Radford, and O.~Klimov.
\newblock Proximal policy optimization algorithms.
\newblock \emph{arXiv preprint arXiv:1707.06347}, 2017.

\bibitem[Sikchi et~al.(2022{\natexlab{a}})Sikchi, Saran, Goo, and Niekum]{sikchi2022ranking}
H.~Sikchi, A.~Saran, W.~Goo, and S.~Niekum.
\newblock A ranking game for imitation learning.
\newblock \emph{arXiv preprint arXiv:2202.03481}, 2022{\natexlab{a}}.

\bibitem[Sikchi et~al.(2022{\natexlab{b}})Sikchi, Zhou, and Held]{sikchi2022learning}
H.~Sikchi, W.~Zhou, and D.~Held.
\newblock Learning off-policy with online planning.
\newblock In \emph{Conference on Robot Learning}, pages 1622--1633. PMLR, 2022{\natexlab{b}}.

\bibitem[Singh et~al.(2022)Singh, Kumar, Vuong, Chebotar, and Levine]{singh2022offline}
A.~Singh, A.~Kumar, Q.~Vuong, Y.~Chebotar, and S.~Levine.
\newblock Offline rl with realistic datasets: Heteroskedasticity and support constraints.
\newblock \emph{arXiv preprint arXiv:2211.01052}, 2022.

\bibitem[Singh and Yee(1994)]{singh1994upper}
S.~P. Singh and R.~C. Yee.
\newblock An upper bound on the loss from approximate optimal-value functions.
\newblock \emph{Machine Learning}, 16:\penalty0 227--233, 1994.

\bibitem[Sun et~al.(2022)Sun, Han, Yang, Ma, Guo, and Zhou]{sun2022optimistic}
H.~Sun, L.~Han, R.~Yang, X.~Ma, J.~Guo, and B.~Zhou.
\newblock Optimistic curiosity exploration and conservative exploitation with linear reward shaping.
\newblock \emph{arXiv preprint arXiv:2209.07288}, 2022.

\bibitem[Sutton and Barto(2018)]{sutton2018reinforcement}
R.~S. Sutton and A.~G. Barto.
\newblock \emph{Reinforcement learning: An introduction}.
\newblock MIT press, 2018.

\bibitem[Swamy et~al.(2021)Swamy, Choudhury, Bagnell, and Wu]{swamy2021moments}
G.~Swamy, S.~Choudhury, J.~A. Bagnell, and S.~Wu.
\newblock Of moments and matching: A game-theoretic framework for closing the imitation gap.
\newblock In \emph{International Conference on Machine Learning}, pages 10022--10032. PMLR, 2021.

\bibitem[Thrun and Schwartz(1993)]{thrun1993issues}
S.~Thrun and A.~Schwartz.
\newblock Issues in using function approximation for reinforcement learning.
\newblock In \emph{Proceedings of the Fourth Connectionist Models Summer School}, volume 255, page 263. Hillsdale, NJ, 1993.

\bibitem[Todorov et~al.(2012)Todorov, Erez, and Tassa]{todorov2012mujoco}
E.~Todorov, T.~Erez, and Y.~Tassa.
\newblock Mujoco: A physics engine for model-based control.
\newblock In \emph{2012 IEEE/RSJ international conference on intelligent robots and systems}, pages 5026--5033. IEEE, 2012.

\bibitem[Tsitsiklis and Van~Roy(1996)]{tsitsiklis1996analysis}
J.~Tsitsiklis and B.~Van~Roy.
\newblock An analysis of temporal-difference learning with function approximation.
\newblock \emph{Rep. LIDS-P-2322). Lab. Inf. Decis. Syst. Massachusetts Inst. Technol. Tech. Rep}, 1996.

\bibitem[Uchendu et~al.(2022)Uchendu, Xiao, Lu, Zhu, Yan, Simon, Bennice, Fu, Ma, Jiao, et~al.]{uchendu2022jump}
I.~Uchendu, T.~Xiao, Y.~Lu, B.~Zhu, M.~Yan, J.~Simon, M.~Bennice, C.~Fu, C.~Ma, J.~Jiao, et~al.
\newblock Jump-start reinforcement learning.
\newblock \emph{arXiv preprint arXiv:2204.02372}, 2022.

\bibitem[Vecerik et~al.(2017)Vecerik, Hester, Scholz, Wang, Pietquin, Piot, Heess, Roth{\"o}rl, Lampe, and Riedmiller]{vecerik2017leveraging}
M.~Vecerik, T.~Hester, J.~Scholz, F.~Wang, O.~Pietquin, B.~Piot, N.~Heess, T.~Roth{\"o}rl, T.~Lampe, and M.~Riedmiller.
\newblock Leveraging demonstrations for deep reinforcement learning on robotics problems with sparse rewards.
\newblock \emph{arXiv preprint arXiv:1707.08817}, 2017.

\bibitem[Viano et~al.(2022)Viano, Kamoutsi, Neu, Krawczuk, and Cevher]{viano2022proximal}
L.~Viano, A.~Kamoutsi, G.~Neu, I.~Krawczuk, and V.~Cevher.
\newblock Proximal point imitation learning.
\newblock \emph{arXiv preprint arXiv:2209.10968}, 2022.

\bibitem[Wu et~al.(2019)Wu, Tucker, and Nachum]{wu2019behavior}
Y.~Wu, G.~Tucker, and O.~Nachum.
\newblock Behavior regularized offline reinforcement learning.
\newblock \emph{arXiv preprint arXiv:1911.11361}, 2019.

\bibitem[Zhang et~al.(2020)Zhang, Dai, Li, and Schuurmans]{zhang2020gendice}
R.~Zhang, B.~Dai, L.~Li, and D.~Schuurmans.
\newblock Gendice: Generalized offline estimation of stationary values.
\newblock \emph{arXiv preprint arXiv:2002.09072}, 2020.

\bibitem[Zheng et~al.(2022)Zheng, Zhang, and Grover]{zheng2022online}
Q.~Zheng, A.~Zhang, and A.~Grover.
\newblock Online decision transformer.
\newblock In \emph{international conference on machine learning}, pages 27042--27059. PMLR, 2022.

\bibitem[Zhu et~al.(2020)Zhu, Lin, Dai, and Zhou]{zhu2020off}
Z.~Zhu, K.~Lin, B.~Dai, and J.~Zhou.
\newblock Off-policy imitation learning from observations.
\newblock \emph{Advances in Neural Information Processing Systems}, 33:\penalty0 12402--12413, 2020.

\bibitem[Ziebart et~al.(2008)Ziebart, Maas, Bagnell, Dey, et~al.]{ziebart2008maximum}
B.~D. Ziebart, A.~L. Maas, J.~A. Bagnell, A.~K. Dey, et~al.
\newblock Maximum entropy inverse reinforcement learning.
\newblock In \emph{Aaai}, volume~8, pages 1433--1438. Chicago, IL, USA, 2008.

\bibitem[Zolna et~al.(2020)Zolna, Novikov, Konyushkova, Gulcehre, Wang, Aytar, Denil, de~Freitas, and Reed]{zolna2020offline}
K.~Zolna, A.~Novikov, K.~Konyushkova, C.~Gulcehre, Z.~Wang, Y.~Aytar, M.~Denil, N.~de~Freitas, and S.~Reed.
\newblock Offline learning from demonstrations and unlabeled experience.
\newblock \emph{arXiv preprint arXiv:2011.13885}, 2020.

\end{thebibliography}
